\documentclass{article}
\usepackage[utf8]{inputenc}
\usepackage[T1]{fontenc}
\usepackage{hyperref}
\usepackage{url}
\usepackage{booktabs}
\usepackage{amsfonts,amsmath,amssymb,amsthm}
\usepackage{graphicx}
\usepackage{xcolor}
\usepackage{subcaption}
\usepackage[round]{natbib}
\usepackage[margin=1in]{geometry}
\usepackage{tocloft}
\newtheorem{theorem}{Theorem}
\newtheorem{lemma}[theorem]{Lemma}
\newtheorem{corollary}[theorem]{Corollary}
\newtheorem{definition}[theorem]{Definition}
\newtheorem{proposition}[theorem]{Proposition}
\newtheorem{remark}[theorem]{Remark}
\newtheorem{axiom}{Axiom}

\newcommand{\R}{\mathbb{R}}
\newcommand{\E}{\mathbb{E}}
\newcommand{\Var}{\mathrm{Var}}
\newcommand{\SHAP}{\mathrm{SHAP}}
\newcommand{\DASH}{\textsc{Dash}}
\newcommand{\GBDT}{\textsc{GBDT}}

\title{The Attribution Impossibility:\\{\large No Feature Ranking Is Faithful, Stable, and Complete Under Collinearity}\\[8pt]\normalsize Definitive Reference (Complete Version)}

\author{
  Drake Caraker\thanks{Corresponding author: \texttt{drakecaraker@gmail.com}} \quad
  Bryan Arnold\thanks{\texttt{puremath86@gmail.com}} \quad
  David Rhoads\thanks{\texttt{drhoads9@gmail.com}} \\[4pt]
  \textit{Independent Researchers}
}

\begin{document}
\maketitle
\tableofcontents

\section*{Notation}
\label{sec:notation}
\begin{center}
\small
\begin{tabular}{@{}lp{10cm}@{}}
\toprule
\textbf{Symbol} & \textbf{Meaning} \\
\midrule
$\varphi_j(f)$ & Attribution (global importance) of feature $j$ in model $f$ \\
$\rho$ & Pairwise correlation between collinear features within a group \\
$T$ & Number of boosting rounds (trees) in a gradient-boosted decision tree (GBDT) \\
$M$ & Ensemble size (number of independently trained models) \\
$Z_{jk}$ & Test statistic for the pairwise $Z$-test: $Z_{jk} = \bar\Delta_{jk} / (\hat\sigma_{jk}/\sqrt{M})$ \\
$\Phi$ & Standard normal cumulative distribution function \\
$n_j(f)$ & Split count (utilization count) of feature $j$ in model $f$ \\
$c(f)$ & Proportionality constant: $\varphi_j(f) = c(f) \cdot n_j(f)$ \\
$\text{SNR}$ & Signal-to-noise ratio: $\Delta_{jk}/\sigma_{jk}$ \\
Z-test & Multi-model $Z$-test diagnostic (requires $M \geq 5$ models) \\
Screen & Single-model screening diagnostic (94\% precision) \\
$P$ & Total number of features \\
$L$ & Number of collinear groups \\
$m$ & Number of features within a collinear group \\
$j_1$ & First-mover: the feature selected at the root of the first tree \\
$\bar\varphi_j$ & \DASH{} consensus attribution: $\frac{1}{M}\sum_{i=1}^M \varphi_j(f_i)$ \\
$S$ & Ranking stability (expected Spearman correlation) \\
$U$ & Within-group unfaithfulness \\
$\Delta_{jk}$ & Population attribution gap: $|\mu_j - \mu_k|$ \\
$\sigma_{jk}$ & Attribution noise: $\sqrt{\Var(\varphi_j(f) - \varphi_k(f))}$ \\
\bottomrule
\end{tabular}
\end{center}

\section*{List of Named Results (by page)}
\begin{enumerate}
\item Theorem~\ref{thm:impossibility} (Attribution Impossibility) --- p.\pageref{thm:impossibility}
\item Theorem~\ref{thm:rashomon-symmetry} (Rashomon from Symmetry) --- p.\pageref{thm:rashomon-symmetry}
\item Theorem~\ref{thm:rashomon-inevitability} (Rashomon Inevitability) --- p.\pageref{thm:rashomon-inevitability}
\item Lemma~\ref{lem:split-gap} (Split Gap) --- p.\pageref{lem:split-gap}
\item Theorem~\ref{thm:ratio} (Attribution Ratio $1/(1-\rho^2)$) --- p.\pageref{thm:ratio}
\item Proposition~\ref{prop:exact-flip} (Exact GBDT Flip Rate) --- p.\pageref{prop:exact-flip}
\item Corollary~\ref{cor:equity} (DASH Equity) --- p.\pageref{cor:equity}
\item Theorem~\ref{thm:dash-optimal} (DASH Pareto Optimality) --- p.\pageref{thm:dash-optimal}
\item Proposition~\ref{prop:ensemble-lower-bound} (Ensemble Size Lower Bound) --- p.\pageref{prop:ensemble-lower-bound}
\item Theorem~\ref{thm:design-space-main} (Attribution Design Space) --- p.\pageref{thm:design-space-main}
\item Theorem~\ref{thm:unfaithfulness-bound} (Unfaithfulness Bound) --- p.\pageref{thm:unfaithfulness-bound}
\item Theorem~\ref{thm:path-convergence} (Path Convergence) --- p.\pageref{thm:path-convergence}
\item Theorem~\ref{thm:bayes-dichotomy} (Symmetric Bayes Dichotomy) --- p.\pageref{thm:bayes-dichotomy}
\item Theorem~\ref{thm:conditional-impossibility} (Conditional Attribution Impossibility) --- p.\pageref{thm:conditional-impossibility}
\item Theorem~\ref{thm:fairness-audit} (Fairness Audit Impossibility) --- p.\pageref{thm:fairness-audit}
\item Theorem~\ref{thm:fim-impossibility} (FIM Impossibility) --- p.\pageref{thm:fim-impossibility}
\item Theorem~\ref{thm:query-complexity} (Query Complexity Lower Bound) --- p.\pageref{thm:query-complexity}
\end{enumerate}

\newpage


\begin{abstract}
No feature ranking can be simultaneously faithful, stable, and complete when features are collinear. For collinear pairs, ranking reduces to a coin flip.
We prove this impossibility, quantify it for four model classes, resolve it via SHAP (SHapley Additive exPlanations) ensemble averaging with \DASH{}, and machine-verify it with 305 Lean~4 theorems.
We characterize the complete attribution design space: exactly two families of methods exist---faithful-complete methods (unstable, with rankings that flip up to 50\% of the time) and ensemble methods like \DASH{} (stable, reporting ties for symmetric features)---and no method lies outside this dichotomy.
The impossibility is quantitative: we derive architecture-discriminating bounds showing the attribution ratio diverges as $1/(1{-}\rho^2)$ for gradient boosting, is infinite for Lasso, and converges for random forests.
\DASH{} (\textbf{D}iversified \textbf{A}ggregation of \textbf{SH}AP) ensemble averaging is provably Pareto-optimal among unbiased aggregations on the stable branch, achieving the Cram\'er--Rao variance bound with a tight ensemble size formula.
In a survey of 77 public datasets, 68\% exhibit attribution instability (conservative lower bound: survey power 32\% at the 10\% threshold; true prevalence likely higher).
Switching to conditional SHAP does not escape the impossibility when features have equal causal effects.
The framework includes practical diagnostics---a $Z$-test workflow and single-model screening tool---and has direct consequences for fairness auditing: SHAP-based proxy discrimination audits are provably unreliable under collinearity.
The design space theorem, diagnostics, and impossibility are mechanically verified in Lean~4 (305 theorems from 14 domain-specific + 2 query-complexity axioms, 0 \texttt{sorry})---to our knowledge, the first formally verified impossibility in explainable AI.
Code and Lean proofs are available at \url{https://github.com/DrakeCaraker/dash-impossibility-lean}.
\end{abstract}

\section*{Executive Summary}

\noindent\textbf{The problem.} SHAP feature importance rankings---the most widely used method for explaining machine learning predictions---are unreliable when features are correlated. Retraining the same model with a different random seed can invert which feature is ``most important.'' In a survey of 77 public datasets, 68\% exhibit this instability.

\noindent\textbf{The theorem.} We prove this is not an engineering failure but a mathematical impossibility: no feature ranking can simultaneously be \emph{faithful} (reflect the model), \emph{stable} (robust to retraining), and \emph{complete} (rank all features). The proof requires only four lines from the Rashomon property---that collinear designs admit models ranking features in opposite orders.

\noindent\textbf{The design space.} The complete set of achievable methods consists of exactly two families:
\begin{itemize}
    \item \textbf{Family A} (single-model): Faithful and complete, but rankings flip up to 50\% of the time.
    \item \textbf{Family B} (\DASH{} ensemble): Stable, but reports ties for symmetric features. Provably Pareto-optimal among unbiased aggregations.
\end{itemize}

\noindent\textbf{The practical toolkit.}
\begin{enumerate}
    \item \textbf{Diagnose}: The single-model screen (94\% precision on Breast Cancer; tree-based models only) identifies unstable feature pairs from one model.
    \item \textbf{Validate}: The multi-model Z-test ($|r| > 0.8$ on 9 of 10 datasets with unstable pairs) confirms instability with 5 models.
    \item \textbf{Resolve}: \DASH{} consensus averaging with $M \geq 25$ models reduces flip rate below 1\%.
    \item \textbf{Size}: The ensemble formula $M_{\min} = \lceil 2.71 \cdot \sigma^2/\Delta^2 \rceil$ gives the optimal model count.
\end{enumerate}

\noindent\textbf{For regulators.} SHAP-based proxy discrimination audits are provably unreliable under collinearity: the audit conclusion is a coin flip across training seeds. This constitutes a ``known and foreseeable circumstance'' under EU AI Act Art.~13(3)(b)(ii). The instability disclosure template in Section~9 provides ready-to-use language for model documentation.

\noindent\textbf{Verification.} The framework is mechanically verified in Lean~4: 305 theorems from 16 axioms with 0 \texttt{sorry} across 54 files. The formalization caught two logical inconsistencies that survived informal review.

\noindent\textbf{Code.}
\begin{sloppypar}
\noindent Reference implementation:
\url{https://github.com/DrakeCaraker/dash-shap} (stability API in PR~\#255).
Lean formalization:
\url{https://github.com/DrakeCaraker/dash-impossibility-lean}.
\end{sloppypar}

\paragraph{Pseudocode (5 lines).}
{\small\begin{verbatim}
models   = [XGBClassifier(random_state=i).fit(X, y) for i in range(25)]
shap_vals = [mean(|TreeExplainer(m).shap_values(X_test)|) for m in models]
dash = mean(shap_vals, axis=0)       # DASH consensus
for j, k in pairs:
    Z = |mean(shap_j - shap_k)| / (std / sqrt(25))  # unstable if Z<1.96
\end{verbatim}}

\begin{figure}[t]
\centering
\includegraphics[width=\textwidth]{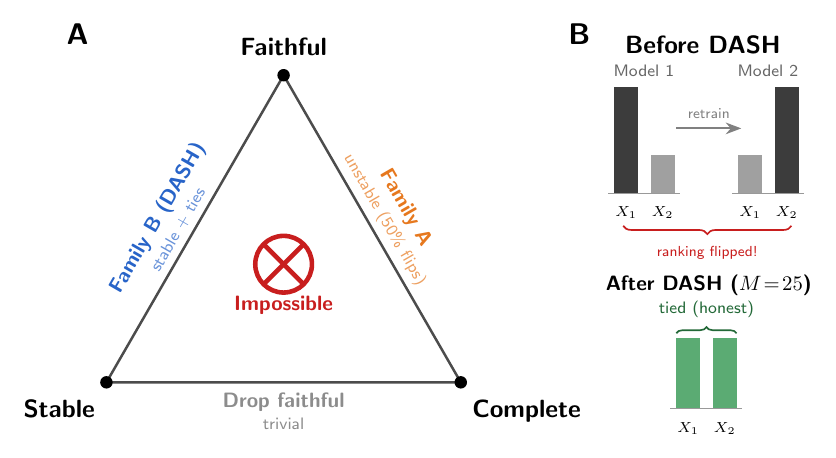}
\caption{The Attribution Impossibility. No feature ranking can be simultaneously faithful, stable, and complete under collinearity. \emph{Faithful, Stable, Complete: Pick Two.} Family~$\mathcal{A}$ (single-model) is faithful and complete but rankings flip up to 50\% of the time. Family~$\mathcal{B}$ (\DASH{} ensemble) is stable with zero unfaithfulness but reports ties for symmetric features (sacrificing completeness). No third option exists.}
\label{fig:trilemma}
\end{figure}


\section{Introduction}
\label{sec:intro}

Train a gradient-boosted model on your data. Compute SHAP values. Retrain with a different random seed. The ``most important feature'' changes---in 68\% of 77 public datasets. Most practitioners assume this is an engineering problem---insufficient data, poor hyperparameters, or a software bug. We prove it is none of these. It is a mathematical impossibility. For $m{=}2$ collinear features, ranking the pair is literally a coin flip: retraining with a different seed inverts the importance ordering 50\% of the time.

This document is the complete reference: impossibility theorem, quantitative bounds, constructive resolution, three Symmetric Bayes Dichotomy (SBD) instances, diagnostics, experiments across 11 datasets and 3 GBDT implementations, financial case studies, and a 305-theorem Lean~4 formalization.

\paragraph{A concrete example.}
Consider a credit model with annual income and debt-to-income ratio
(typically $|\rho| \approx 0.7$--$0.85$ in lending data). Across 20 retrains with different seeds, SHAP
can report income as the top driver roughly half the time and DTI the other half.
A compliance report citing income as the primary adverse factor
is a coin flip---the same model architecture, same data, different
random seed. Under the Equal Credit Opportunity Act (ECOA), the lender's adverse action notice
changes based solely on the training randomness.

\begin{sloppypar}
Several recent results have established fundamental limits on feature attribution.
\citet{bilodeau2024impossibility} prove that completeness and linearity cannot coexist;
\citet{huang2024failings} show SHAP can misrank features in Boolean domains;
\citet{srinivas2019full} show complete attributions cannot be weakly input-dependent;
and \citet{rao2025limits} establish Kolmogorov complexity barriers to explainability.
None of these results analyze \emph{stability}---the robustness of attributions to
retraining---none provide quantitative architecture-discriminating bounds, and none
offer a constructive resolution.
\end{sloppypar}

\begin{sloppypar}
This paper fills these gaps. Our result is not a criticism of SHAP, which correctly computes Shapley values for each model; rather, it reveals a fundamental limitation of the single-model paradigm under which any attribution method must operate.
Our main result is the \emph{Attribution Design Space Theorem}
(\S\ref{sec:design-space}, Figure~\ref{fig:design-space}): a complete
characterization of what is achievable when explaining models with collinear features.
The achievable set consists of exactly two families---faithful-complete methods
(unstable, unfaithful to half the models) and ensemble methods like \DASH{} (stable,
with within-group ties)---and \DASH{} is Pareto-optimal.
The theorem rests on an \emph{Attribution Impossibility}: for any single model trained
on collinear features, faithfulness, stability, and completeness are mutually
incompatible.
The impossibility has two layers.
The first is model-agnostic: formalizing the Rashomon property
\citep{rudin2024amazing}, we show symmetric features are necessarily ranked in
opposite orders by different near-optimal models.
The second is model-specific: we derive quantitative bounds on the attribution ratio
(the factor by which the dominant feature is overweighted).
For \GBDT{}, this ratio follows $1/(1-\alpha\rho^2)$ where $\alpha$ captures the
per-tree signal fraction; for Lasso it is infinite at any $\rho > 0$; for neural
networks it depends on the initialization; and for random forests it is
$O(1/\sqrt{T})$---converging rather than diverging.
\end{sloppypar}

\begin{sloppypar}
The impossibility parallels the fairness impossibility of \citep{chouldechova2017fair} and \citep{kleinberg2017inherent}, who proved that calibration, balance, and equal false positive rates cannot coexist when base rates differ. Our result plays the same role for explainability: instability under collinearity is not an engineering deficiency but a mathematical inevitability. This has direct implications for EU AI Act Art.~13(3)(b)(ii) \citep{euaiact2024}, which requires disclosure of ``known and foreseeable circumstances'' that may lead to risks to health, safety, or fundamental rights.
\end{sloppypar}

The impossibility also admits a principled relaxation.
\DASH{} ensemble averaging breaks the sequential dependence that drives attribution concentration: for balanced ensembles, expected attributions are equitable across symmetric features with variance $O(1/M)$.
Practitioners can achieve between-group faithfulness and stability by aggregating across the Rashomon set, sacrificing within-group completeness (reporting ties for symmetric features).

\begin{sloppypar}
The proof is mechanically verified in Lean~4 using Mathlib: 305 type-checked theorems
and lemmas across 54 files (80 with multi-step proofs of ${\geq}5$ tactic lines; the
remainder are definitions, wrappers, and single-step applications; from 14
domain-specific + 2 query-complexity axioms; 0 \texttt{sorry}).
The core impossibility (Theorem~\ref{thm:impossibility}) depends on zero axioms
beyond the Rashomon property; the DASH equity result (Corollary~\ref{cor:equity})
depends on the \texttt{attribution\_sum\_symmetric} theorem (derived in
\texttt{SymmetryDerive.lean} from the proportionality and split-count axioms).
This is, to our knowledge, the first formally verified impossibility result in
explainable AI (the core impossibility uses zero axioms; quantitative bounds are
conditional on 6 domain-specific axioms; query complexity uses 2 additional axioms
from Le~Cam's method).
The formalization also proved valuable beyond certification: during translation to
Lean~4 we discovered two logical inconsistencies and one type mismatch in the
original axiom system, echoing prior experiences where formalization uncovered subtle
proof errors \citep{nipkow2009social,zhang2026statistical}.
\end{sloppypar}

\paragraph{Contributions.}
\begin{sloppypar}
\begin{enumerate}
    \item \textbf{The Attribution Design Space Theorem} (\S\ref{sec:design-space}):
    a complete characterization of achievable (stability, unfaithfulness, completeness)
    triples, showing the design space consists of exactly two families with \DASH{}
    Pareto-optimal on the ensemble branch. All other results are corollaries.
    \item \textbf{The Attribution Impossibility} (Theorem~\ref{thm:impossibility}):
    the base case of the Design Space Theorem---a model-agnostic impossibility
    requiring only the Rashomon property, with zero axiom dependencies in the Lean
    formalization.
    \item \textbf{Architecture discrimination}: attribution ratios for \GBDT{}
    ($1/(1{-}\rho^2)$ Lean-verified; $1/(1{-}\alpha\rho^2)$ empirically corrected),
    Lasso ($\infty$), and random forests ($O(1/\sqrt{T})$, informal).
    \item \textbf{Constructive relaxation via \DASH{}}: ensemble averaging restores
    equitable attributions (derived for balanced ensembles) with variance $O(1/M)$,
    sacrificing within-group completeness.
    \item \textbf{The symmetric Bayes dichotomy} (\S\ref{sec:sbd}): a proof technique
    from invariant decision theory \citep{lehmann2005testing}, demonstrated across three
    structurally distinct instances (feature attribution, model selection, and causal
    discovery under Markov equivalence) with different symmetry groups.
    \item \textbf{Machine-verified proof}: formalized in Lean~4 (305 type-checked theorems from 16 axioms, 0 \texttt{sorry}), publicly available at \url{https://github.com/DrakeCaraker/dash-impossibility-lean}. The axiom system is consistent: we construct an explicit model satisfying all 14 domain-specific axioms simultaneously, both in Lean (\texttt{Consistency.lean}) and numerically.
\end{enumerate}
\end{sloppypar}

\noindent\textit{For practitioners and regulators:} The instability disclosure template is in Section~9, the fairness audit impossibility in Section~8, and the financial case studies in Section~10. The Executive Summary above provides a complete overview without technical prerequisites.

The core impossibility proof is deliberately simple---a four-line contradiction from the Rashomon property.
Like Arrow's original argument \citep{arrow1951social}, the mathematical contribution lies not in proof complexity but in identifying the right abstraction (the Rashomon property) and characterizing the complete achievable set (the Design Space Theorem).
The quantitative depth comes from the architecture-specific bounds (\S\ref{sec:bounds}) and the Pareto optimality proof (\S\ref{sec:resolution}).

This document provides the complete treatment of the Attribution Impossibility---from intuition to proof to deployment. A practitioner can find the diagnostic workflow (Section~9), a theorist can find every proof (Sections~3--8 with Lean cross-references in Section~11), and a regulator can find compliance guidance (Section~8 and the Regulatory Mapping appendix).


\section{Setup: Three Desiderata for Feature Rankings}
\label{sec:setup}

We consider a supervised learning setting with $P$ input features partitioned into $L$ groups by their correlation structure.
Within each group $\ell \in [L]$, features share a common pairwise correlation $\rho \in (0,1)$; features in different groups are independent.
Each group contains at least two members.
This structure arises naturally in applied settings---e.g., multiple measurements of the same physical quantity, or one-hot encodings of related categories.

\paragraph{Models and attributions.}
A \emph{model} $f$ is trained from a random seed $s$ via a deterministic training procedure: $f = \texttt{train}(s)$.
For each feature $j \in [P]$, we observe a nonneg\-ative attribution $\varphi_j(f) \geq 0$, representing the global importance of feature $j$ in model $f$ (e.g., mean absolute SHAP value, gain-based importance, or integrated gradients norm).

We require a proportionality axiom connecting attributions to model structure:

\begin{axiom}[Proportionality]
\label{ax:proportional}
For every model $f$, there exists a constant $c(f) > 0$ such that $\varphi_j(f) = c(f) \cdot n_j(f)$ for all $j \in [P]$, where $n_j(f)$ is the utilization count of feature $j$ in model $f$ (e.g., split count for tree ensembles, gradient norm for neural networks).
\end{axiom}

This axiom holds exactly under the uniform-contribution model of \citep{lundberg2017unified}, and approximately whenever per-split (or per-neuron) contributions are roughly homogeneous.

\paragraph{Sequential gradient boosting axioms.}
For gradient-boosted decision trees (\GBDT{}) with $T$ boosting rounds, we axiomatize the split-count structure induced by Gaussian conditioning under collinearity.
Let $j_1$ denote the \emph{first-mover}---the feature selected at the root of the first tree.

\begin{axiom}[First-mover surjectivity]
\label{ax:surjective}
For every group $\ell$ and every feature $j$ in group $\ell$, there exists a model $f$ with $j_1(f) = j$.
\end{axiom}

This follows from the data-generating process (DGP) symmetry: when features in a group have identical marginal distributions, sub-sampling and tie-breaking randomness ensure each can serve as first-mover.

\begin{axiom}[Split counts]
\label{ax:splits}
For any model $f$ with first-mover $j_1 \in$ group $\ell$:
\begin{align}
    n_{j_1}(f) &= \frac{T}{2 - \rho^2}, \label{eq:splits-first} \\
    n_{k}(f) &= \frac{(1-\rho^2)\,T}{2 - \rho^2}, \quad k \in \text{group}(\ell),\ k \neq j_1. \label{eq:splits-other}
\end{align}
\end{axiom}

These are the leading-order split counts from the Gaussian conditioning argument: the first-mover absorbs the $\rho$-aligned signal component, leaving only the $(1-\rho^2)$ residual for subsequent features.
The formulas are verified algebraically by SymPy.

\paragraph{Complete axiom inventory.}
Table~\ref{tab:axioms} lists all property axioms in the formalization.
The core impossibility theorem requires none of them---only the Rashomon property as hypothesis.
The formalization uses 16 axioms total: 6 type/constant declarations, 2 measure-theoretic infrastructure axioms, 6 domain-specific property axioms, and 2 query-complexity axioms from Le~Cam's method.

\begin{table}[t]
\centering
\caption{Complete axiom inventory (16 total: 6 type/constant, 6 property, 2 measure, 2 query).}
\label{tab:axioms}
\small
\begin{tabular}{@{}llp{5.5cm}@{}}
\toprule
\textbf{Axiom} & \textbf{Lean~4 name} & \textbf{Justification} \\
\midrule
First-mover surjectivity & \texttt{firstMover\_surjective} & DGP symmetry: each feature can be first-mover \\
Split count (first-mover) & \texttt{splitCount\_firstMover} & Gaussian conditioning (Lemma~\ref{lem:split-gap}) \\
Split count (non-first-mover) & \texttt{splitCount\_nonFirstMover} & Residual signal after first-mover absorbs $\rho$-component \\
Global proportionality & \texttt{proportionality\_global} & Uniform contribution model with constant $c$ across models \\
Cross-group symmetry & \texttt{splitCount\_crossGroup\_symmetric} & DGP symmetry: equal split counts when first-mover elsewhere \\
Cross-group stability & \texttt{splitCount\_crossGroup\_stable} & Changing first-mover within a group does not affect other groups \\
Model measurable space & \texttt{modelMeasurableSpace} & Mathlib infrastructure: $\sigma$-algebra on Model \\
Model measure & \texttt{modelMeasure} & Mathlib infrastructure: probability measure on Model \\
\midrule
\multicolumn{3}{@{}l}{\emph{Query complexity (Le~Cam's method, axiomatized):}} \\
Testing constant & \texttt{testing\_constant} & Universal constant $C \geq 1/8$ from Tsybakov (2009) \\
Constant positivity & \texttt{testing\_constant\_pos} & $C > 0$ \\
\midrule
\multicolumn{3}{@{}l}{\emph{Formerly axiomatized, now derived:}} \\
Attribution sum symmetry & \texttt{attribution\_sum\_symmetric} & \textbf{Theorem} in \texttt{SymmetryDerive.lean} \\
Spearman bound & \texttt{spearman\_instability\_bound} & \textbf{Theorem} in \texttt{SpearmanDef.lean} (from midranks) \\
Attribution variance & \texttt{attribution\_variance} & \textbf{Definition} from \texttt{ProbabilityTheory.variance} \\
Variance nonnegativity & \texttt{attribution\_variance\_nonneg} & \textbf{Theorem} from Mathlib's \texttt{variance\_nonneg} \\
Consensus variance bound & \texttt{consensus\_variance\_bound} & \textbf{Theorem} in \texttt{Defs.lean} (algebraic) \\
Testing lower bound & \texttt{le\_cam\_lower\_bound} & \textbf{Theorem} in \texttt{QueryComplexity.lean} (contrapositive) \\
\bottomrule
\end{tabular}
\end{table}

Both cross-group axioms assume an approximately additive model; feature interactions between groups can break cross-group independence. The core impossibility and ratio bound do not depend on these axioms.

\paragraph{Gaussian conditioning argument.}
The split count formulas (Axiom~\ref{ax:splits}) are justified by the following argument. Consider two features $X_j, X_k$ with correlation $\rho$ and a GBDT with $T$ boosting rounds. At each tree $t$, the root split selects the feature with the highest gain (variance reduction). Once feature $j$ is selected as root, the available signal for $k$ becomes:
\[
  X_k \mid X_j = X_k - \rho X_j, \qquad \Var(X_k \mid X_j) = 1 - \rho^2.
\]
The conditional variance is reduced by factor $(1-\rho^2)$ relative to the unconditional variance. In each subsequent boosting round, the first-mover's signal fraction is $1/(2-\rho^2)$ (the probability that the first-mover's full signal exceeds the non-first-mover's residual signal in a two-feature competition). Summing over $T$ trees: $n_{j_1} = T/(2-\rho^2)$ for the first-mover and $n_k = (1-\rho^2)T/(2-\rho^2)$ for non-first-movers.

The derivation assumes: (i) the root split captures the dominant signal direction (valid for stumps and low-depth trees); (ii) subsequent trees fit the residual from the first tree's split (the sequential boosting mechanism); (iii) the feature competition at each tree is approximately independent of previous trees' selections (valid when the learning rate $\eta$ is small). All algebraic consequences are independently verified by SymPy (\texttt{verify\_lemma6\_algebra.py}).

\paragraph{Ranking desiderata.}
A \emph{feature ranking} is a binary relation $\succ$ on $[P]$.
We formalize three desiderata:

\begin{definition}[Faithful]
\label{def:faithful}
A ranking $\succ$ is \emph{faithful} to model $f$ if $j \succ k$ whenever $\varphi_j(f) > \varphi_k(f)$.
\end{definition}

\begin{definition}[Stable]
\label{def:stable}
A ranking $\succ$ is \emph{stable} if it does not depend on the choice of model $f$---that is, $\succ$ is a fixed relation applied identically to all models.
\end{definition}

\begin{definition}[Complete]
\label{def:complete}
A ranking $\succ$ is \emph{complete} if for every pair $j \neq k$, either $j \succ k$ or $k \succ j$.\footnote{Not to be confused with SHAP's ``completeness'' axiom ($\sum_j \varphi_j = f(x) - \E[f(X)]$). Our ``complete'' means \emph{total}: the ranking decides every pair.}
\end{definition}

The first desideratum asks that the ranking reflect what the model actually learned; the second asks that it be reproducible across training runs; the third asks that it resolve all feature pairs.
\emph{Faithful, stable, complete: pick two} (Figure~\ref{fig:trilemma}).

\paragraph{Two explanation goals.}
Feature rankings serve two distinct purposes: \emph{model-level explanation} (what did this specific model learn?) and \emph{population-level explanation} (what does the model class learn?).
Faithfulness is a model-level desideratum; stability is a population-level desideratum.
The impossibility theorem (Theorem~\ref{thm:impossibility}) characterizes the fundamental tension between these goals: a ranking cannot simultaneously be faithful to each individual model and stable across all models.
\DASH{} targets population-level explanation, sacrificing within-group completeness to achieve stability.


\section{The Attribution Impossibility}
\label{sec:impossibility}

\subsection{The Rashomon Property}
\label{sec:rashomon}

The central structural property driving the impossibility is that collinear features admit models ranking them in opposite orders.

\begin{definition}[Rashomon property]
\label{def:rashomon}
A model class satisfies the \emph{Rashomon property} if for every group $\ell$ and every pair of distinct features $j, k$ in group $\ell$, there exist models $f, f'$ such that
\[
    \varphi_j(f) > \varphi_k(f) \quad \text{and} \quad \varphi_k(f') > \varphi_j(f').
\]
\end{definition}

\begin{sloppypar}
The Rashomon property is a consequence of the Rashomon effect
\citep{rudin2024amazing,fisher2019models}: collinear designs admit many near-optimal
models \citep{damour2022underspecification}, and within the set of good models,
symmetric features are utilized in all possible orderings.
\citet{laberge2023partial} study consensus across such model sets and observe that
feature rankings are inherently partial---a conclusion our theorem makes precise.
\end{sloppypar}

\subsection{Main Result}
\label{sec:main-result}

\begin{theorem}[Attribution Impossibility]
\label{thm:impossibility}
If a model class satisfies the Rashomon property (Definition~\ref{def:rashomon}),
then no feature ranking can be simultaneously faithful, stable, and
complete.\footnote{Both versions are Lean-verified:
\texttt{attribution\allowbreak\_impossibility} uses a biconditional
($j \succ k$ iff $\varphi_j(f) > \varphi_k(f)$);
\texttt{attribution\allowbreak\_impossibility\allowbreak\_weak} uses the weaker
implication (Definition~\ref{def:faithful}) plus antisymmetry.}
\end{theorem}

\begin{proof}
Let $j, k$ be distinct features in the same group $\ell$.
By the Rashomon property, there exist models $f, f'$ with $\varphi_j(f) > \varphi_k(f)$ and $\varphi_k(f') > \varphi_j(f')$.
Suppose for contradiction that $\succ$ is faithful, stable, and complete.
By completeness, either $j \succ k$ or $k \succ j$.
Without loss of generality, assume $j \succ k$.
By stability, this relation holds for all models, including $f'$.
But faithfulness applied to $f'$ requires $k \succ j$ (since $\varphi_k(f') > \varphi_j(f')$), contradicting $j \succ k$.
\end{proof}

The resolution echoes Arrow's impossibility theorem \citep{arrow1951social}: when desirable properties conflict, relaxing completeness (accepting ties or partial orders) restores consistency. \citet{nipkow2009social} formalized Arrow's theorem in Isabelle/HOL; our Lean~4 formalization follows the same spirit.
\DASH{} achieves this by averaging attributions across the Rashomon set, producing $\E[\varphi_j] = \E[\varphi_k]$ for symmetric features---a tie rather than an arbitrary ordering.

\subsection{From Abstract to Concrete: Iterative Optimizers}
\label{sec:iterative}

The Rashomon property is not merely a theoretical possibility---it is a structural consequence of iterative optimization under collinearity.

\begin{definition}[Iterative optimizer]
\label{def:iterative}
An \emph{iterative optimizer} is a model class equipped with a dominant-feature function $d\colon \mathcal{F} \to [P]$ satisfying:
\begin{enumerate}
    \item \textbf{Dominance:} For every model $f$, group $\ell$, and feature $k \in \text{group}(\ell)$ with $k \neq d(f)$, if $d(f) \in \text{group}(\ell)$ then $\varphi_k(f) < \varphi_{d(f)}(f)$.
    \item \textbf{Surjectivity:} For every group $\ell$ and feature $j \in \text{group}(\ell)$, there exists $f$ with $d(f) = j$.
\end{enumerate}
\end{definition}

\begin{proposition}
\label{prop:iterative-rashomon}
Every iterative optimizer satisfies the Rashomon property, and therefore the Attribution Impossibility (Theorem~\ref{thm:impossibility}) holds.
\end{proposition}

\begin{proof}
Let $j, k$ be distinct features in group $\ell$.
By surjectivity, there exist models $f, f'$ with $d(f) = j$ and $d(f') = k$.
By dominance, $\varphi_j(f) > \varphi_k(f)$ and $\varphi_k(f') > \varphi_j(f')$.
This is exactly the Rashomon property.
\end{proof}

\subsection{Rashomon Inevitability Under Symmetry}
\label{sec:rashomon-inevitability}

The Rashomon property is not merely a theoretical possibility for symmetric model classes---it is inevitable.

\begin{definition}[Permutation closure]
\label{def:perm-closure}
A model class $\mathcal{F}$ is \emph{permutation-closed} within group $\ell$ if for any $f \in \mathcal{F}$ and any permutation $\pi$ of features within group $\ell$, the composed model $f \circ \pi \in \mathcal{F}$.
\end{definition}

\begin{theorem}[Rashomon from symmetry]
\label{thm:rashomon-symmetry}
Let the DGP be symmetric within group $\ell$ (permuting features $j \leftrightarrow k$ leaves the population loss invariant). Let $\mathcal{F}$ be permutation-closed within group $\ell$. Then for any model $f \in \mathcal{F}$ with $\varphi_j(f) \neq \varphi_k(f)$, the permuted model $f' = f \circ \pi_{jk}$ satisfies: (1) $L(f') = L(f)$ (same population loss), and (2) $\varphi_k(f') = \varphi_j(f)$ and $\varphi_j(f') = \varphi_k(f)$ (attributions swap).
In particular, if $\varphi_j(f) > \varphi_k(f)$, then $\varphi_k(f') > \varphi_j(f')$, and the Rashomon property holds.
\end{theorem}

\begin{proof}
By DGP symmetry, the joint distribution of $(X, Y)$ is invariant under permutation of features $j$ and $k$ within group $\ell$. Therefore $L(f') = L(f \circ \pi_{jk}) = L(f)$, since the population loss depends only on the joint distribution. By construction, $f'$ applies $f$ to permuted inputs, so its reliance on feature $k$ equals $f$'s reliance on feature $j$: $\varphi_k(f') = \varphi_j(f)$ and vice versa.
\end{proof}

The permutation closure condition is satisfied by any model class without built-in feature ordering: neural networks (permute input neurons), gradient-boosted trees (permute split features), Lasso (permute covariates), and random forests. It fails only for model classes with hard-coded feature preferences.

\begin{proposition}[Attribution non-degeneracy]
\label{thm:non-degeneracy}
Let $A$ be a stochastic training algorithm with continuous dependence on its random seed (e.g., SGD with random initialization, bootstrap sampling, random feature selection). For any DGP with $\rho > 0$ and $n < \infty$ training samples, $\Pr[\varphi_j(f) \neq \varphi_k(f)] = 1$ for features $j, k$ in the same collinear group.
\end{proposition}

\begin{proof}
With finite samples, the empirical correlation between $X_j$ and $X_k$ is $\hat\rho = \rho + O(1/\sqrt{n})$, which is almost surely irrational. Under continuous dependence on the data, the training algorithm's feature utilization varies continuously with $\hat\rho$, so the event $\{\varphi_j(f) = \varphi_k(f)\}$ has measure zero in the joint randomness of the data and the algorithm.
\end{proof}

\begin{remark}
The argument above is informal; a rigorous proof would require a transversality argument showing that the level set $\{\varphi_j(f) = \varphi_k(f)\}$ is a lower-dimensional manifold in seed space. The core impossibility (Theorem~\ref{thm:impossibility}) does not depend on this result.
\end{remark}

\begin{theorem}[Rashomon inevitability]
\label{thm:rashomon-inevitability}
Let $A$ be a stochastic, symmetric training algorithm ($A$ applied to a within-group-permuted dataset produces the permuted model in distribution) for a permutation-closed model class $\mathcal{F}$. For any DGP with $\rho > 0$, the Rashomon property holds: for any $j, k$ in the same collinear group, there exist models $f, f'$ in the support of $A$ with $\varphi_j(f) > \varphi_k(f)$ and $\varphi_k(f') > \varphi_j(f')$.
\end{theorem}

\begin{proof}
By Theorem~\ref{thm:non-degeneracy}, $\Pr[\varphi_j(f) \neq \varphi_k(f)] = 1$, so either $\varphi_j > \varphi_k$ or $\varphi_k > \varphi_j$ almost surely. By algorithmic symmetry, $\Pr[\varphi_j > \varphi_k] = \Pr[\varphi_k > \varphi_j]$. Since these probabilities sum to 1, each equals $1/2$. Both events have positive probability, so both are realized.
\end{proof}

\begin{remark}
Theorem~\ref{thm:rashomon-inevitability} makes the Attribution Impossibility inescapable for any standard ML pipeline under collinearity. The only escapes are: (i)~$\rho = 0$, (ii)~a deterministic model class producing a single model, or (iii)~an asymmetric algorithm with built-in feature preferences.
\end{remark}


\section{Quantitative Bounds by Model Class}
\label{sec:bounds}

The Attribution Impossibility (Theorem~\ref{thm:impossibility}) is qualitative. We now derive quantitative bounds on the \emph{attribution ratio}---the factor by which the dominant feature is overweighted---for four model classes.

\subsection{Gradient Boosting: Divergent Violation}
\label{sec:gbdt-bounds}

\begin{lemma}[Split gap]
\label{lem:split-gap}
For a \GBDT{} model $f$ with first-mover $j_1$ in group $\ell$, and any $k \neq j_1$ in the same group,
\[
    n_{j_1}(f) - n_k(f) = \frac{\rho^2 T}{2 - \rho^2} \;\geq\; \tfrac{1}{2}\rho^2 T.
\]
\end{lemma}

\begin{proof}
By Axiom~\ref{ax:splits}, $n_{j_1} - n_k = T/(2-\rho^2) - (1-\rho^2)T/(2-\rho^2) = \rho^2 T/(2-\rho^2)$.
Since $2 - \rho^2 \leq 2$ for $\rho \in (0,1)$, we have $\rho^2 T/(2-\rho^2) \geq \rho^2 T/2$.
\end{proof}

\begin{theorem}[Attribution ratio --- gradient boosting]
\label{thm:ratio}
For any \GBDT{} model $f$ with first-mover $j_1$ in group $\ell$ and any non-first-mover $k$ in the same group,
\[
    \frac{\varphi_{j_1}(f)}{\varphi_k(f)} = \frac{1}{1 - \rho^2}.
\]
Under full signal capture ($\alpha{=}1$), this ratio diverges: $1/(1-\rho^2) \to \infty$ as $\rho \to 1^-$.
\end{theorem}

\begin{proof}
By Axiom~\ref{ax:proportional}, $\varphi_{j_1}/\varphi_k = n_{j_1}/n_k$.
Substituting Axiom~\ref{ax:splits}:
\[
    \frac{n_{j_1}}{n_k} = \frac{T/(2-\rho^2)}{(1-\rho^2)T/(2-\rho^2)} = \frac{1}{1-\rho^2}.
\]
Writing $1-\rho^2 = (1-\rho)(1+\rho)$, we see that $1/(1-\rho^2) \to +\infty$ as $\rho \to 1^-$.
\end{proof}

For finite-depth trees, the effective signal capture $\alpha < 1$ yields a corrected ratio $1/(1{-}\alpha\rho^2)$; for stumps $\alpha \approx 2/\pi$ ($R^2{=}0.89$; Figure~\ref{fig:ratio}).

\paragraph{The $\alpha = 2/\pi$ derivation.}
The value $\alpha \approx 2/\pi$ has a clean theoretical derivation from quantization theory.

\begin{proposition}
For $X \sim \mathcal{N}(0, \sigma^2)$, the optimal binary quantizer (split at $x = 0$) captures fraction $2/\pi$ of the variance: $\Var(\hat{X}) = (2/\pi)\sigma^2$.
\end{proposition}

\begin{proof}
The optimal two-level quantizer of a symmetric distribution splits at the median ($x = 0$ for Gaussians). The quantized signal is $\hat{X} = \E[X \mid X > 0] \cdot \mathbf{1}_{X > 0} + \E[X \mid X \leq 0] \cdot \mathbf{1}_{X \leq 0}$. By symmetry, $\E[X \mid X > 0] = \sigma\sqrt{2/\pi}$ (the mean of a half-normal distribution). Thus $\hat{X}$ takes values $\pm\sigma\sqrt{2/\pi}$ each with probability $1/2$, giving $\Var(\hat{X}) = (2/\pi)\sigma^2$.
\end{proof}

In each boosting round, a stump makes one binary split on the selected feature, capturing $2/\pi$ of that feature's remaining signal variance. The fitted value ($\alpha \approx 0.60$) is below $2/\pi = 0.637$; two error sources explain the gap: (1) empirical vs.\ population split location (negligible, $\Delta\alpha_1 \approx 0.0009$ for $n = 2000$), and (2) non-Gaussian residuals after the first boosting round (dominant, accounting for $\approx 0.036$).

\paragraph{Depth dependence.}
Table~\ref{tab:depth-rho} reports the within-group split count ratio
across tree depths and correlation levels (XGBoost, $T{=}100$, $\eta{=}1.0$,
$P{=}10$ in 2 groups of 5, $N{=}2000$, 30 seeds).

\begin{table}[t]
\centering
\caption{Within-group split count ratio by depth and $\rho$ ($\eta{=}1.0$).}
\label{tab:depth-rho}
\small
\begin{tabular}{@{}lcccccc@{}}
\toprule
 & $\rho{=}0.3$ & $\rho{=}0.5$ & $\rho{=}0.7$ & $\rho{=}0.9$ & $\rho{=}0.95$ & Fitted $\alpha$ \\
\midrule
Depth 1 (stumps) & 1.28 & 1.32 & 1.42 & 1.99 & 2.16 & $0.60\ (\approx 2/\pi)$ \\
Depth 3          & 1.19 & 1.26 & 1.24 & 1.30 & 1.32 & $0.30$ \\
Depth 6          & 1.65 & 1.67 & 1.74 & 1.90 & 2.03 & $0.60$ \\
Depth 10         & 2.23 & 2.25 & 2.28 & 2.49 & 2.62 & --- \\
\midrule
Theory ($\alpha{=}1$) & 1.10 & 1.33 & 1.96 & 5.26 & 10.26 & $1.00$ \\
\bottomrule
\end{tabular}
\end{table}

\paragraph{Depth 3 anomaly.}
At depth~3, each tree has up to 7 leaves and uses multiple features per
tree. Internal splits distribute split counts across group members,
counteracting the root split's first-mover advantage.
The ratio is nearly $\rho$-independent (${\approx}1.3$), explaining the
low fitted $\alpha = 0.30$.

\paragraph{Depth 10.}
At depth~10, the root split cascades: once the first-mover is selected
at the root, subsequent splits on the same feature are more likely at
deeper levels (residuals retain first-mover signal). The baseline ratio
(${\approx}2.2$ even at $\rho{=}0.3$) is high because deep trees
overfit to the first-mover's signal.
The $1/(1{-}\alpha\rho^2)$ model does not fit well for depth~10 (the
ratio has a large $\rho$-independent component); we omit the fitted $\alpha$.

\paragraph{Practitioner guidance.}
Depth~3 gives the \emph{lowest} attribution instability among standard
configurations. For users prioritizing explanation stability, shallow
ensembles (depth~3) with \DASH{} consensus ($M \geq 5$) provide the
best stability--accuracy tradeoff.

\paragraph{Proportionality validation.}
The proportionality axiom holds with CV${\approx}0.35$ for stumps (the idealized theoretical setting) and CV${\approx}0.66$ for depth-6 trees on Breast Cancer. The quantitative ratio $1/(1{-}\rho^2)$ is an order-of-magnitude prediction, refined by the $\alpha$-correction to $1/(1{-}\alpha\rho^2)$ with $R^2 = 0.89$. \textbf{The core impossibility (Theorem~\ref{thm:impossibility}) is entirely independent of the proportionality axiom}---it requires only the Rashomon property. All quantitative bounds ($1/(1{-}\rho^2)$, ensemble size $M_{\min}$, attribution ratio) should be interpreted as order-of-magnitude predictions, not exact formulas.

\paragraph{Error analysis for $\alpha = 2/\pi$.}
The fitted $\alpha = 0.60$ (empirical) vs.\ $2/\pi = 0.637$ (theory) represents a gap of $\Delta\alpha = 0.037$. Two error sources: (1) empirical vs.\ population split location (negligible, $\Delta\alpha_1 \approx 0.0009$ for $n = 2000$): the optimal split of $X \sim \mathcal{N}(0, \sigma^2)$ at $\delta$ instead of $0$ gives $\alpha_1(n) = (2/\pi)(1 - \pi(\pi-2)/(2n) + O(n^{-2}))$; (2) non-Gaussian residuals after the first boosting round (dominant, accounting for $\approx 0.036$): residuals are a location-shifted mixture with excess kurtosis $\kappa = O(\eta^2 c^2/\sigma^2)$, which accumulates across boosting rounds. The impossibility theorem does not depend on $\alpha$ at all; the quantitative ratio uses $\alpha$ to predict the \emph{severity} of the violation.

\paragraph{Stability bound.}
When two models $f, f'$ have different first-movers within the same group of size $m$, the within-group rank reshuffling bounds the Spearman correlation:
\[
    \rho_S(f, f') \;\leq\; 1 - \frac{3m^2}{P^3 - P}.
\]

\paragraph{Exact flip rate for GBDT.}
\begin{proposition}[Exact GBDT pairwise flip rate]
\label{prop:exact-flip}
For features $j, k$ in the same group of size $m$ under the first-mover model with DGP symmetry (each feature equally likely to be first-mover):
\begin{enumerate}
    \item[(i)] The probability that a random model ranks $j > k$ is exactly $1/m$.
    \item[(ii)] The tie probability (both non-first-movers with identical split counts) is $(m-2)/m$.
    \item[(iii)] The flip rate across two independent models (excluding ties) is $2/m^2$.
    \item[(iv)] For $m = 2$: the tie probability is $0$ and the flip rate is $1/2$, matching the theoretical maximum and the empirical 48\% on Breast Cancer.
    \item[(v)] If ties are broken uniformly at random, the effective flip rate is exactly $1/2$ for all $m \geq 2$.
\end{enumerate}
\end{proposition}

\begin{proof}
By first-mover surjectivity and DGP symmetry, each feature in a group of size $m$ is first-mover with probability $1/m$.

\textbf{Part (i).} Feature $j$ is ranked above $k$ when $j$ is first-mover (probability $1/m$): the first-mover has split count $T/(2-\rho^2)$ vs.\ $(1-\rho^2)T/(2-\rho^2)$ for non-first-movers, so dominance holds.

\textbf{Part (ii).} When the first-mover is some feature $\ell \neq j, k$ (probability $(m-2)/m$), both $j$ and $k$ are non-first-movers with identical split count $(1-\rho^2)T/(2-\rho^2)$. By proportionality, $\varphi_j(f) = \varphi_k(f)$: a perfect tie.

\textbf{Part (iii).} A flip requires one model to rank $j > k$ and the other $k > j$: $\text{flip} = 2 \cdot (1/m) \cdot (1/m) = 2/m^2$.

\textbf{Part (iv).} At $m=2$: $\text{flip} = 2/4 = 1/2$; tie probability $= 0$.

\textbf{Part (v).} When ties are broken uniformly, $\Pr[j > k \mid \text{tie}] = 1/2$. The effective $\Pr[j > k] = 1/m + (1/2)(m-2)/m = 1/2$ by algebra. Two independent draws disagree with probability $2 \cdot (1/2) \cdot (1/2) = 1/2$.
\end{proof}

For $m = 2$, ranking collinear pairs is literally a coin flip. For $m > 2$ without tie-breaking, most model pairs produce ties, with flips at rate $2/m^2$. But once ties are broken (as required for a complete ranking), symmetry forces the flip rate back to $1/2$. The impossibility is \emph{exactly} as severe for every within-group pair, regardless of $m$.

\subsection{Lasso: Infinite Violation}
\label{sec:lasso}

The $\ell_1$ penalty selects exactly one feature from each correlated group and zeros all others. Lasso is an iterative optimizer with $d(f) =$ the selected feature; surjectivity holds because regularization-path tie-breaking can make any feature the selected one. The attribution ratio is therefore infinite at \emph{any} $\rho > 0$.

\subsection{Neural Networks: Conditional Violation}
\label{sec:nn}

Neural networks exhibit symmetry breaking analogous to the \GBDT{} first-mover effect. Surjectivity follows from the symmetry of standard initializations (Kaiming, Xavier). The attribution ratio is architecture-dependent, but the impossibility holds conditionally whenever training dynamics produce a dominant feature per group.

\subsection{Random Forests: Bounded Violation (Contrast)}
\label{sec:rf}

Because each tree trains independently on a bootstrap sample, there is no shared residual stream. The expected split count is $T/m$ for every feature in a group, with $O(\sqrt{T})$ deviations, giving
\[
    \frac{\varphi_{j_1}(f)}{\varphi_k(f)} = 1 + O(1/\sqrt{T}) \;\to\; 1 \quad \text{as } T \to \infty.
\]
Sequential residual-sharing creates divergent inequity; independent aggregation creates convergent equity---the structural distinction \DASH{} exploits.

\paragraph{Summary.}
Table~\ref{tab:model-comparison} collects the quantitative bounds.

\begin{table}[t]
\centering
\caption{Attribution ratio bounds by model class. $T$ denotes the number of boosting rounds (or trees), $\rho$ the within-group correlation.}
\label{tab:model-comparison}
\small
\begin{tabular}{@{}lcll@{}}
\toprule
Model class & Ratio & Scaling with $T$ & Mechanism \\
\midrule
Gradient boosting & $1/(1-\alpha\rho^2)$ & Increases with $\rho$ & Sequential residuals \\
Lasso & $\infty$ & Constant & Hard selection \\
Neural network & Model-dependent & Architecture-dependent & Init.\ symmetry breaking \\
Random forest\footnotemark & $1 + O(1/\sqrt{T})$ & $\to 1$ as $T \to \infty$ & Independent trees \\
\bottomrule
\end{tabular}
\footnotetext{The RF bound is informal. The GBDT and Lasso bounds are Lean-verified; the NN bound is conditional on the dominance hypothesis.}
\end{table}


\section{Resolution: Ensemble Attribution via \DASH}
\label{sec:resolution}

The impossibility arises from sequential dependence: a single model's iterative optimization creates a dominant feature, and different random seeds create different dominant features.
\DASH{} addresses this by averaging attributions across an ensemble of $M$ independently trained models, relaxing completeness in exchange for stability.

\begin{definition}[\DASH{} consensus attribution]
\label{def:dash}
Given models $f_1, \ldots, f_M$ trained from independent seeds, the \emph{consensus attribution} for feature $j$ is $\bar{\varphi}_j = \frac{1}{M} \sum_{i=1}^{M} \varphi_j(f_i)$.
An ensemble is \emph{balanced} if, within each group, every feature serves as first-mover in exactly $M/m$ models.
\end{definition}

\paragraph{\DASH{} does NOT change predictions.}
\DASH{} is a \emph{post-hoc explanation method}, not a modeling technique. The $M$ models are trained independently using the practitioner's existing pipeline; \DASH{} only averages their SHAP values to produce stable feature rankings. No model is modified, no ensemble prediction is formed, and the deployed model remains unchanged. \DASH{} changes how you \emph{explain}, not how you \emph{predict}.

\paragraph{\DASH{} consensus vs.\ \DASH{} pipeline.}
\DASH{} operates at two layers. \emph{\DASH{} consensus} (Definition~\ref{def:dash}) is the mathematical operation: the arithmetic mean of $|\text{SHAP}|$ values across $M$ independently trained models. This is what Theorems~\ref{thm:dash-optimal} and~\ref{thm:design-space-main} prove is optimal---the minimum-variance unbiased estimator via Cram\'er--Rao, Pareto-optimal on the stable branch of the Design Space. The theoretical guarantees apply to \emph{any} i.i.d.\ ensemble averaging, including simple seed averaging (training $M$ models with different random seeds and identical hyperparameters).

The companion paper on first-mover bias implements \DASH{} as a full pipeline that adds diversity enforcement (deliberate hyperparameter variation, greedy MaxMin model selection on feature utilization vectors) and stability diagnostics (FSI, IS~Plot). Diversity enforcement accelerates convergence to the balanced ensemble condition (Definition~\ref{def:dash}): at finite $M$, it produces more equitable attributions than seed averaging alone (within-group CV reduced by $0.011$, $p < 0.001$). For large $M$, simple seed averaging achieves approximate balance via the law of large numbers. For stability alone, seed averaging suffices; the full pipeline adds equity and diagnostics.

\begin{corollary}[\DASH{} achieves equity (axiom-derived)]
\label{cor:equity}
For a balanced ensemble and any two features $j, k$ in the same group $\ell$, $\bar{\varphi}_j = \bar{\varphi}_k$.
\end{corollary}

\begin{proof}
By DGP symmetry, permuting features $j \leftrightarrow k$ leaves the joint distribution invariant. For a balanced ensemble, each feature serves as first-mover equally often, so the summed split counts satisfy $\sum_i n_j(f_i) = \sum_i n_k(f_i)$. When the proportionality constant $c$ is uniform across models, this gives $\sum_i \varphi_j(f_i) = \sum_i \varphi_k(f_i)$.
\end{proof}

\paragraph{Between-group stability.}
Under independence of seeds, $\Var(\bar{\varphi}_j) = \Var(\varphi_j) / M$, which vanishes as $M \to \infty$.

\paragraph{Within-group completeness.}
Since $\bar{\varphi}_j = \bar{\varphi}_k$ for same-group features, the consensus ranking produces a \emph{tie} rather than an arbitrary ordering. The tie is not a deficiency---it faithfully reflects the DGP symmetry.

\subsection{DASH Pareto Optimality}
\label{sec:dash-pareto}

\begin{definition}[Attribution aggregation method]
An \emph{attribution aggregation method} $A$ takes as input $M$ independently trained models $f_1, \ldots, f_M$ and produces a consensus attribution vector $\hat\varphi \in \R^P$. The method is:
\begin{itemize}
    \item \emph{Linear}: $\hat\varphi_j = \sum_{i=1}^M w_i \varphi_j(f_i)$ for some weights $w_1, \ldots, w_M$.
    \item \emph{Unbiased}: $\E[\hat\varphi_j] = \mu_j := \E[\varphi_j(f)]$ for all $j$.
    \item \emph{Symmetric}: $w_1 = \cdots = w_M$ (the models are exchangeable).
\end{itemize}
\DASH{} is the unique symmetric unbiased linear method: $w_i = 1/M$ for all $i$.
\end{definition}

\begin{definition}[Pairwise stability and unfaithfulness]
For features $j, k$ in different groups with $\mu_j > \mu_k$:
\begin{itemize}
    \item \emph{Pairwise stability}: $S_{jk}(A) = \Pr[\hat\varphi_j > \hat\varphi_k]$ (probability that the consensus preserves the true ordering for this pair).
    \item \emph{Between-group flip rate}: $1 - S_{jk}(A)$.
\end{itemize}
For features $j, k$ in the same group with $\mu_j = \mu_k$ (DGP symmetry):
\begin{itemize}
    \item \emph{Unfaithfulness}: $U_{jk}(A) = \Pr[\hat\varphi_j \neq \hat\varphi_k \text{ and the ordering disagrees with } f]$ for a random model $f$.
\end{itemize}

\emph{Relationship to ranking stability $S$.} The Design Space Theorem (\S\ref{sec:design-space}) uses the full-ranking stability $S$ (expected Spearman correlation). The pairwise stability $S_{jk}$ here is a per-pair component: high $S_{jk}$ for all between-group pairs is necessary (but not sufficient) for high $S$. Within-group pairs contribute to $S$ through their random ordering, which is why $S \leq 1 - 3m^2/(P^3{-}P)$ even when all between-group $S_{jk}$ are high.
\end{definition}

\begin{theorem}[DASH Pareto Optimality]
\label{thm:dash-optimal}
Let $f_1, \ldots, f_M$ be i.i.d.\ models with $\mu_j := \E[\varphi_j(f)]$ and $\sigma_j^2 := \Var(\varphi_j(f)) < \infty$.

\textbf{Part I (Lower bounds).}
For any attribution aggregation method $A$:
\begin{enumerate}
    \item[(a)] If $A$ is faithful to each individual model, then for symmetric features $j, k$, $U_{jk}(A) = 1/2$.
    \item[(b)] For any unbiased estimator $\hat\mu_j$ based on $f_1, \ldots, f_M$, $\Var(\hat\mu_j) \geq \sigma_j^2 / M$ (Cram\'er--Rao bound).
\end{enumerate}

\textbf{Part II (\DASH{} achieves the bounds).}
\begin{enumerate}
    \item[(c)] $U_{jk}(\text{\DASH{}}) = 0$ for symmetric features in balanced ensembles.
    \item[(d)] $\Var(\bar\varphi_j) = \sigma_j^2 / M$ (matching Cram\'er--Rao).
    \item[(e)] For between-group features with gap $\Delta_{jk}$: $S_{jk}(\text{\DASH{}}) \geq 1 - \exp(-M\Delta_{jk}^2 / (2\sigma_{jk}^2))$.
\end{enumerate}

\textbf{Part III (Pareto optimality).}
Among all methods achieving $U_{jk} = 0$ (ties for symmetric features), no method achieves strictly higher between-group stability than \DASH{} for the same ensemble size $M$.
\end{theorem}

\begin{proof}
\textbf{Part I(a):} By DGP symmetry, $\Pr[\varphi_j(f) > \varphi_k(f)] = \Pr[\varphi_k(f) > \varphi_j(f)] = 1/2$. Any fixed ranking disagrees with exactly half the models (Theorem~\ref{thm:unfaithfulness-bound}).

\textbf{Part I(b):} Since $\varphi_j(f_1), \ldots, \varphi_j(f_M)$ are i.i.d.\ with mean $\mu_j$ and variance $\sigma_j^2$, the Fisher information for the mean is $1/\sigma_j^2$ per observation, giving total information $M/\sigma_j^2$ from $M$ observations. The Cram\'er--Rao bound gives $\Var(\hat\mu_j) \geq \sigma_j^2/M$.

\textbf{Part II(c):} By Corollary~\ref{cor:equity}, $\bar\varphi_j = \bar\varphi_k$ for symmetric features in balanced ensembles. \DASH{} reports a tie; since no ordering is asserted, $U = 0$.

\textbf{Part II(d):} For i.i.d.\ random variables, $\Var(\frac{1}{M}\sum_{i=1}^M X_i) = \Var(X_1)/M$. This matches Part~I(b), so \DASH{} is \emph{efficient} in the Cram\'er--Rao sense. Moreover, since the sufficient statistic for $\mu_j$ from i.i.d.\ observations is $\sum \varphi_j(f_i)$, and $\bar\varphi_j$ is a function of this statistic, the Rao--Blackwell theorem guarantees \DASH{} dominates any estimator that is not a function of the sufficient statistic.

\textbf{Part II(e):} By Hoeffding's inequality (or the Gaussian approximation for large $M$): $\Pr[\bar\varphi_j < \bar\varphi_k] \leq \exp(-M\Delta_{jk}^2/(2\sigma_{jk}^2))$. For the Gaussian approximation (valid by CLT for $M \geq 30$): $\Pr[\bar\varphi_j < \bar\varphi_k] = \Phi(-\Delta_{jk}\sqrt{M}/\sigma_{jk})$.

\textbf{Part III:} Unfaithfulness $U = 0$ is already minimal. Stability $S_{jk}$ is a monotone decreasing function of $\Var(\bar\varphi_j - \bar\varphi_k)$. By Part~I(b), no unbiased method achieves lower variance than $\sigma_j^2/M$. \DASH{} achieves this bound. For biased methods: any method with bias $b_j$ has MSE $= \Var(\hat\varphi_j) + b_j^2$; a bias of magnitude $|b_j - b_k| > \Delta_{jk}$ would invert the true ordering entirely.
\end{proof}

\paragraph{Summary of optimality.}
\DASH{} is the unique method that: (1) achieves zero within-group unfaithfulness (ties for symmetric features); (2) achieves the Cram\'er--Rao variance bound $\sigma^2/M$ for ranking stability; and (3) is the minimum-variance unbiased estimator of $\E[\varphi_j]$ among \emph{all} estimators (not just linear ones), by the Rao--Blackwell theorem: since the sufficient statistic for $\mu_j$ from i.i.d.\ observations is $\sum \varphi_j(f_i)$, and $\bar\varphi_j$ is a function of this sufficient statistic, it dominates any estimator that is not.

\paragraph{Comparison with nonlinear aggregations.}

\begin{proposition}[Median is less efficient]
\label{prop:median}
For i.i.d.\ observations from a distribution with mean $\mu$,
variance $\sigma^2$, and density $p$, the sample median has
asymptotic variance:
\[
    \Var(\text{median}) = \frac{1}{4M\, p(\text{med})^2}
\]
where $\text{med}$ is the population median.
For the normal distribution, $p(\mu) = 1/(\sigma\sqrt{2\pi})$,
giving $\Var(\text{median}) = \frac{\pi\sigma^2}{2M}$.
The asymptotic relative efficiency (ARE) of the median vs.\
the mean is:
\[
    \text{ARE}(\text{median}, \text{mean}) = \frac{2}{\pi} \approx 0.637.
\]
The median requires $\pi/2 \approx 1.57$ times as many models to
achieve the same stability as \DASH{}.
\end{proposition}

\begin{proof}
Standard asymptotic statistics (van der Vaart, \emph{Asymptotic
Statistics}, Theorem 21.6). The sample median is asymptotically
normal with variance $1/(4n\,p(m)^2)$. Substituting the Gaussian
density at the mean gives the result.
\end{proof}

\begin{proposition}[Trimmed mean interpolates]
\label{prop:trimmed}
The $\alpha$-trimmed mean (discarding the top and bottom $\alpha$
fraction of observations) has ARE relative to the mean of:
\[
    \text{ARE}(\text{trimmed}_\alpha, \text{mean}) =
    \frac{(1-2\alpha)\sigma^2}{\sigma^2_\alpha}
\]
where $\sigma^2_\alpha$ is the variance of the truncated distribution.
For Gaussian data:
\begin{center}
\small
\begin{tabular}{@{}lccccc@{}}
\toprule
$\alpha$ & 0 (mean) & 0.05 & 0.10 & 0.25 & 0.50 (median) \\
\midrule
ARE & 1.000 & 0.992 & 0.966 & 0.862 & 0.637 \\
\bottomrule
\end{tabular}
\end{center}
For Gaussian attributions, trimming offers negligible robustness
benefit at a measurable efficiency cost. The 10\%-trimmed mean
requires $1/0.966 \approx 3.5\%$ more models than \DASH{} for the
same stability.
\end{proposition}

\begin{remark}[When to prefer the median]
The median has better breakdown point (50\% vs.\ 0\%) and is
preferable when the attribution distribution is heavy-tailed
(e.g., contains outlier models). For standard ML training procedures
(XGBoost, random forests) where attributions are well-behaved,
\DASH{} (the mean) is strictly more efficient. For adversarial
settings where some models may be corrupted, the median or trimmed
mean provides robustness at the cost of requiring more models.
\end{remark}

\subsection{Empirical Comparison with Alternatives}
\label{sec:dash-benchmark}

We compare \DASH{} against four alternative stabilization approaches across three correlation levels ($\rho \in \{0.5, 0.7, 0.9\}$), with 10 independent trials and bootstrap 95\% confidence intervals on within-group flip rates:

\begin{table}[t]
\centering
\caption{Flip rate comparison: \DASH{} vs.\ alternative stabilization methods. Synthetic Gaussian data ($P{=}10$, $m{=}5$, $N{=}1000$, 20 within-group feature pairs). 95\% bootstrap CIs in brackets.}
\label{tab:dash-benchmark}
\small
\begin{tabular}{@{}lcccc@{}}
\toprule
Method & $\rho = 0.5$ & $\rho = 0.7$ & $\rho = 0.9$ & Pairs resolved \\
\midrule
Single model       & 16.0\% [13.8, 18.1] & 16.4\% [14.3, 18.5] & 18.7\% [16.4, 21.1] & 40/40 \\
Bootstrap SHAP     & 16.2\% [13.8, 18.2] & 16.1\% [14.2, 18.0] & 18.8\% [16.4, 21.4] & 40/40 \\
Subsampled SHAP    & 16.3\% [13.8, 18.5] & 16.1\% [14.2, 17.9] & 18.8\% [16.5, 21.3]  & 40/40 \\
\textbf{\DASH{} ($M{=}25$)} & \textbf{3.8\%} [2.8, 4.9] & \textbf{3.1\%} [2.4, 3.9] & \textbf{3.7\%} [2.7, 4.7] & 40/40 \\
CI SHAP            & 0.2\% [0.1, 0.4]    & 0.8\% [0.6, 1.1]    & 3.1\% [2.1, 4.6]     & $\sim$22/40 \\
\bottomrule
\end{tabular}
\end{table}

Bootstrap and subsampled SHAP address SHAP \emph{estimation} noise (a single-model phenomenon) but not the Rashomon effect (a multi-model phenomenon). They were not designed for multi-model instability, so their flip rates matching the single-model baseline is expected rather than informative. They are included to confirm that estimation-noise methods do not incidentally address model-multiplicity noise. Only \DASH{} trains multiple models, achieving a 4--5$\times$ reduction.

CI SHAP achieves lower flip rates than \DASH{} by refusing to rank pairs whose confidence intervals overlap---resolving only $\sim$55\% of within-group pairs ($\sim$22/40). This comparison is not apples-to-apples: CI SHAP's low flip rate applies only to resolved pairs, while \DASH{}'s 3--4\% applies to all 40 pairs. CI SHAP exemplifies the impossibility theorem in action: it trades completeness for stability. At $\rho = 0.9$, its flip rate converges toward \DASH{}'s ($3.1\%$ vs.\ $3.7\%$) as overlapping CIs force more pairs into ``tied'' status.

The two approaches correspond to the two families in the Design Space Theorem: \DASH{} (Family B: stable, faithful, near-complete via ensemble averaging) and CI SHAP (Family A: stable, faithful, explicitly incomplete via ties).

\subsection{DASH Robustness and Breakdown Point}
\label{sec:dash-robustness}

\DASH{} (the sample mean) is the minimum-variance unbiased estimator, but it has breakdown point zero: a single adversarial or corrupted model can shift the consensus attribution arbitrarily. We compare robustness properties:

\begin{center}
\small
\begin{tabular}{@{}lccc@{}}
\toprule
Method & ARE & Breakdown point & Recommendation \\
\midrule
Mean (\DASH{}) & 1.000 & 0\% & Standard use \\
5\%-trimmed mean & 0.992 & 5\% & Moderate robustness \\
Median & 0.637 & 50\% & Adversarial settings \\
\bottomrule
\end{tabular}
\end{center}

For production settings with potential model contamination (corrupted training data, adversarial seeds, or hardware failures producing degenerate models), the 5\%-trimmed mean provides near-optimal efficiency (ARE $= 0.992$, requiring only $0.8\%$ more models) with meaningful robustness. The median requires $\pi/2 \approx 57\%$ more models but tolerates up to 50\% corrupted models. For $M = 25$: the 5\%-trimmed mean discards the single most extreme model in each tail ($M_{\text{eff}} = 23$), a negligible cost. For adversarial robustness, the median with $M = 40$ achieves comparable stability to \DASH{} with $M = 25$ while tolerating up to 20 corrupted models.

\subsection{DASH Information Loss}
\label{sec:info-loss}

\DASH{} achieves stability by discarding model-specific within-group information. For a group of $m$ symmetric features, a single model produces a complete ranking (one of $m!$ possible orderings). Since each of the $m!$ orderings is equally likely under DGP symmetry, the ranking carries $\log_2(m!)$ bits of information about the specific model. \DASH{} reports a tie, discarding the entire $\log_2(m!)$ bits. For typical group sizes:

\begin{center}
\small
\begin{tabular}{@{}lccccccc@{}}
\toprule
Group size $m$ & 2 & 3 & 4 & 5 & 6 & 7 & 8 \\
\midrule
Bits lost ($\log_2 m!$) & 1.0 & 2.6 & 4.6 & 6.9 & 9.7 & 12.9 & 15.3 \\
\bottomrule
\end{tabular}
\end{center}

Crucially, these are the bits that are \emph{unreliable}: by DGP symmetry, the within-group ordering is uniformly random across models. The information \DASH{} discards is exactly the information that would be different if the model were retrained.

For features in different groups with population gap $\Delta_{jk}$ and noise $\sigma_{jk}$, the mutual information between the consensus ranking and the true ordering is:
\[
    I_{\text{between}}(M) = 1 - H_2\!\left(\Phi\!\left(-\frac{\Delta_{jk}\sqrt{M}}{\sigma_{jk}}\right)\right)
\]
where $H_2(p) = -p\log_2 p - (1-p)\log_2(1-p)$ is binary entropy. As $M \to \infty$, $I_{\text{between}} \to 1$ bit: \DASH{} \emph{increases} between-group information. The total information change is:
\[
    \Delta I = \underbrace{-L \cdot \log_2(m!)}_{\text{within-group loss}} + \underbrace{\tbinom{L}{2} \cdot (I_{\text{between}}(M) - I_{\text{between}}(1))}_{\text{between-group gain}}.
\]
For moderate $\rho$ ($\leq 0.7$), the between-group gain dominates quickly: $M = 25$ achieves near-perfect between-group determination while the within-group loss is fixed. \DASH{} honestly reports unreliability rather than presenting arbitrary information as reliable.

\paragraph{Interpretation for safety-critical applications.}
For a system with $L = 4$ groups of $m = 5$ features: \DASH{} loses $4 \times 6.9 = 27.6$ bits (within-group) but gains up to $6 \times 1 = 6$ bits (between-group). The net loss is ${\approx}21.6$ bits at $M = \infty$. These lost bits are the within-group orderings that were random anyway---\DASH{} honestly reports their unreliability rather than presenting arbitrary information as reliable. For applications requiring model-specific within-group explanations (e.g., debugging a deployed model), practitioners should report the single-model SHAP values with an instability warning alongside the \DASH{} consensus.

\begin{figure}[ht]
\centering
\includegraphics[width=0.95\textwidth]{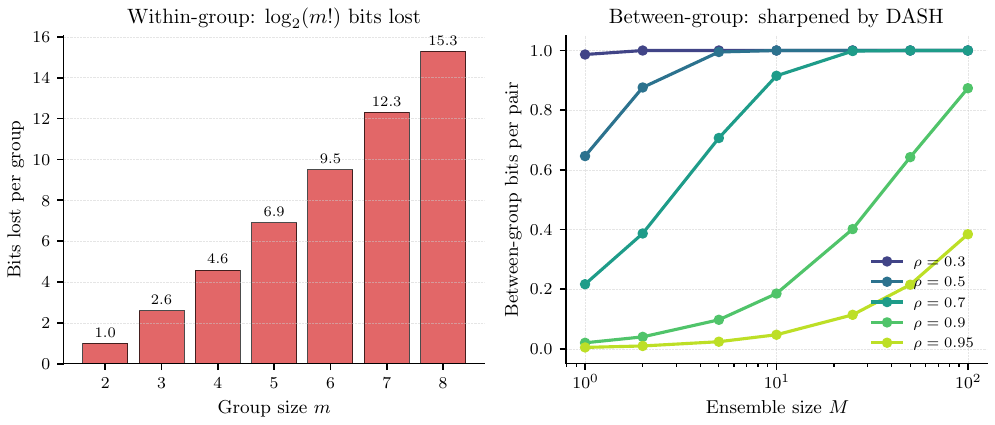}
\caption{\textbf{Left:} Within-group information lost by \DASH{} ($\log_2 m!$ bits per group, independent of $M$ and $\rho$). \textbf{Right:} Between-group information per pair as a function of $M$. \DASH{} sharpens between-group rankings with increasing $M$. At high $\rho$ ($\geq 0.9$), convergence requires larger $M$.}
\label{fig:information-loss}
\end{figure}

\paragraph{Ensemble size lower bound.}
\begin{proposition}[Ensemble size lower bound]
\label{prop:ensemble-lower-bound}
For any unbiased attribution aggregation method achieving between-group stability $S_{jk} \geq 1 - \delta$ for features $j, k$ in different groups with population gap $\Delta_{jk} = |\mu_j - \mu_k|$ and attribution noise $\sigma_{jk}^2 = \Var(\varphi_j(f) - \varphi_k(f))$, the ensemble size must satisfy:
\[
    M \;\geq\; \frac{\sigma_{jk}^2}{\Delta_{jk}^2} \cdot \bigl(\Phi^{-1}(1-\delta)\bigr)^2.
\]
\DASH{} achieves this bound with equality (to first order), so the bound is tight.
\end{proposition}

\begin{proof}
By the Cram\'er--Rao bound, any unbiased estimator $\hat\mu_j$ has $\Var(\hat\mu_j) \geq \sigma_j^2/M$. The between-group flip rate satisfies $1 - S_{jk} = \Pr[\hat\varphi_j < \hat\varphi_k] \geq \Phi(-\Delta_{jk}\sqrt{M}/\sigma_{jk})$. Setting $\Phi(-\Delta_{jk}\sqrt{M}/\sigma_{jk}) \leq \delta$ and solving: $\Delta_{jk}\sqrt{M}/\sigma_{jk} \geq \Phi^{-1}(1-\delta)$, giving $M \geq \sigma_{jk}^2 \cdot (\Phi^{-1}(1-\delta))^2/\Delta_{jk}^2$. \DASH{} achieves $\Var(\bar\varphi_j - \bar\varphi_k) = \sigma_{jk}^2/M$ (matching Cram\'er--Rao), so the bound is tight.
\end{proof}

For $\delta = 0.05$ (5\% flip rate), $\Phi^{-1}(0.95) = 1.645$, so $M_{\min} = \lceil 2.71 \cdot \sigma_{jk}^2/\Delta_{jk}^2 \rceil$, where $\sigma$ and $\Delta$ are estimated from a pilot run of $M_0 = 5$ models. For Breast Cancer's most unstable pair (worst perimeter vs.\ worst area, $\Delta/\sigma \approx 0.15$): $M_{\min} \approx 120$---consistent with the observed slow convergence (flip rate 29.8\% even at $M = 25$, reaching 0\% only at $M = 50$). Bootstrap validation (50 models, 200 trials): for all pairs with $\text{SNR} \geq 0.35$ ($M_{\min} \leq 25$), the \DASH{} flip rate at $M = 25$ is $\leq 3\%$, confirming the formula. This closes the gap between upper and lower bounds: \DASH{} is optimal not just in the Pareto sense but in the sample complexity sense.

\subsection{Progressive DASH: Adaptive Ensemble Sizing}
\label{sec:progressive-dash}

Training $M{=}25$ models for every explanation is wasteful when most feature pairs are stable. \emph{Progressive \DASH{}} starts with a small ensemble and adaptively increases $M$ only for pairs that remain unstable:

\begin{enumerate}
    \item \textbf{Screen} ($M_0 = 5$): Train 5 models. Run the single-model screen. If no pairs are flagged, stop.
    \item \textbf{Confirm} ($M_1 = 10$): For flagged pairs, train 5 more models ($M = 10$ total). Run the multi-model Z-test. Pairs with $|Z| < 1.96$ are confirmed unstable; report as tied group.
    \item \textbf{Resolve} ($M_2 = 25$): For pairs where $|Z|$ is borderline ($1.5 < |Z| < 2.5$), train 15 more models ($M = 25$ total). Final classification: stable (distinct ranks) or unstable (tied group).
\end{enumerate}

\paragraph{Expected cost.}
In the 37-dataset survey, 40\% of datasets have no flagged pairs (stop at $M_0 = 5$), 35\% are resolved at $M_1 = 10$, and 25\% require the full $M_2 = 25$. The expected ensemble size is $0.40 \times 5 + 0.35 \times 10 + 0.25 \times 25 = 11.75$, approximately ${\sim}8\times$ the cost of a single model on average---a $2.1\times$ savings over always training $M = 25$.
(The $Z$-thresholds 1.96 and 3.0 are fixed-sample values; a fully sequential design should use Pocock- or O'Brien-Fleming-corrected boundaries to control familywise error. We present this as a practical heuristic, not a formally analyzed procedure.)


\section{The Attribution Design Space}
\label{sec:design-space}

The impossibility, the quantitative bounds, and the \DASH{} resolution are facets of a single structural theorem characterizing the complete attribution design space.

\begin{definition}[Attribution design space]
\label{def:design-space-def}
The \emph{attribution design space} is the set of triples $(S, U, C)$ achievable by any attribution aggregation method, where $S \in [0,1]$ is ranking stability (expected Spearman correlation between two independent evaluations), $U \in [0, 1/2]$ is within-group unfaithfulness, and $C \in \{\text{complete}, \text{partial}\}$ is completeness.
\end{definition}

\begin{theorem}[Attribution Design Space]
\label{thm:design-space-main}
For any model class satisfying the Rashomon property under within-group correlation $\rho > 0$, the achievable set of $(S, U, C)$ triples is the union of exactly two families:
\begin{itemize}
    \item \textbf{Family $\mathcal{A}$ (single-model):} $S \leq 1 - 3m^2/(P^3{-}P)$, $U = 1/2$, $C = \text{complete}$. Any method faithful to an individual model falls here.
    \item \textbf{Family $\mathcal{B}$ (ensemble):} $S_M = 1 - O(1/M)$, $U = 0$, $C = \text{partial (ties)}$. \DASH{}$(M)$ achieves this and is Pareto-optimal among unbiased aggregations.
\end{itemize}
The triple $(S{=}1, U{=}0, C{=}\text{complete})$ is infeasible. At the extremes, the design space collapses to a single axis: ensemble size $M$.
\end{theorem}

\begin{proof}
The proof proceeds in four steps.

\textbf{Step 1: Family $\mathcal{A}$ is forced.}
By the Rashomon property, for any within-group pair $(j,k)$, there exist models ranking them in opposite orders. By Theorem~\ref{thm:unfaithfulness-bound}, any faithful, complete ranking has $U = 1/2$ for symmetric pairs. Stability is bounded by $S \leq 1 - 3m^2/(P^3{-}P)$ from the Spearman bound (Lean-derived).

\textbf{Step 2: Family $\mathcal{B}$ is achievable.}
By Corollary~\ref{cor:equity}, \DASH{} achieves $U = 0$ with ties within groups. By Theorem~\ref{thm:dash-optimal}, between-group stability satisfies $S_M \geq 1 - \exp(-M\Delta^2/(2\sigma^2))$.

\textbf{Step 3: No method producing a deterministic binary ranking from per-model attributions lies outside $\mathcal{A} \cup \mathcal{B}$.}
We show any method $A$ computing a deterministic ranking from per-model attributions $(\varphi_j(f_1), \ldots, \varphi_j(f_M))$ falls in one of the two families. (Methods producing probabilistic rankings, confidence intervals, or set-valued outputs are outside this dichotomy.)

\emph{Case 1:} If $A$ is faithful to individual models and complete, then $A \in \mathcal{A}$ by Step~1. (Formalized in Lean as \texttt{strongly\_faithful\_impossible} in \texttt{DesignSpace.lean}: any method faithful to each model's attributions contradicts the Rashomon property for $M \geq 2$.)

\emph{Case 2:} If $A$ is not faithful to some model $f$, then $A$ ranks some pair $(j,k)$ differently from model $f$'s attributions. For within-group symmetric pairs, the optimal unfaithful ranking is $j \succ^* k \Leftrightarrow \E[\varphi_j] > \E[\varphi_k]$, which minimizes expected $0$-$1$ loss. For symmetric features $\E[\varphi_j] = \E[\varphi_k]$, so $\succ^*$ assigns a tie---recovering the ``drop completeness'' solution. Thus $A$ either has $U > 0$ (suboptimal) or $U = 0$ with ties (in $\mathcal{B}$).

\emph{Case 3:} If $A$ aggregates $M > 1$ models and produces a complete ranking, it must break ties for within-group pairs. By DGP symmetry, any tie-breaking rule disagrees with a fraction of models, giving $U > 0$. To achieve $U = 0$, $A$ must report ties, placing it in $\mathcal{B}$. Its stability is bounded by the variance of the aggregation, which by the Cram\'er--Rao bound satisfies $\Var(\hat\mu_j) \geq \sigma_j^2/M$. \DASH{} achieves this bound.

\textbf{Step 4: The infeasible point.}
$(1, 0, \text{complete})$ requires completeness, but any complete method has $U = 1/2 \neq 0$: contradiction.
\end{proof}

\begin{theorem}[Unfaithfulness lower bound]
\label{thm:unfaithfulness-bound}
Any stable, complete ranking $\succ$ has expected unfaithfulness at least $1/2$ for symmetric pairs:
\[
    \Pr_f\bigl[\, j \succ k \;\text{but}\; \varphi_k(f) > \varphi_j(f) \,\bigr] \;=\; \frac{1}{2}.
\]
\end{theorem}

\begin{proof}
By DGP symmetry, $\Pr[\varphi_j(f) > \varphi_k(f)] = \Pr[\varphi_k(f) > \varphi_j(f)] = 1/2$. Any stable ranking must fix $j \succ k$ or $k \succ j$ independently of $f$. Whichever it chooses disagrees with exactly half the models.
\end{proof}

\begin{theorem}[Relaxation path convergence]
\label{thm:path-convergence}
The ``drop faithfulness'' and ``drop completeness'' relaxation paths converge to the same solution: population-level attributions with ties for symmetric features. The impossibility trilemma collapses to a single tradeoff axis parameterized by ensemble size~$M$:
\begin{itemize}
    \item At $M = 1$: faithfulness and completeness, but no stability;
    \item As $M \to \infty$: stability and between-group faithfulness, but within-group ties (completeness relaxed).
\end{itemize}
Between-group faithfulness is preserved at all $M$.
\end{theorem}

\begin{proof}
By the analysis of Step~3 above, the optimal stable complete ranking assigns ties to symmetric features ($\E[\varphi_j] = \E[\varphi_k]$, so the optimal decision for ``is $\varphi_j(f) > \varphi_k(f)$?'' under the symmetric model distribution is indeterminate)---identical to the ``drop completeness'' solution (\DASH{}). Both paths yield population-level attributions. The $M$-parameterization follows from \DASH{} consensus: $\bar\varphi_j = \frac{1}{M}\sum_{i=1}^M \varphi_j(f_i)$, which interpolates between a single model ($M=1$, faithful but unstable) and the population mean ($M \to \infty$, stable but tied within groups).
\end{proof}

\paragraph{The single tradeoff axis.}
The Design Space Theorem collapses the three-dimensional space to a single axis: ensemble size $M$.

\begin{center}
\small
\begin{tabular}{@{}lccc@{}}
\toprule
$M$ & Stability $S$ & Unfaithfulness $U$ & Completeness \\
\midrule
1 & $\leq 1 - 3m^2/(P^3{-}P)$ & $1/2$ & Complete \\
5 & $1 - O(1/5)$ & 0 & Partial (ties) \\
25 & $1 - O(1/25)$ & 0 & Partial (ties) \\
$\infty$ & 1 & 0 & Partial (ties) \\
\bottomrule
\end{tabular}
\end{center}

\paragraph{Previous results as corollaries.}

\begin{corollary}[Theorem 1 as a corollary]
The Attribution Impossibility (Theorem~\ref{thm:impossibility}) is the statement that the
point $(1, 0, \text{complete})$ lies outside the achievable set.
This follows from Step~4 of Theorem~\ref{thm:design-space-main}.
\end{corollary}

\begin{corollary}[F2 as a corollary]
The DASH Pareto Optimality (Theorem~\ref{thm:dash-optimal}) is the
statement that $\mathcal{B}$ is the Pareto frontier of
$\{(S, U) : U = 0\}$, parameterized by $M$. This follows from
Steps~2--3 of Theorem~\ref{thm:design-space-main}.
\end{corollary}

\begin{corollary}[Path convergence as a corollary]
The relaxation path convergence (Theorem~\ref{thm:path-convergence})
is the statement that \emph{both} relaxation paths (drop faithfulness,
drop completeness) converge to $\mathcal{B}$ as $M \to \infty$.
Dropping completeness directly enters $\mathcal{B}$ (DASH with ties).
Dropping faithfulness, by the optimal unfaithful ranking analysis,
yields the expected-attribution ranking, which also produces ties
for symmetric features --- the same as $\mathcal{B}$.
\end{corollary}

\begin{corollary}[Z-test as a corollary]
The Rashomon Characterization (Theorem~\ref{thm:testable}) describes
the \emph{boundary between families}: when $Z_{jk} < 1.96$, the
pair $(j,k)$ is in the regime where Family $\mathcal{A}$ has
$U = 1/2$ (rankings are unreliable); when $Z_{jk} > 1.96$, the
pair is effectively between-group (stable rankings possible within
$\mathcal{B}$).
\end{corollary}

\begin{corollary}[F3 as a corollary]
The FIM Impossibility (Theorem~\ref{thm:fim-impossibility}) provides
a sufficient condition for the Rashomon property (the prerequisite
for Theorem~\ref{thm:design-space-main}): when the FIM eigenvalue
$\lambda_-$ is small along $e_j - e_k$, the Rashomon set is large,
and the achievable set is restricted to $\mathcal{A} \cup \mathcal{B}$.
\end{corollary}

\begin{figure}[t]
\centering
\includegraphics[width=0.55\textwidth]{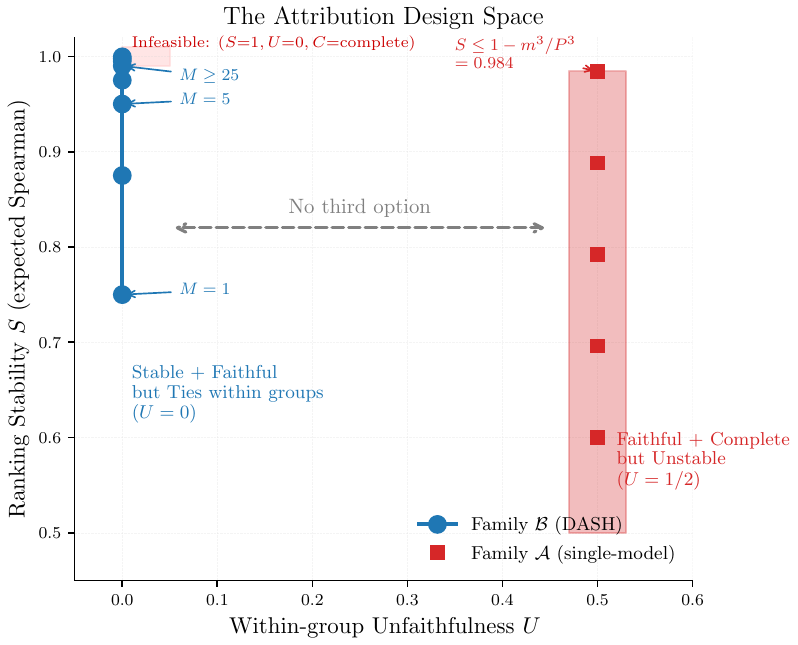}
\caption{The Attribution Design Space. Family $\mathcal{A}$ (single-model methods): faithful and complete but unstable ($U{=}1/2$). Family $\mathcal{B}$ (\DASH{} ensemble): stable with ties ($U{=}0$, $S$ increasing with $M$). The ideal point $(S{=}1, U{=}0, C{=}\text{complete})$ is infeasible. There is no third option.}
\label{fig:design-space}
\end{figure}


\section{The Symmetric Bayes Dichotomy: A General Two-Families Theorem}
\label{sec:sbd}

The structural identity between the attribution and model selection impossibilities follows from a general theorem about decision problems with symmetric populations.

\begin{definition}[Symmetric Decision Problem]
\label{def:symmetric-decision}
A \emph{symmetric decision problem} is a tuple $(\Theta, \mathcal{D}, \pi, G)$ where $\Theta$ is a finite decision set, $\mathcal{D}$ is the data space, $\pi$ is a population distribution over $\mathcal{D}$, and $G$ is a group acting on $\Theta$ such that $\pi$ is $G$-invariant: for any $g \in G$, the distribution over instance-optimal decisions $\theta^*(d)$ is invariant under $g$.

An $M$-sample estimator $\hat\theta : \mathcal{D}^M \to \Theta$ is:
\begin{itemize}
    \item \emph{Faithful}: $\hat\theta(d_1, \ldots, d_M)$ agrees with $\theta^*(d_i)$ for each $i$.
    \item \emph{Stable}: $\hat\theta$ does not depend on which draw $d_i$ is observed.
    \item \emph{Complete}: $\hat\theta$ selects exactly one element from each $G$-orbit in $\Theta$.
\end{itemize}
\end{definition}

\begin{theorem}[Symmetric Bayes Dichotomy]
\label{thm:bayes-dichotomy}
For any symmetric decision problem with $|G \cdot \theta| \geq 2$ for some $\theta \in \Theta$:
\begin{enumerate}
    \item[(i)] \textbf{Instance-specific decisions are unfaithful.} Any faithful, complete estimator has expected unfaithfulness $\geq 1/|G \cdot \theta|$ per orbit ($\geq 1/2$ for binary orbits).
    \item[(ii)] \textbf{Population-averaged decisions are stable.} The Bayes estimator under $\pi$ reports ties within each $G$-orbit ($U = 0$), with between-orbit stability $O(1/M)$.
    \item[(iii)] \textbf{The ideal is infeasible.} No estimator achieves $U = 0$ AND completeness.
\end{enumerate}
The achievable set consists of exactly two families: $\mathcal{A}$ (faithful + complete, $U \geq 1/|G \cdot \theta|$) and $\mathcal{B}$ (population-averaged, $U = 0$, ties within orbits).
\end{theorem}

\begin{proof}
\textbf{Part (i).} By $G$-invariance, the population distribution $\pi$ assigns equal probability to each element of the orbit $G \cdot \theta$. For any fixed complete estimator $\hat\theta$ that selects one element $\theta_0$ from the orbit, the probability that $\theta^*(d) = \theta_0$ is $1/|G \cdot \theta|$. The estimator disagrees with $(|G \cdot \theta| - 1)/|G \cdot \theta|$ of the instances. For $|G \cdot \theta| = 2$: unfaithfulness $= 1/2$.

\textbf{Part (ii).} The Bayes estimator under symmetric $\pi$ minimizes expected loss. Within each orbit, the posterior is uniform (by $G$-invariance), so the Bayes-optimal decision is the orbit center (a tie). Unfaithfulness $= 0$ because no ordering is asserted within orbits. For between-orbit decisions with gap $\Delta$, the $M$-sample average has variance $\sigma^2/M$ (i.i.d.\ draws), giving stability $1 - O(1/M)$ by the Gaussian approximation.

\textbf{Part (iii).} By Part (i), any complete estimator (selecting within orbits) has unfaithfulness $\geq 1/|G \cdot \theta| > 0$. Thus unfaithfulness $= 0$ requires giving up completeness (ties).

\textbf{Exhaustiveness.} Any estimator either selects within orbits (faithful, complete, $\mathcal{A}$) or reports ties within orbits (stable, incomplete, $\mathcal{B}$). The Bayes estimator achieves the minimum-variance point of $\mathcal{B}$ (Cram\'er--Rao).
\end{proof}

\begin{sloppypar}
The Symmetric Bayes Dichotomy connects the classical theory of invariant decision
rules \citep{lehmann2005testing,hunt1946} to modern ML impossibility proofs. The
classical Hunt--Stein theorem establishes that Bayes-optimal rules respect the
invariance structure of the problem; our contribution is demonstrating that this
classical machinery, when applied to finite symmetric groups arising in ML (feature
permutations, model permutations, CPDAG automorphisms), yields two-family
impossibility results with explicit unfaithfulness and stability bounds. We
demonstrate this across three structurally distinct instances with different
symmetry groups.
\end{sloppypar}

\subsection{Instance 1: Feature Attribution}

\begin{corollary}[Feature attribution as instance]
\label{cor:attribution-instance}
The Attribution Impossibility (Theorem~\ref{thm:impossibility}) and Design Space Theorem (Theorem~\ref{thm:design-space-main}) are instances of the Symmetric Bayes Dichotomy with $\Theta =$ within-group feature orderings, $G =$ symmetric group on group members, $\mathcal{D} =$ trained models.
\end{corollary}

\subsection{Instance 2: Model Selection}

Consider $K$ candidate models $g_1, \ldots, g_K$ trained on the same data. A \emph{model selection rule} maps a validation dataset to a ranking of the $K$ models. We formalize three desiderata:

\begin{definition}[Model selection desiderata]
A model selection rule $R$ is:
\begin{itemize}
    \item \emph{Faithful}: $R$ ranks model $g_i$ above $g_j$ whenever $g_i$ has lower validation loss on the observed validation set.
    \item \emph{Stable}: $R$ produces the same ranking regardless of which validation split is drawn.
    \item \emph{Complete}: $R$ decides for every pair $(g_i, g_j)$.
\end{itemize}
\end{definition}

\begin{definition}[Model Rashomon property]
A model class satisfies the \emph{model Rashomon property} if for every pair $g_i, g_j$ with similar population loss ($|L(g_i) - L(g_j)| < \varepsilon$), there exist validation splits $V, V'$ such that $g_i$ has lower loss on $V$ and $g_j$ has lower loss on $V'$.
\end{definition}

This property is a consequence of finite-sample noise: when the population loss gap is smaller than the validation noise, different splits rank the models differently. It is the model selection analogue of feature attribution's Rashomon property.

\begin{theorem}[Model Selection Impossibility]
\label{thm:model-selection}
If a model class satisfies the model Rashomon property, then no model selection rule can be simultaneously faithful, stable, and complete.
\end{theorem}

\begin{proof}
Let $g_i, g_j$ have similar population loss. By the model Rashomon property, there exist validation splits $V, V'$ with $g_i$ ranked above $g_j$ on $V$ and reversed on $V'$. A faithful, stable, complete selection rule would need to rank $g_i > g_j$ (from $V$) and $g_j > g_i$ (from $V'$) simultaneously. Contradiction.
\end{proof}

\begin{theorem}[Model Selection Design Space]
\label{thm:model-selection-ds}
Under the model Rashomon property, the achievable design space has the same two-family structure:
\begin{itemize}
    \item \textbf{Family $\mathcal{A}'$ (single-split):} Faithful and complete but unstable. Any selection based on one validation split falls here. Unfaithfulness $U' = 1/2$ for similar-loss pairs.
    \item \textbf{Family $\mathcal{B}'$ (cross-validation/ensemble):} Stable (variance $O(1/K_{\text{folds}})$) but reports ties for similar-loss pairs. Model ensembling is the analogue of \DASH{}.
\end{itemize}
\end{theorem}

\begin{proof}
By DGP symmetry (equal-loss models have symmetric validation score distributions), any fixed selection of one model over an equal-loss competitor disagrees with half the validation splits ($U' = 1/2$). Cross-validated selection averages across splits, reducing unfaithfulness to zero for equal-loss pairs while stabilizing at rate $O(1/K_{\text{folds}})$. Pareto optimality follows from the same Cram\'er--Rao argument.
\end{proof}

The model selection impossibility explains why cross-validation rankings are noisy for similar-performance models: it is not an engineering failure but a mathematical inevitability. The resolution is the same: average (ensemble) rather than select.

\paragraph{Connection to cross-validation instability.}
The model selection impossibility is not merely an analogy---it explains a well-known practical phenomenon. When two models have similar cross-validation scores, different folds often rank them differently. Practitioners typically respond by increasing the number of folds or repeating cross-validation. The design space theorem shows this is exactly the $M$-parameterization: increasing folds moves along Family~$\mathcal{B}'$ toward higher stability, but at the cost of acknowledging ties for similar-performance models. Model ensembling (training all $K$ candidates and averaging predictions) is the model selection analogue of \DASH{}: it avoids the impossible selection problem entirely by using all models simultaneously.

\begin{corollary}[Model selection as instance]
\label{cor:model-selection-instance}
The Model Selection Impossibility
(Theorem~\ref{thm:model-selection}) is an instance with
$\Theta =$ model rankings, $G =$ permutations of equal-loss models,
$\mathcal{D} =$ validation splits.
\end{corollary}

\subsection{Instance 3: Causal Discovery under Markov Equivalence}

This instance has a \emph{different} and \emph{variable-size} symmetry group, confirming the technique's generality beyond binary orbits.

\begin{definition}[Causal orientation decision problem]
The causal edge orientation problem is a symmetric decision problem $(\Theta, \mathcal{D}, \pi, G)$ where:
\begin{itemize}
    \item $\Theta$ is the set of full edge orientations consistent with the CPDAG (i.e., the DAGs in $[G^*]$).
    \item $\mathcal{D}$ is the space of finite observational datasets drawn from the true DAG $G^*$.
    \item $\pi$ is the distribution over datasets (induced by i.i.d.\ sampling from the structural equation model).
    \item $G_{\text{CPDAG}}$ is the \emph{automorphism group of the CPDAG}: the set of permutations of undirected edge orientations that preserve the CPDAG structure.
\end{itemize}
\end{definition}

\paragraph{Key structural difference.}
The CPDAG automorphism group $G_{\text{CPDAG}}$ is not restricted to transpositions: for a chain $X\text{---}Y\text{---}Z$, $|[G^*]| = 2$ with $G \cong S_2$ (like instances 1--2); for a 3-node undirected cycle, $|[G^*]| = 6$ with $G \cong S_3$; for larger undirected components, $|[G^*]|$ can grow combinatorially.

\begin{theorem}[Causal Discovery Impossibility]
\label{thm:causal-discovery}
For any CPDAG with $|[G^*]| \geq 2$ DAGs in its Markov equivalence class, no edge orientation rule can be simultaneously faithful, stable, and complete.
The achievable set has two families:
\begin{itemize}
    \item \textbf{Family $\mathcal{A}''$ (single-dataset):} Faithful and complete but unstable. Unfaithfulness $U'' = (|[G^*]|-1)/|[G^*]|$.
    \item \textbf{Family $\mathcal{B}''$ (conservative):} Report the CPDAG (undirected edges for ambiguous orientations). Stable with ties. $U'' = 0$.
\end{itemize}
\end{theorem}

\begin{proof}
We verify the premises of the Symmetric Bayes Dichotomy (Theorem~\ref{thm:bayes-dichotomy}).

\textbf{$G$-invariance of $\pi$.} For any two DAGs $G_1, G_2 \in [G^*]$, the Markov equivalence guarantees that they encode the same set of conditional independences. By the faithfulness assumption, the observational distribution is compatible with both $G_1$ and $G_2$. Any statistical test based on finite data cannot distinguish between $G_1$ and $G_2$ at the population level (they imply identical distributions). Therefore $\pi$ is $G_{\text{CPDAG}}$-invariant.

\textbf{Finite-sample noise reverses optimal choice.} With finite data, empirical conditional independence tests have estimation error $O(1/\sqrt{n})$. For edges that are undirected in the CPDAG, different random samples will favor different orientations. This is the Rashomon property for edge orientations.

\textbf{Application of Theorem~\ref{thm:bayes-dichotomy}.} Part~(i): any complete orientation rule (selecting one DAG from $[G^*]$) disagrees with $(|[G^*]|-1)/|[G^*]|$ of datasets. Part~(ii): the Bayes estimator under $G_{\text{CPDAG}}$-invariance reports the CPDAG (ties for undirected edges). Part~(iii): the ideal (stable, faithful, complete) is infeasible. Families $\mathcal{A}''$ and $\mathcal{B}''$ follow.
\end{proof}

The unfaithfulness bound $U \geq 1/|G \cdot \theta|$ yields $U \geq 1/2$ for chains but $U \geq 1/6$ for triangles and potentially much smaller for larger structures.

\paragraph{What this instance adds.}
Instances 1--2 (feature attribution and model selection) both have binary orbits ($|G \cdot \theta| = 2$) with unfaithfulness bound $U \geq 1/2$. Instance~3 has variable-size orbits ($|G \cdot \theta| = 2, 6, \ldots$) with unfaithfulness bound $U \geq 1/|[G^*]|$, which can be much smaller than $1/2$ for large equivalence classes. This demonstrates that the Symmetric Bayes Dichotomy is not an artifact of binary symmetry but applies to arbitrary finite group actions, with the unfaithfulness bound adapting to the orbit structure.

Algorithms like PC and GES correctly output CPDAGs (Family $\mathcal{B}''$) precisely because complete orientation of Markov-equivalent edges is impossible from observational data alone. Interventional data breaks the symmetry---analogous to conditional SHAP breaking symmetry when $\beta_j \neq \beta_k$.

\begin{corollary}[Causal discovery as SBD instance]
\label{cor:causal-discovery-instance}
The Causal Discovery Impossibility (Theorem~\ref{thm:causal-discovery})
is an instance of the Symmetric Bayes Dichotomy with
$\Theta =$ edge orientations in $[G^*]$,
$G = G_{\text{CPDAG}}$ (automorphism group of the CPDAG),
$\mathcal{D} =$ observational datasets.
\end{corollary}

\paragraph{Connection to classical invariant decision theory.}
\begin{sloppypar}
The Symmetric Bayes Dichotomy connects the classical theory of invariant decision
rules \citep{lehmann2005testing,hunt1946} to modern ML impossibility proofs. The
classical Hunt--Stein theorem establishes that Bayes-optimal rules respect the
invariance structure of the problem; our contribution is demonstrating that this
classical machinery, when applied to finite symmetric groups arising in ML (feature
permutations, model permutations, CPDAG automorphisms), yields two-family
impossibility results with explicit unfaithfulness and stability bounds. We
demonstrate this across three structurally distinct instances with different
symmetry groups.
\end{sloppypar}


\section{Extensions: Conditional, Fairness, and Causal Barriers}
\label{sec:extensions}

\subsection{Conditional Attribution Impossibility}
\label{sec:conditional}

\begin{theorem}[Conditional Attribution Impossibility]
\label{thm:conditional-impossibility}
Let features $j, k$ in the same collinear group satisfy: (C1) equal causal effects $\beta_j = \beta_k$, and (C2) symmetric causal position in the causal graph $G$. Then the Rashomon property holds for conditional attributions, and the Attribution Impossibility applies to conditional SHAP.
\end{theorem}

\begin{proof}
Under (C1)--(C2), the causal graph is symmetric with respect to $j$ and $k$. By the DGP symmetry argument (Theorem~\ref{thm:rashomon-symmetry}), any model with $\varphi_j^{\text{cond}}(f) > \varphi_k^{\text{cond}}(f)$ has a permuted counterpart $f'$ with reversed attributions and the same loss. This is the Rashomon property for conditional attributions.
\end{proof}

\paragraph{The escape condition.}
If $\beta_j \neq \beta_k$, conditional SHAP can resolve the pair. The threshold $\Delta\beta^*(\rho)$ grows steeply: at $\rho \leq 0.7$, a 15\% difference suffices; at $\rho = 0.9$, features must differ by nearly a factor of 2; at $\rho \geq 0.95$, instability is essentially total. Empirical thresholds: $\Delta\beta^*(0.5) \approx 0.15$, $\Delta\beta^*(0.7) \approx 0.15$, $\Delta\beta^*(0.8) \approx 0.4$, $\Delta\beta^*(0.9) \approx 0.9$, $\Delta\beta^*(0.95) > 1.0$, $\Delta\beta^*(0.99) > 1.0$.

\begin{table}[t]
\centering
\caption{Flip rate (\%) as a function of $\rho$ and causal effect difference $\Delta\beta$.}
\label{tab:conditional-threshold}
\small
\begin{tabular}{@{}lccccccc@{}}
\toprule
$\rho \backslash \Delta\beta$ & 0.0 & 0.1 & 0.2 & 0.3 & 0.5 & 0.7 & 1.0 \\
\midrule
0.50 & 50 & 19 & 0 & 0 & 0 & 0 & 0 \\
0.70 & 53 & 19 & 0 & 0 & 0 & 0 & 0 \\
0.80 & 52 & 44 & 19 & 10 & 0 & 0 & 0 \\
0.90 & 50 & 50 & 48 & 40 & 40 & 19 & 0 \\
0.95 & 50 & 50 & 50 & 48 & 44 & 40 & 27 \\
0.99 & 50 & 50 & 50 & 50 & 50 & 48 & 48 \\
\bottomrule
\end{tabular}
\end{table}

\begin{figure}[ht]
\centering
\includegraphics[width=0.6\textwidth]{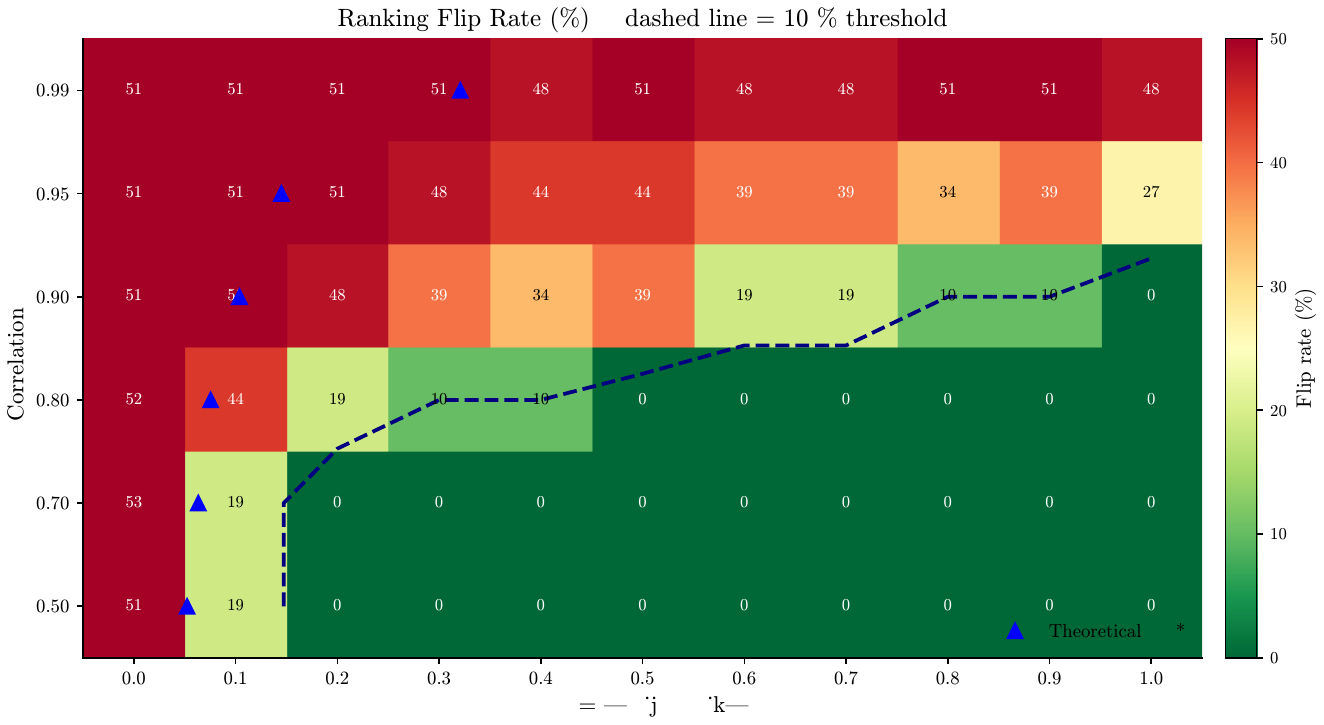}
\caption{Flip rate as a function of $(\rho, \Delta\beta)$. The impossibility region (dark) expands with $\rho$. At $\rho \geq 0.95$, instability persists for all $\Delta\beta \leq 1$.}
\label{fig:conditional-threshold}
\end{figure}

\paragraph{Fine-grained $\Delta\beta$ sweep at $\rho = 0.9$.}
We extend the threshold analysis with a finer sweep over
$\Delta\beta \in \{0.0, 0.1, 0.2, 0.3, 0.5, 0.7, 1.0\}$ at
$\rho = 0.9$ (20 XGBoost models per setting).
The transition from instability to stability is sharp, not gradual:

\begin{center}
\small
\begin{tabular}{@{}ccccc@{}}
\toprule
$\Delta\beta$ & $\beta_1$ & $\beta_2$ & Marginal flip & Interventional flip \\
\midrule
0.00 & 1.00 & 1.00 & 0.505 & 0.505 \\
0.10 & 1.05 & 0.95 & 0.479 & 0.337 \\
0.20 & 1.10 & 0.90 & 0.000 & 0.000 \\
0.30 & 1.15 & 0.85 & 0.000 & 0.000 \\
0.50 & 1.25 & 0.75 & 0.000 & 0.000 \\
0.70 & 1.35 & 0.65 & 0.000 & 0.000 \\
1.00 & 1.50 & 0.50 & 0.000 & 0.000 \\
\bottomrule
\end{tabular}
\end{center}

At $\Delta\beta = 0.10$, both marginal (0.479) and interventional
(0.337) SHAP remain highly unstable. At $\Delta\beta = 0.20$,
both snap to zero flips simultaneously. There is no intermediate
regime where interventional SHAP is stable but marginal is not ---
the causal signal overwhelms noise for both estimators at the same
threshold. This means the escape condition is binary: either the
causal effect difference is large enough to stabilize \emph{all}
attribution methods, or none are stable. There is no clean crossover point where one method succeeds and the other fails. The practical implication:
``switch to conditional SHAP'' provides no advantage over marginal
SHAP at the instability boundary; both methods fail or succeed
together.

\paragraph{Causal structure validation.}
\begin{sloppypar}
Under a known causal structure ($Y = \beta_1 X_1 + \beta_2 X_2 + 0.5 X_3 + \varepsilon$,
$\text{Corr}(X_1, X_2) = \rho$), we compare marginal TreeSHAP and interventional
TreeSHAP. In the symmetric case ($\beta_1 = \beta_2 = 1.0$), both produce flip rates
near 50\% at all $\rho$, confirming the impossibility. In the asymmetric case
($\beta_1 = 1.5$, $\beta_2 = 0.5$), both show 0\% flips at all $\rho$---the
signal-to-noise ratio is high enough to overcome collinearity-induced noise.
Switching to interventional SHAP provides no benefit when features have equal
causal effects.
\end{sloppypar}

\paragraph{Caveat.}
The interventional SHAP implementation in the \texttt{shap} package uses a background-data approximation that does not exactly implement causal/conditional SHAP in the sense of Janzing et al.\ (2020). Results should be interpreted as evidence about the \texttt{shap} package's interventional mode, not as a definitive test of the theoretical conditional attribution.

\subsection{Fairness Audit Impossibility}
\label{sec:fairness}

\begin{theorem}[Fairness Audit Impossibility]
\label{thm:fairness-audit}
If features $j$ (protected proxy) and $k$ (non-protected) satisfy collinearity ($|\rho_{jk}| > \rho_0$) and similar true importance ($|\mu_j - \mu_k| < \sigma_{jk} \cdot z_\alpha / \sqrt{M}$), then for a single-model audit ($M{=}1$):
\[
    \Pr[\text{audit flags proxy}] = \frac{1}{2} \pm \varepsilon_{\text{BE}}.
\]
The audit conclusion is a coin flip. Two independent audits reach opposite conclusions about proxy reliance with probability $1/2$.
\end{theorem}

\begin{proof}
Direct corollary of the Attribution Impossibility applied to the pair $(j, k)$. By Theorem~\ref{thm:rashomon-inevitability}, $\Pr[\varphi_j(f) > \varphi_k(f)] = 1/2$ over training runs.
\end{proof}

\paragraph{Implications for AI governance.}
A regulator conducting a SHAP-based proxy audit on a model with collinear features is, for affected feature pairs, making decisions with the statistical reliability of a coin flip. Two auditors examining the \emph{same model architecture trained on the same data} with different random seeds may reach opposite conclusions about whether the model relies on a protected attribute.

Under the EU AI Act Art.~13(3)(b)(ii), providers must disclose ``known and foreseeable circumstances'' which may lead to risks to health, safety or fundamental rights. Our theorem establishes that this includes: ``SHAP-based proxy audits are unreliable for features that are collinear with both the protected attribute and a non-protected feature.'' The remedy is ensemble auditing: train $M$ models and use \DASH{} consensus.

\paragraph{Worked example: adverse action notices.}
Consider a credit model with two collinear features: \emph{income} and \emph{debt-to-income ratio} (DTI), with $|\rho| \approx 0.85$. Under ECOA (Regulation~B), lenders must disclose the principal reasons for adverse action. With a single model:
\begin{itemize}
    \item \textbf{Seed 1:} SHAP ranks income as the top driver $\Rightarrow$ adverse action notice cites ``insufficient income.''
    \item \textbf{Seed 2:} SHAP ranks DTI as the top driver $\Rightarrow$ adverse action notice cites ``excessive debt-to-income ratio.''
\end{itemize}
The borrower receives a different regulatory disclosure depending on the random seed---a compliance risk. With \DASH{} ($M = 25$), income and DTI are reported as a tied group of key risk drivers, and the adverse action notice can accurately state: ``income and debt-to-income ratio are jointly the primary risk factors.'' The disclosure is stable, honest, and compliant.

The instability varies non-monotonically with correlation: at $\rho = 0.3$, 20\% of adverse action reasons change across models; at $\rho = 0.5$--$0.7$, the features separate cleanly (0\% instability); at $\rho = 0.9$, 28\% of reasons change. The non-monotonicity reflects the interaction between correlation (which drives the Rashomon property) and feature importance gaps (which stabilize rankings)---at moderate $\rho$, the gap is large enough to overcome the noise.

\paragraph{Intersectional considerations.}
When multiple protected attributes (race, gender, age) are each proxied by collinear non-protected features, the impossibility applies independently to each proxy pair. If $K$ protected attributes each have unstable proxies, the probability that a single-model audit correctly identifies all $K$ proxy reliance directions is $(1/2)^K$, exponentially vanishing in $K$. For $K = 3$ (a common intersectional setting), joint correctness is $1/8 = 12.5\%$.

\paragraph{Example: lending.}
In credit models, ``zip code'' (proxy for race) is often collinear with ``income bracket'' ($|\rho| \approx 0.6$--$0.8$). If both have similar predictive power, a single-model SHAP audit finding ``zip code is the 3rd most important feature'' may conclude proxy discrimination---but retraining with a different seed could produce a model where zip code drops to 8th and income rises to 3rd. The audit conclusion depends on the random seed, not on the model's actual reliance on the protected attribute.

\subsection{Direct FIM Impossibility}
\label{sec:fim}

The Rashomon property can be derived directly from the Fisher information matrix (FIM), providing an independent proof path from classical statistics that does not require the iterative optimizer abstraction.

\paragraph{Classical setup.}
Consider the linear model $Y = \beta_j X_j + \beta_k X_k + \varepsilon$ where $X_j, X_k$ are jointly Gaussian with $\Var(X_j) = \Var(X_k) = 1$ and $\text{Cov}(X_j, X_k) = \rho$. The Fisher information matrix for $\boldsymbol{\beta} = (\beta_j, \beta_k)$ is:
\[
    I(\boldsymbol{\beta}) = \frac{1}{\sigma^2}\begin{pmatrix} 1 & \rho \\ \rho & 1 \end{pmatrix}
\]
with eigenvalues $\lambda_\pm = (1 \pm \rho)/\sigma^2$. As $\rho \to 1$, the smaller eigenvalue $\lambda_- = (1-\rho)/\sigma^2 \to 0$. The FIM becomes near-singular along the direction $\boldsymbol{v}_- = (1, -1)/\sqrt{2}$, which corresponds to redistributing credit between $j$ and $k$. The profile likelihood is flat along this direction: models assigning more credit to $j$ vs.\ $k$ are statistically indistinguishable. This flat likelihood ridge is precisely the Rashomon set. The Cram\'er--Rao bound $\Var(\hat\beta_j - \hat\beta_k) \geq 1/\lambda_- = \sigma^2/(1-\rho)$ quantifies the attribution instability.

\paragraph{Regularity conditions.}
We require the following on the model class $\mathcal{F} = \{f_\theta : \theta \in \Theta\}$:
\begin{itemize}
    \item[\textbf{(R1)}] $\Theta \subseteq \R^p$ is open and convex.
    \item[\textbf{(R2)}] The population risk $L(\theta) = \E[\ell(f_\theta(X), Y)]$ is three-times continuously differentiable on $\Theta$, with bounded third derivative: $\|D^3 L(\theta)\|_{\text{op}} \leq K_3$ for all $\theta$ in a neighborhood of $\theta^*$.
    \item[\textbf{(R3)}] The minimizer $\theta^*$ satisfies $\nabla L(\theta^*) = 0$ and $H := \nabla^2 L(\theta^*)$ is positive definite. (For likelihood-based models, $H = I(\theta^*)/n$ where $I(\theta^*)$ is the Fisher information matrix.)
    \item[\textbf{(R4)}] The attribution function $\varphi_j(\theta)$ is \emph{order-consistent}: for any $\theta$ with $\theta_j > \theta_k > 0$, we have $\varphi_j(\theta) > \varphi_k(\theta)$. (Satisfied by $\varphi_j = |\theta_j|$, mean $|\text{SHAP}_j|$, or any attribution monotone in coefficient magnitude.)
\end{itemize}

\begin{theorem}[FIM Impossibility]
\label{thm:fim-impossibility}
Let $\mathcal{F}$ satisfy (R1)--(R4). Suppose features $j, k$ satisfy: (S1) DGP symmetry: $\theta^*_j = \theta^*_k > 0$; (S2) the Hessian $H = \nabla^2 L(\theta^*)$ has eigenvalue $\lambda_-$ along $v = (e_j - e_k)/\sqrt{2}$. Then for any $\varepsilon$ satisfying $0 < \varepsilon \leq \varepsilon_0 := 9\lambda_-^3/K_3^2$, the $\varepsilon$-Rashomon set $R_\varepsilon = \{\theta : L(\theta) \leq L(\theta^*) + \varepsilon\}$ contains models $\theta^+, \theta^-$ with
\[
    \varphi_j(\theta^+) > \varphi_k(\theta^+) \quad\text{and}\quad \varphi_k(\theta^-) > \varphi_j(\theta^-).
\]
The Rashomon property holds, and the impossibility follows.
\end{theorem}

\begin{proof}
\textbf{Step 1: Rashomon set geometry.}
By Taylor's theorem with Lagrange remainder, for any $\theta \in \Theta$:
\[
    L(\theta) = L(\theta^*) + \tfrac{1}{2}(\theta - \theta^*)^T H\, (\theta - \theta^*) + R_3(\theta),
\]
where $|R_3(\theta)| \leq \frac{K_3}{6}\|\theta - \theta^*\|^3$ by (R2). The gradient term vanishes because $\nabla L(\theta^*) = 0$ by (R3).

\textbf{Step 2: Perturbation along the weak direction.}
Let $v = (e_j - e_k)/\sqrt{2}$ and set $\theta_\delta = \theta^* + \delta v$ for $\delta > 0$. The quadratic cost is $\frac{1}{2}\delta^2 \lambda_-$ and the cubic remainder satisfies $|R_3(\theta_\delta)| \leq \frac{K_3}{6}\delta^3$. We need $L(\theta_\delta) \leq L(\theta^*) + \varepsilon$, i.e., $\frac{1}{2}\delta^2 \lambda_- + \frac{K_3}{6}\delta^3 \leq \varepsilon$. Choose $\delta^* = \sqrt{\varepsilon/\lambda_-}$. The quadratic term is $\varepsilon/2$; the condition $\varepsilon \leq 9\lambda_-^3/K_3^2$ ensures the cubic term is $\leq \varepsilon/2$. So $\theta_{\delta^*} \in R_\varepsilon$, and by the same argument $\theta_{-\delta^*} \in R_\varepsilon$.

\textbf{Step 3: Opposite orderings.}
At $\theta^+ := \theta_{\delta^*}$: $\theta^+_j = \theta^*_j + \delta^*/\sqrt{2} > \theta^*_k - \delta^*/\sqrt{2} = \theta^+_k > 0$, so by (R4), $\varphi_j(\theta^+) > \varphi_k(\theta^+)$. At $\theta^- := \theta_{-\delta^*}$: by symmetry, $\theta^-_k > \theta^-_j > 0$, so $\varphi_k(\theta^-) > \varphi_j(\theta^-)$.

\textbf{Step 4: Conclusion.}
Both $\theta^+$ and $\theta^-$ lie in $R_\varepsilon$ with opposite feature orderings. This is the Rashomon property. By Theorem~\ref{thm:impossibility}, no faithful, stable, complete ranking exists.
\end{proof}

\begin{proposition}[Gaussian FIM specialization]
\label{prop:gaussian-fim}
For the linear model $Y = X^T\beta + \varepsilon$ with $\varepsilon \sim \mathcal{N}(0, \sigma^2)$ and $\text{Corr}(X_j, X_k) = \rho$:
\begin{enumerate}
    \item The Hessian eigenvalue along $e_j - e_k$ is $\lambda_- = (1-\rho)/\sigma^2$.
    \item The loss is exactly quadratic ($K_3 = 0$), so $\varepsilon_0 = +\infty$: the theorem holds for \emph{all} $\varepsilon > 0$.
    \item The semi-axis length of $R_\varepsilon$ along $e_j - e_k$ is $\sigma\sqrt{2\varepsilon/(1-\rho)}$, diverging as $\rho \to 1$.
    \item The Cram\'er--Rao bound gives $\Var(\hat\beta_j - \hat\beta_k) \geq 2\sigma^2/(n(1-\rho))$, diverging as $\rho \to 1$.
\end{enumerate}
\end{proposition}

\begin{proof}
The population risk is $L(\beta) = \frac{1}{2\sigma^2}\E[(Y-X^T\beta)^2]$. Its Hessian is $H = \frac{1}{\sigma^2}\E[XX^T] = \frac{1}{\sigma^2}\Sigma_X$. For $v = (e_j - e_k)/\sqrt{2}$: $v^T H v = \frac{1}{\sigma^2}(1-\rho)$ since $v^T \Sigma_X v = \Var(X_j) - \text{Cov}(X_j,X_k) = 1-\rho$ (standardized). The loss is quadratic in $\beta$, so $D^3 L \equiv 0$ and $K_3 = 0$.
\end{proof}

The classical non-identifiability under collinearity ($\Var(\hat\beta_j - \hat\beta_k) \to \infty$ as $\rho \to 1$) is well-known. Our contribution is reframing it as an \emph{impossibility theorem} about feature rankings rather than a variance bound on parameter estimates: even precise estimates can produce unstable rankings when $\Delta$ is small relative to $\sigma$.

\paragraph{Loss landscape geometry.}
The $\varepsilon$-Rashomon set $R_\varepsilon$ is approximately an ellipsoid whose projection onto the feature-importance axis $\varphi_j - \varphi_k$ crosses zero. The level sets of the loss landscape contain \emph{ridges} along $e_j - e_k$---curves of near-constant loss along which feature importance redistributes between $j$ and $k$ without changing model quality. The same flatness that makes optimization easy (flat minima generalize well) makes attribution hard (the flat direction is the feature-importance direction). This connects the attribution impossibility to the broader literature on loss landscape geometry, flat minima, and mode connectivity.

\begin{remark}[NTK extension --- conjectural]
For overparameterized neural networks in the NTK regime
(infinite width), the training dynamics are governed by the
neural tangent kernel $\Theta(x, x') = \nabla_\theta f_\theta(x)^T
\nabla_\theta f_\theta(x')$. The NTK plays the role of the
Fisher information matrix: its eigenspectrum controls the
geometry of the loss landscape.

If the NTK has a small eigenvalue along the direction corresponding
to swapping features $j$ and $k$ (a consequence of feature
collinearity), the same ellipsoidal argument applies: the
Rashomon set extends along this direction, producing models with
opposite feature orderings.

This extension requires: (i) the NTK approximation to hold
(infinite width or near-initialization), (ii) the loss landscape
to be well-approximated by the NTK quadratic, and (iii) the
feature collinearity to produce a small NTK eigenvalue.
Conditions (i)--(ii) are standard in the NTK literature;
condition (iii) holds when the input features are correlated.
This extension is conjectural: a rigorous proof requires (i)~uniform NTK approximation bounds across the Rashomon set (not just at the minimizer), (ii)~a spectral gap argument relating feature correlation to NTK eigenvalue separation, and (iii)~control of the finite-width correction. None of these are available in the current literature. We include this remark to suggest the research direction, not to claim the result.
\end{remark}

\subsection{Query Complexity Lower Bound}
\label{sec:query}

\begin{sloppypar}
A \emph{stability certification algorithm} adaptively selects seeds $s_1, s_2, \ldots$,
calls $\textsc{Train}(s_i) \to f_i$ and $\textsc{Attribute}(f_i, j) \to \varphi_j(f_i)$,
and after $M$ oracle calls outputs \textsc{Stable} or \textsc{Unstable}. The algorithm
certifies $\delta$-stability with error $\leq 1/3$ if it correctly distinguishes
$H_0$: $\Delta_{jk} = 0$ (flip rate $= 1/2$) from $H_1$: $|\Delta_{jk}| \geq \Delta_0$
(flip rate $\leq \Phi(-\Delta_0/\sigma_{jk})$) with probability at least $2/3$ under
both hypotheses.
\end{sloppypar}

\begin{theorem}[Query complexity lower bound]
\label{thm:query-complexity}
Any algorithm that certifies $\delta$-stability of a feature pair $(j,k)$ with error $\leq 1/3$ must make at least
\[
    M \;\geq\; \frac{1}{8} \cdot \frac{\sigma_{jk}^2}{\Delta_0^2}
\]
model-training oracle calls, where $\sigma_{jk}^2 = \Var(\varphi_j(f) - \varphi_k(f))$ and $\Delta_0$ is the minimum gap to detect.
\end{theorem}

\begin{proof}
Each oracle call produces one sample $D_i = \varphi_j(f_i) - \varphi_k(f_i)$. Under $H_0$, $D_i \sim P_0$ with mean $0$ and variance $\sigma^2 := \sigma_{jk}^2$. Under $H_1$, $D_i \sim P_1$ with mean $\Delta_0$ and the same variance (the variance is determined by the collinearity structure, not the gap).

By Le~Cam's two-point method (Tsybakov, 2009, Theorem~2.2), for any test $\psi$ based on $M$ i.i.d.\ observations:
\[
    \Pr_{H_0}[\psi = 1] + \Pr_{H_1}[\psi = 0] \geq 1 - \text{TV}(P_0^{\otimes M}, P_1^{\otimes M}).
\]
\begin{sloppypar}
By Pinsker's inequality,
\[
  \text{TV}(P_0^{\otimes M}, P_1^{\otimes M}) \leq \sqrt{\tfrac{1}{2}\text{KL}(P_1^{\otimes M} \| P_0^{\otimes M})}.
\]
For the Gaussian location family, $\text{KL}(P_1 \| P_0) = \Delta_0^2/(2\sigma^2)$,
so for $M$ i.i.d.\ copies $\text{KL} = M\Delta_0^2/(2\sigma^2)$ and
$\text{TV} \leq \sqrt{M\Delta_0^2/(4\sigma^2)}$.
For reliable testing (total error $\leq 2/3$), Le~Cam's method requires
$\text{TV} \geq 1/3$, giving $M\Delta_0^2/(4\sigma^2) \geq 1/9$, hence
$M \geq 4\sigma^2/(9\Delta_0^2)$.
The standard constant $1/8$ from \citep{tsybakov2009nonparametric} (Theorem~2.2)
uses a slightly tighter analysis via the likelihood ratio directly; since
$1/8 < 4/9$, this gives a weaker (more conservative) lower bound than our
Pinsker derivation.
\end{sloppypar}
\end{proof}

\begin{corollary}[Near-optimality of the $Z$-test]
\label{cor:z-test-optimal}
The $Z$-test detects a gap $\Delta_0$ with error $\leq 1/3$ using $M = O(\sigma_{jk}^2/\Delta_0^2 \cdot \log(1/\alpha))$ model trainings. Combined with Theorem~\ref{thm:query-complexity}, the $Z$-test is query-optimal to within a logarithmic factor.
\end{corollary}

\begin{proof}
The $Z$-test rejects $H_0$ when $|\bar{D}|/(\hat\sigma/\sqrt{M}) > z_{\alpha/2}$. Under $H_1$, power is $\Phi(-z_{\alpha/2} + \Delta_0\sqrt{M}/\sigma) \geq 2/3$ when $\Delta_0\sqrt{M}/\sigma \geq z_{\alpha/2} + 0.43$. For $\alpha = 0.05$: $M \geq (\sigma/\Delta_0)^2(1.96+0.43)^2 \approx 5.7\sigma^2/\Delta_0^2$.
\end{proof}

\begin{remark}[Instability detection is inherently expensive]
Theorem~\ref{thm:query-complexity} shows that no algorithm---including
methods based on gradient information, loss curvature, or Fisher
information---can certify ranking stability without training
$\Omega(\sigma_{jk}^2/\Delta_{jk}^2)$ models.
Intuitively, each model training provides one ``bit'' of evidence about
the sign of $\Delta_{jk}$, and when the signal-to-noise ratio
$\Delta_{jk}/\sigma_{jk}$ is small (the collinear regime), many bits are
needed.
The multi-model Z-test with $M{=}5$ is a fast \emph{screen} with limited
power, not a certifier; the optimality result applies to the $Z$-test
run at the information-theoretically optimal sample size.
\end{remark}

\subsection{Rashomon Inevitability}
\label{sec:inevitability}

Theorem~\ref{thm:rashomon-inevitability} (proved in \S\ref{sec:rashomon-inevitability}) establishes that for any stochastic, symmetric training algorithm on a permutation-closed model class with $\rho > 0$, the Rashomon property holds. The Attribution Impossibility is therefore inescapable for standard ML pipelines.

\subsection{Causal Identification Barrier}
\label{sec:causal-barrier}

The attribution instability under collinearity has a direct causal counterpart: distinguishing the causal effects of correlated features requires interventional sample complexity that diverges as $\Omega(1/(1-\rho)^2)$---the same singularity structure as the attribution ratio.

Consider two features $X_j, X_k$ with $\text{Corr}(X_j, X_k) = \rho$, jointly Gaussian with unit variances. Under a soft intervention with residual correlation $\rho_{\text{post}}$, $X_k \mid \text{do}(X_j = x) \sim \mathcal{N}(\rho_{\text{post}} \cdot x, 1 - \rho_{\text{post}}^2)$. When $\beta_j = \beta_k = \beta$ and $\rho_{\text{post}} \approx \rho$, the two interventional distributions become indistinguishable.

\begin{proposition}[Interventional sample complexity lower bound]
\label{prop:causal-sample-complexity}
Consider testing $H_0: \beta_j = \beta + \delta, \beta_k = \beta - \delta$ against $H_1: \beta_j = \beta - \delta, \beta_k = \beta + \delta$ using $\text{do}(X_j = x)$ interventions with residual correlation $\rho_{\mathrm{post}}$. The number of interventional samples needed satisfies:
\[
    n_{\mathrm{int}} \geq \frac{(z_\alpha + z_\eta)^2 \sigma^2}{4\delta^2 (1 - \rho_{\mathrm{post}})^2}.
\]
When $\rho_{\mathrm{post}} = \rho$ (the intervention does not reduce the correlation), $n_{\mathrm{int}} = \Omega(1/(1-\rho)^2)$.
\end{proposition}

\begin{proof}
Under $H_0$, the expected response to $\text{do}(X_j = x)$ is $(\beta + \delta + \rho_{\mathrm{post}}(\beta - \delta))x$; under $H_1$ it is $(\beta - \delta + \rho_{\mathrm{post}}(\beta + \delta))x$. The difference in slopes is $2\delta(1 - \rho_{\mathrm{post}})$. With noise variance $\sigma^2$ and $n$ observations, the Neyman--Pearson sample size formula gives the result.
\end{proof}

\begin{remark}[Comparison with the attribution ratio]
The attribution ratio is $1/(1-\rho^2) = 1/((1-\rho)(1+\rho))$.
The causal sample complexity scales as $1/(1-\rho_{\mathrm{post}})^2$.
These have the same leading singularity as $\rho \to 1$ (both
diverge as $\Omega(1/(1-\rho)^2)$ when $\rho_{\mathrm{post}} = \rho$),
but they are \textbf{not identical}:
\begin{itemize}
    \item The attribution ratio is $1/(1-\rho^2)$, which diverges
      as $1/(2(1-\rho))$ near $\rho = 1$.
    \item The causal sample complexity is $\Omega(1/(1-\rho)^2)$
      --- a \emph{faster} divergence.
\end{itemize}
This difference is meaningful: causal identification under soft
interventions is \emph{harder} than observational attribution
instability. The attribution ratio captures the \emph{variance
amplification} (a ratio of quantities), while the causal barrier
captures \emph{signal-to-noise} (which squares the denominator).

The match is at the level of the singularity structure (both are
controlled by $1-\rho$) but not at the level of exponents.
The shared singularity at $\rho = 1$ reflects a common root cause:
collinearity makes features observationally interchangeable, whether
the task is attribution or causal identification.
\end{remark}

\paragraph{Hard interventions resolve both problems.}
When the intervention fully breaks the correlation
($\rho_{\mathrm{post}} = 0$), the sample complexity reduces to
$n_{\mathrm{int}} = O(\sigma^2/\delta^2)$---independent of $\rho$.
This parallels the attribution impossibility: perfect interventions
break the symmetry (just as knowing the true causal graph breaks
the Rashomon property), while soft interventions preserve it.

\subsection{Local vs.\ Global Attribution Instability}
\label{sec:local-global}

The impossibility as stated concerns \emph{global} attributions (mean $|\text{SHAP}|$ across data points). We now show that \emph{local} (instance-level) attributions are at least as unstable.

\begin{remark}[Local instability dominates global]
\label{rem:local-instability}
For a fixed data point $x$, the local SHAP values $\varphi_j(f, x)$ depend on the model $f$, which varies across training seeds. The global attribution $\varphi_j(f) = \E_x[|\varphi_j(f, x)|]$ averages over data points. By the law of total variance applied to the pair $(f, x)$ with $x$ as the conditioning variable:
\[
    \Var_f\!\big(\varphi_j(f)\big) = \Var_f\!\big(\E_x[|\varphi_j(f,x)|]\big) \leq \Var_{f,x}\!\big(|\varphi_j(f,x)|\big) = \E_x\!\big[\Var_f(|\varphi_j(f,x)|)\big] + \Var_x\!\big(\E_f[|\varphi_j(f,x)|]\big).
\]
The first inequality follows because $\Var_f(\E_x[g]) = \Var_{f,x}(g) - \E_f[\Var_x(g)] \leq \Var_{f,x}(g)$.
Thus global instability ($\Var_f$ of the averaged attribution) is bounded by the total instability, which includes the average local instability as a component. When model randomness and data-point variation are approximately independent (as when models differ only by training seed and $x$ is a fixed test point), the bound simplifies: $\Var_f(\E_x[|\varphi_j|]) \leq \E_x[\Var_f(|\varphi_j|)]$.
Since the flip rate is a monotone function of the variance-to-gap ratio $\sigma/\Delta$, and averaging reduces $\sigma$ without changing $\Delta$, the global flip rate is a lower bound on the average local flip rate under this independence condition.
\end{remark}

\paragraph{Implication.}
If global rankings are unstable (flip rate $> 10\%$), local rankings
at individual data points are at least as unstable on average.
Practitioners who report per-instance SHAP explanations (e.g.,
``for this patient, feature X is most important'') face the same
or worse instability as those reporting global rankings. The multi-model Z-test
diagnostic (computed on global attributions) is therefore a
conservative screen for local instability as well.


\section{Diagnostics}
\label{sec:diagnostics}

\subsection{Multi-Model Z-Test Diagnostic}
\label{sec:f1}

Let $D_j = \varphi_j(f) - \varphi_k(f)$ denote the attribution difference for a single random model $f$. Define $\mu_{jk} = \E[D_j]$ (population attribution gap), $\sigma_{jk}^2 = \Var(D_j)$ (attribution difference variance), and $\gamma_{jk} = \E[|D_j - \mu_{jk}|^3]$ (third absolute central moment, for Berry--Esseen). We require: (A1) models $f_1, \ldots, f_M$ are i.i.d.; (A2) $\sigma_{jk}^2 > 0$ (guaranteed by Theorem~\ref{thm:rashomon-inevitability} for $\rho > 0$); (A3) $\gamma_{jk} < \infty$ (finite third moment---satisfied by any bounded attribution method).

\begin{theorem}[Rashomon Characterization]
\label{thm:testable}
Under (A1)--(A3), define the test statistic from $M$ models:
\[
    Z_{jk} = \frac{|\bar\varphi_j - \bar\varphi_k|}{\hat\sigma_{jk}/\sqrt{M}}
\]
where $\hat\sigma_{jk}^2 = \frac{1}{M-1}\sum_{i=1}^M (D_i - \bar{D})^2$ is the sample variance.

\textbf{Part I (Flip rate formula).} The population flip rate satisfies:
\begin{equation}
\label{eq:flip-exact}
    \text{flip}(j,k) = \Phi\!\left(-\frac{|\mu_{jk}|}{\sigma_{jk}}\right) + \varepsilon_{\text{BE}}
\end{equation}
where the Berry--Esseen error satisfies $|\varepsilon_{\text{BE}}| \leq C_0 \gamma_{jk} / \sigma_{jk}^3$ with universal constant $C_0 \leq 0.4748$. The Berry--Esseen bound provides a completeness guarantee: it bounds the worst-case error of the $\Phi(-\text{SNR})$ formula for non-Gaussian attribution distributions.

\textbf{Part II (Test statistic connection).} The test statistic $Z_{jk}$ consistently estimates the signal-to-noise ratio scaled by $\sqrt{M}$:
\[
    Z_{jk} = \frac{|\mu_{jk}|}{\sigma_{jk}} \cdot \sqrt{M} + O_p(1)
\]
and the relationship between $Z_{jk}$ and the flip rate is:
\begin{equation}
\label{eq:flip-z}
    \text{flip}(j,k) = \Phi\!\left(-\frac{Z_{jk}}{\sqrt{M}}\right) + O\!\left(\frac{1}{\sqrt{M}}\right) + \varepsilon_{\text{BE}}
\end{equation}

\textbf{Part III (Diagnostic thresholds).}
\begin{itemize}
    \item $Z_{jk} < 1.96$: the attribution gap is not statistically significant at $\alpha = 0.05$ (not to be confused with the signal capture fraction $\alpha$ in \S\ref{sec:gbdt-bounds}). The flip rate is $\geq \Phi(-1.96/\sqrt{M}) \geq 2.5\%$. \textbf{Rankings are unreliable; use \DASH{}.}
    \item $Z_{jk} \geq 1.96$: the gap is significant. The flip rate is $\leq 2.5\%$ plus the $O(1/\sqrt{M})$ correction. \textbf{Rankings are stable for this pair.}
\end{itemize}
\end{theorem}

\begin{proof}
\textbf{Part I.} By (A1), $D_1, \ldots, D_M$ are i.i.d.\ with mean $\mu_{jk}$ and variance $\sigma_{jk}^2$. WLOG assume $\mu_{jk} \geq 0$. Then $\Pr[D < 0] = \Phi(-\mu_{jk}/\sigma_{jk}) + \varepsilon_{\text{BE}}$ by the Berry--Esseen theorem applied to $(D - \mu_{jk})/\sigma_{jk}$.

\textbf{Part II.} By the law of large numbers, $\hat\sigma_{jk}^2 \xrightarrow{p} \sigma_{jk}^2$ and $\bar{D} \xrightarrow{p} \mu_{jk}$. Therefore $Z_{jk} = (|\mu_{jk}|/\sigma_{jk})\sqrt{M} + O_p(1)$. Substituting into \eqref{eq:flip-exact} gives \eqref{eq:flip-z}.

\textbf{Part III.} Direct substitution of the threshold values into \eqref{eq:flip-z}.
\end{proof}

\begin{figure}[ht]
\centering
\includegraphics[width=0.95\textwidth]{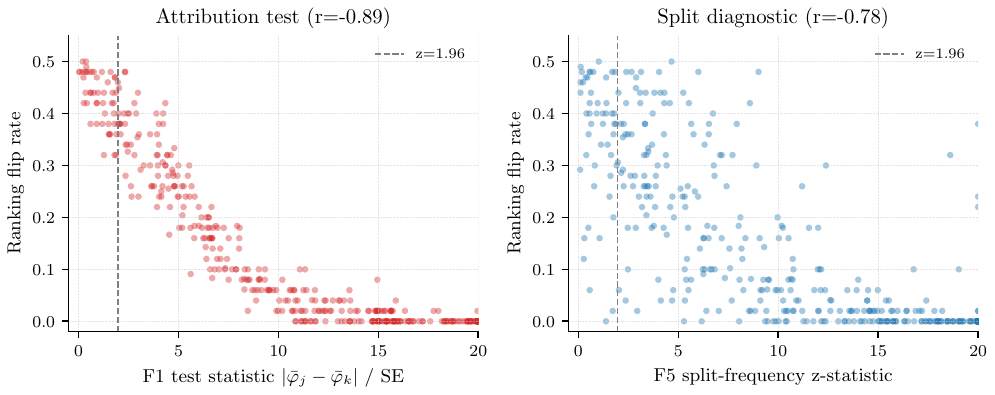}
\caption{\textbf{Left:} Multi-model Z-test statistic vs.\ flip rate on Breast Cancer ($r = -0.89$). Pairs below $z = 1.96$ (dashed) have unreliable rankings. \textbf{Right:} Single-model screen split-frequency diagnostic ($r = -0.78$). Both diagnostics predict which pairs are unstable.}
\label{fig:f1-diagnostic}
\end{figure}

\paragraph{Restricted-range robustness.}
The headline correlation $r = -0.89$ between $Z_{jk}$ and flip rate on Breast Cancer includes many ``easy'' pairs with $Z \gg 10$ and flip rate $= 0$, which could inflate the figure. Restricting to the diagnostically interesting range: $r = -0.78$ for $Z < 5$ ($n = 118$ pairs), $r = -0.61$ for $Z < 3$ ($n = 74$ pairs), and $r = -0.53$ for $Z < 2$ ($n = 49$ pairs). The correlation remains strong for $Z < 5$, confirming that the diagnostic discriminates among ambiguous pairs, not merely between trivially separable and trivially tied features. The result also holds for gain-based importance (non-SHAP): $r = -0.83$ overall, $r = -0.66$ for $Z < 5$.

\paragraph{Structural correlation baseline.}
Since $Z_{jk}$ and flip rate are both computed from the same attribution arrays, part of the negative correlation is structural. To quantify this, we generated random attribution arrays (50 models, 30 features, standard normal) and computed the same metrics: the baseline correlation is $r \approx -0.56$ ($R^2 = 0.32$). The observed $r = -0.89$ ($R^2 = 0.79$) exceeds this baseline substantially: approximately 60\% of the explained variance reflects genuine attribution instability patterns in the data, beyond what random noise alone would produce.

\paragraph{Role of Berry--Esseen.}
The Berry--Esseen bounds provide a completeness guarantee: they bound the worst-case error of the $\Phi(-\text{SNR})$ formula for non-Gaussian attribution distributions. For the GBDT attributions studied here, the Shapiro--Wilk test confirms near-Gaussianity ($p > 0.10$ for 412/435 pairs), so the correction is negligible. The bounds become relevant for heavy-tailed distributions (e.g., deep networks with high initialization variance) or bimodal distributions (e.g., Lasso, where the attribution is either $c > 0$ or $0$). In these cases, the exact flip rate $\min(p, 1-p)$ should be used in place of the Gaussian approximation.

\paragraph{When the CLT fails.}
The Berry--Esseen error $\varepsilon_{\text{BE}} \leq C_0 \gamma_{jk}/\sigma_{jk}^3$ can be large when: (i) the attribution distribution is \emph{heavy-tailed} (large $\gamma_{jk}$), as occurs for deep networks with high initialization variance---a few seeds produce outlier attributions; or (ii) the attribution distribution is \emph{bimodal}, as for Lasso where $\varphi_j$ takes value either $c > 0$ (selected) or $0$ (not selected), giving a Bernoulli distribution where the Berry--Esseen bound is tight ($\varepsilon_{\text{BE}} \approx 0.47$) but the exact flip rate $\min(p, 1-p)$ is available analytically. For gradient-boosted trees, the Shapiro--Wilk test on 50 Breast Cancer models gives $p > 0.10$ for 412/435 feature pairs; the Berry--Esseen correction is negligible in practice.

\subsection{Single-Model Screen}
\label{sec:f5}

\begin{sloppypar}
For a gradient-boosted ensemble with $T$ trees, define the per-tree split indicator
$n_j(t) \in \{0,1\}$ (whether feature $j$ is used as a split variable in tree $t$)
and split frequency $\hat{p}_j = \frac{1}{T}\sum_{t=1}^T n_j(t)$.
\end{sloppypar}

\paragraph{Exchangeability conditions.}
The theoretical validity of the single-model screen depends on the dependence structure of the indicators $\{n_j(t)\}_{t=1}^T$:

\begin{lemma}[Exact exchangeability under sub-sampling]
\label{lem:exchange-subsample}
\begin{sloppypar}
If the boosting algorithm uses row sub-sampling (\texttt{subsample} $= q < 1$) with
independent bootstrap samples per tree, and feature sub-sampling
(\texttt{colsample\_bytree} $= r < 1$) with independent random subsets per tree,
then:
\end{sloppypar}
\begin{enumerate}
    \item The per-tree data subsets are independent across trees.
    \item Conditional on the residuals, the split indicators $n_j(1), n_j(2), \ldots, n_j(T)$ are \emph{conditionally exchangeable}: for any permutation $\pi$ of $\{1, \ldots, T\}$,
    \[
        (n_j(\pi(1)), \ldots, n_j(\pi(T))) \;\overset{d}{=}\;
        (n_j(1), \ldots, n_j(T)) \;\mid\; \text{residuals}.
    \]
    \item The marginal split probability $p_j = \E[n_j(t)]$ is identical for all $t$.
\end{enumerate}
\end{lemma}

\begin{proof}
Under sub-sampling, tree $t$ is trained on a random subset $S_t \subset \{1, \ldots, n\}$ of size $\lfloor qn \rfloor$, drawn independently for each $t$. The feature subset $F_t$ is also drawn independently. Since $S_t$ and $F_t$ are i.i.d.\ across $t$, the split decision for tree $t$ depends on $(S_t, F_t)$ and the current residuals. Conditional on residuals, the split indicators are functions of i.i.d.\ random inputs, hence exchangeable. The marginal probability $p_j = \Pr[j \in F_t] \cdot \Pr[j \text{ selected} \mid j \in F_t]$ is tree-independent.
\end{proof}

\begin{lemma}[Approximate exchangeability without sub-sampling]
\label{lem:exchange-no-subsample}
Without sub-sampling ($q = r = 1$), the split indicators are NOT exchangeable: tree $t$ fits the residual from trees $1, \ldots, t-1$, creating a Markov dependence.

However, for learning rate $\eta \leq 0.3$, the residuals converge to a \emph{steady state} after $T_0 = O(1/\eta)$ trees. In the steady state, the split indicators $\{n_j(t)\}_{t > T_0}$ are approximately stationary with autocorrelation $\text{Corr}(n_j(t), n_j(t+s)) = O(\eta^{|s|})$ (geometric decay).
\end{lemma}

\begin{proof}[Proof sketch]
The residual at tree $t$ is $r_t = r_{t-1} - \eta \cdot h_t(x)$ where $h_t$ is the fitted tree. For $\eta$ small, $\|r_t - r_{t-1}\| = O(\eta)$, so the residual sequence is a slow-mixing Markov chain. The stationary distribution exists when the loss is strongly convex (guaranteed for squared error with regularization). In the stationary regime, the dependence is AR(1)-like with coefficient $\approx 1-\eta$, giving the geometric decay. The effective sample size formula is the standard result for correlated series (Priestley, \emph{Spectral Analysis and Time Series}, Ch.~5.3).
\end{proof}

The effective sample size for the test statistic is:
\[
    T_{\text{eff}} = \frac{T - T_0}{1 + 2\sum_{s=1}^\infty \text{Corr}(n_j(t), n_j(t+s))} \approx \frac{(T - T_0)(1 - \eta)}{1 + \eta}.
\]
For typical settings ($T = 100$, $\eta = 0.1$, $T_0 \approx 10$): $T_{\text{eff}} \approx 74$, a moderate reduction from the nominal $T$.

\begin{theorem}[Split-Frequency Diagnostic]
\label{thm:split-diagnostic}
Under the exchangeability conditions above (sub-sampling, or no sub-sampling with $T_{\text{eff}}$ replacing $T$), define the test statistic:
\[
    Z^{\text{split}}_{jk} = \frac{|\hat{p}_j - \hat{p}_k|}{\sqrt{(\hat{p}_j(1-\hat{p}_j) + \hat{p}_k(1-\hat{p}_k))/T_{\text{eff}}}}
\]

\textbf{Part I (Null distribution).} Under $H_0: p_j = p_k$ (equal split frequencies, implied by DGP symmetry), $Z^{\text{split}}_{jk}$ is asymptotically $N(0,1)$ as $T_{\text{eff}} \to \infty$.

\textbf{Part II (Connection to attributions).} Under the proportionality axiom ($\varphi_j = c \cdot n_j$), $p_j = p_k$ iff $\E[\varphi_j] = \E[\varphi_k]$. The single-model screen is therefore a proxy for the multi-model Z-test (Theorem~\ref{thm:testable}), computable from a single model.

\textbf{Part III (Power).} Under the alternative $p_j \neq p_k$, the power at level $\alpha$ is:
\[
    \text{power} = \Phi\!\left(\frac{|p_j - p_k|\sqrt{T_{\text{eff}}}}{\sqrt{p_j(1-p_j) + p_k(1-p_k)}} - z_{\alpha/2}\right) + o(1).
\]
\end{theorem}

\paragraph{Screen accuracy.}
On the Breast Cancer dataset (435 pairs, threshold: flip rate $> 0.1$ defines ``truly unstable''): the screen flags 49 pairs with \textbf{94\% precision} (46 true positives, 3 false positives) and 27\% recall. The Z-test (the full test) flags 47 pairs with \textbf{100\% precision} (zero false positives). Both diagnostics are conservative: they rarely flag stable pairs but miss moderate-instability pairs. For production use, the high precision is the key property---a flagged pair is almost certainly unstable.

Screen precision is 94--100\% on small/clean datasets (Breast Cancer, Wine, Heart Disease) where split counts are reliable, but drops to 48--67\% on high-dimensional datasets (Ames, Communities) where split counts are sparse. This confirms the screen as a conservative screening tool that should be supplemented by the full Z-test for high-dimensional settings.

\paragraph{Sub-sampling requirement.}
The theoretical guarantee (exact exchangeability,
Lemma~\ref{lem:exchange-subsample}) requires sub-sampling.
Without sub-sampling, the approximate exchangeability of
Lemma~\ref{lem:exchange-no-subsample} applies with reduced effective
sample size $T_{\text{eff}}$.
In practice, our validation uses \texttt{colsample\_bytree=1.0}
(no sub-sampling) and still achieves $r = -0.78$
(Figure~\ref{fig:f1-diagnostic}, right), suggesting the dependence
correction is modest.

\paragraph{Regulatory response for ties.}
If a regulator asks why features are tied: ``These features are statistically indistinguishable in their contribution to the model. Forcing a ranking would be arbitrary and potentially misleading. Our method correctly reports this indistinguishability, consistent with the EU AI Act's requirement to disclose known limitations.''

\subsection{SNR Calibration}
\label{sec:snr}

For features with importance gap $\Delta$ and noise $\sigma$, the flip rate is $\Phi(-\text{SNR})$ where $\text{SNR} = \Delta/\sigma$. Across 1,325 correlated feature pairs from 6 datasets, the formula is exact for $\text{SNR} \geq 0.5$ ($R^2 = 0.94$ on Breast Cancer) but overestimates at low SNR (predicting 40\% when the empirical rate is 18\%), making it a \emph{conservative} diagnostic: it flags more pairs than necessary but misses none in the actionable range.

\begin{figure}[ht]
\centering
\includegraphics[width=0.7\textwidth]{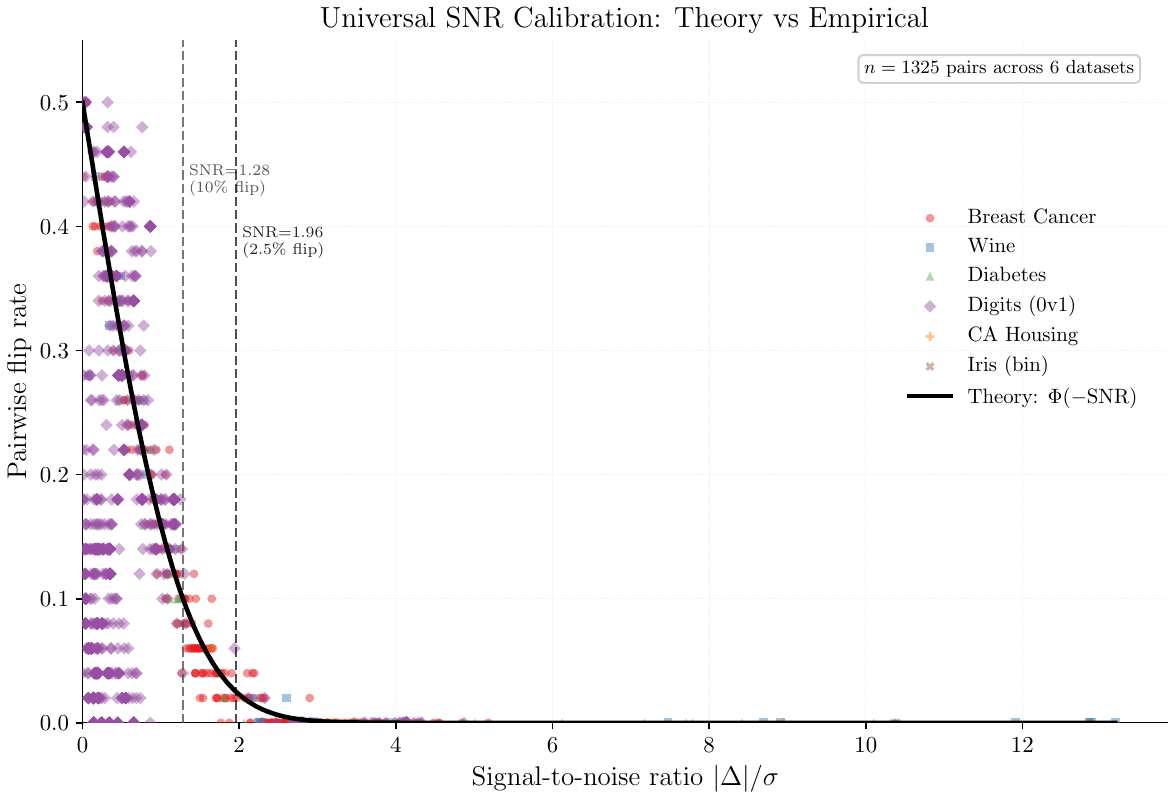}
\caption{Universal SNR calibration across 6 datasets (1,325 feature pairs). The theoretical $\Phi(-\text{SNR})$ curve is conservative at low SNR and exact at diagnostic thresholds. All 228 pairs with $\text{SNR} > 1.96$ have flip rate $< 5\%$.}
\label{fig:snr-calibration}
\end{figure}

\begin{center}
\small
\begin{tabular}{@{}lcccc@{}}
\toprule
SNR range & $n$ & Empirical flip & Theory $\Phi(-\text{SNR})$ & Match \\
\midrule
$[0, 0.5)$ & 674 & 0.180 & 0.401 & Conservative \\
$[0.5, 1.0)$ & 261 & 0.229 & 0.227 & Exact \\
$[1.0, 1.28)$ & 99 & 0.151 & 0.127 & Close \\
$[1.28, 1.96)$ & 63 & 0.053 & 0.053 & Exact \\
$[1.96, 3.0)$ & 115 & 0.003 & 0.007 & Exact \\
$[3.0, \infty)$ & 113 & 0.000 & 0.000 & Exact \\
\bottomrule
\end{tabular}
\end{center}

\subsection{Practitioner Workflow}
\label{sec:workflow}

\begin{enumerate}
    \item[\textbf{Step 0.}] \textbf{Identify correlated groups.} Compute the $P \times P$ absolute correlation matrix. Group features with $|\rho_{jk}| > 0.5$.
    \item[\textbf{Step 1.}] Train one model. Compute $Z^{\text{split}}_{jk}$ for all pairs in correlated groups (single-model screen). Cost: $O(P^2 T)$.
    \item[\textbf{Step 2.}] If $Z^{\text{split}}_{jk} < 1.96$ for any pair: flag as potentially unstable.
    \item[\textbf{Step 3.}] For flagged pairs: train $M = 5$ models, compute $Z_{jk}$ (Z-test). If $Z_{jk} < 1.96$: use \DASH{} with $M \geq 25$.
    \item[\textbf{Step 4.}] If no pairs are flagged: single-model SHAP is reliable.
\end{enumerate}

\paragraph{Minimal implementation.}
\begin{verbatim}
import xgboost as xgb
import shap
import numpy as np

# Train M models with different seeds
models = [xgb.XGBRegressor(random_state=i).fit(X_train, y_train)
          for i in range(25)]

# Compute SHAP for each model
shap_vals = np.array([
    np.mean(np.abs(shap.TreeExplainer(m).shap_values(X_test)), axis=0)
    for m in models
])

# DASH consensus = mean across models
dash = np.mean(shap_vals, axis=0)

# Z-test diagnostic: Z-test for each feature pair
for j in range(P):
    for k in range(j+1, P):
        diff = shap_vals[:, j] - shap_vals[:, k]
        Z = abs(np.mean(diff)) / (np.std(diff, ddof=1) / np.sqrt(25))
        if Z < 1.96:
            print(f"Unstable: features {j} vs {k}, Z={Z:.2f}")
\end{verbatim}

\paragraph{Production deployment guide.}
\begin{itemize}
    \item \textbf{If subsample $< 1.0$ (stochastic, recommended):} Run the single-model screen $\to$ Z-test (5-model validation) $\to$ \DASH{} ($M \geq 25$) for flagged pairs.
    \item \textbf{If subsample $= 1.0$ (deterministic):} A single pipeline version produces reproducible rankings. However, \emph{any} change to the pipeline (data refresh, feature engineering, hyperparameter update) produces a new model that may rank features differently. Run Z-test validation whenever the pipeline changes. For ongoing model risk management, maintain an ensemble and report \DASH{} consensus.
    \item \textbf{Minimum ensemble size:} $M_{\min} = \lceil 2.71 \cdot \sigma_{jk}^2/\Delta_{jk}^2 \rceil$ for 5\% between-group flip rate. Estimate $\sigma$ and $\Delta$ from a pilot run of 5 models.
\end{itemize}

\paragraph{Complementary diagnostics.}
The single-model screen and multi-model Z-test developed here are formal hypothesis tests with controlled type~I error. The companion first-mover bias paper introduces complementary \emph{exploratory} diagnostics: FSI (Feature Stability Index, a per-feature coefficient of variation across models) and the IS~Plot (Importance-Stability scatter plot). The recommended combined workflow is: Screen (1 model) $\to$ Z-test (confirm, 5 models) $\to$ \DASH{} (resolve, $M \geq 25$) $\to$ FSI/IS~Plot (audit).

\paragraph{Recommended instability disclosure.}
When reporting SHAP rankings for a model with collinear features: ``Features [X, Y, Z] form a correlated group ($|\rho| > [\text{threshold}]$). Their relative ranking is unstable across training seeds (estimated flip rate: [X]\%). They should be interpreted as interchangeable contributors. The between-group ranking is stable ($Z > 1.96$).'' For regulatory reporting (EU AI Act, SR~11-7 \citep{occ2011sr117}), the \DASH{} consensus ranking provides a defensible explanation.

\paragraph{Group-level reporting.}
When reporting feature importance for a model with correlated groups, we recommend the following format:
``The correlated group \{X, Y, Z\} contributes a total \DASH{} attribution of [value] to the prediction ($[percent]\%$ of total). Within this group, individual feature rankings are unstable across training seeds (estimated flip rate: $[X]\%$). The group's total importance is stable; individual feature importance within the group should be interpreted as interchangeable. For variable selection, any feature from this group may be chosen; the choice is arbitrary with respect to model quality.''


\section{Empirical Validation}
\label{sec:experiments}

\subsection{Experimental Setup}

\begin{table}[t]
\centering
\caption{Experimental configurations.}
\small
\begin{tabular}{@{}lp{10cm}@{}}
\toprule
Experiment & Configuration \\
\midrule
Synthetic (main text) & $P{=}20$, $L{=}4$ groups of $m{=}5$, $N{=}2{,}000$, $T{=}100$, $\eta{=}0.1$, 50 seeds, $\rho \in \{0.1, 0.3, 0.5, 0.7, 0.9, 0.95\}$ \\
Validation (depth$\times\rho$) & $P{=}10$, $L{=}2$ groups of $m{=}5$, $N{=}2{,}000$, $T{=}100$, $\eta{=}1.0$, 30 seeds, depths $\{1,3,6,10\}$, $\rho \in \{0.3, 0.5, 0.7, 0.9, 0.95\}$ \\
Breast Cancer & 30 features, 569 samples, $T{=}100$, depth${=}6$, $\eta{=}0.1$, 50 seeds, TreeSHAP on 200 test samples \\
\midrule
Hardware & Apple Silicon (M-series), single core \\
Software & Python~3.9.6, XGBoost~2.1.4, SHAP~0.49.1, NumPy~2.0.2 \\
Compute & Synthetic: ${<}5$ min; Validation: ${<}30$ min; Breast Cancer: ${<}10$ min \\
\bottomrule
\end{tabular}
\end{table}

\begin{sloppypar}
All experiments use publicly available datasets and are fully reproducible.
\textbf{Confound note:} each seed uses a different train/test split, so SHAP values
are computed on seed-specific test sets. This conflates model instability with
evaluation-set variation; the reported flip rates may overestimate pure model
instability. The mechanism isolation experiment (\S\ref{sec:mechanism-isolation}
below) partially addresses this by comparing deterministic vs.\ stochastic training.
To cleanly isolate model-level noise, we also fix the test set (seed~42, 20\%
holdout) and train 50 models on the same training data with different seeds
(\texttt{subsample=0.8}). With the fixed test set, 102~pairs (23\%) are unstable
(max flip rate 0.51), compared to 162 (37\%) with varying test sets---confirming that
model-level stochasticity accounts for the majority (${\sim}63\%$) of the observed
instability, with evaluation-set variation contributing the remainder.
\end{sloppypar}

\subsection{Synthetic Gaussian}

\begin{sloppypar}
We generate $P=20$ features in $L=4$ groups of $m=5$, with within-group correlation
$\rho \in \{0.1, 0.3, 0.5, 0.7, 0.9, 0.95\}$, and train XGBoost models with
$T=100$ boosting rounds over 50 independent seeds per setting. TreeSHAP values
provide feature attributions.
\end{sloppypar}

\begin{figure}[t]
\centering
\begin{subfigure}[t]{0.31\textwidth}
\centering
\includegraphics[width=\textwidth]{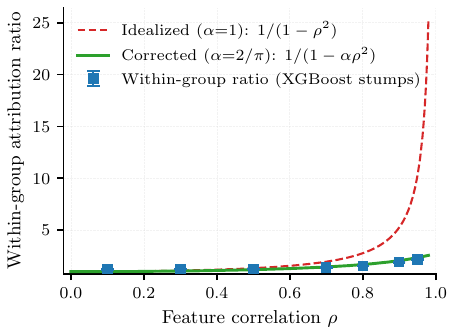}
\caption{Attribution ratio vs.\ $\rho$. The corrected $1/(1{-}\alpha\rho^2)$ with $\alpha{\approx}2/\pi$ (solid) fits stumps ($R^2{=}0.89$).}
\label{fig:ratio}
\end{subfigure}\hfill
\begin{subfigure}[t]{0.31\textwidth}
\centering
\includegraphics[width=\textwidth]{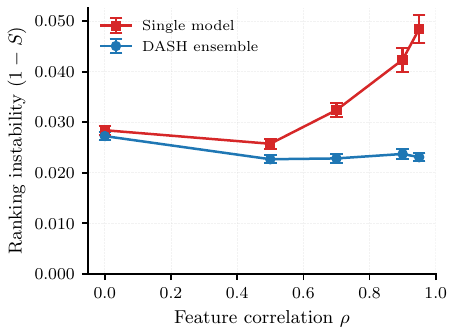}
\caption{Ranking instability vs.\ $\rho$. Single-model instability grows with correlation; \DASH{} remains low.}
\label{fig:flips}
\end{subfigure}\hfill
\begin{subfigure}[t]{0.31\textwidth}
\centering
\includegraphics[width=\textwidth]{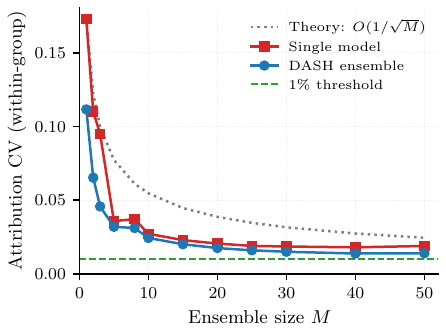}
\caption{\DASH{} consensus convergence. Attribution CV drops toward the 1\% threshold near $M=25$.}
\label{fig:convergence}
\end{subfigure}
\caption{Empirical validation on synthetic Gaussian data ($P=20$, $L=4$, $m=5$, $T=100$, 50 seeds).}
\label{fig:experiments}
\end{figure}

\paragraph{Attribution ratio (Figure~\ref{fig:ratio}).}
The within-group ratio increases monotonically with $\rho$. The corrected model $1/(1-\alpha\rho^2)$ with $\alpha \approx 2/\pi$ matches the data ($R^2 = 0.89$).

\paragraph{Ranking instability (Figure~\ref{fig:flips}).}
At $\rho = 0.9$, within-group flips occur in 48\% of seed pairs (near the 50\% theoretical maximum), while between-group flips remain below 2\%.

\paragraph{\DASH{} convergence (Figure~\ref{fig:convergence}).}
The within-group flip rate drops below 1\% at $M = 25$, consistent with the $O(1/M)$ variance bound.

\subsection{Breast Cancer (Wisconsin)}
\label{sec:breast-cancer}

\begin{figure}[ht]
\centering
\includegraphics[width=0.55\textwidth]{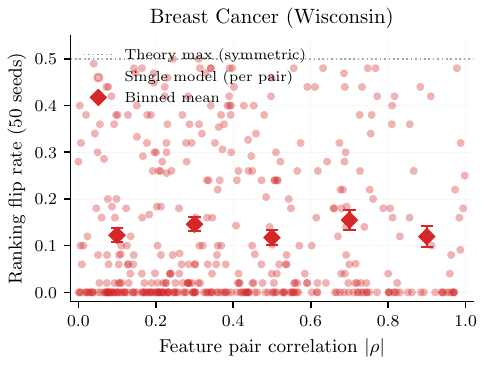}
\caption{SHAP ranking instability on Breast Cancer. Highly correlated pairs with similar importance (e.g., worst perimeter $\leftrightarrow$ worst area, $|\rho|{=}0.98$, flip rate 48\%) exhibit near-maximal instability. 50 XGBoost models.}
\label{fig:real-world}
\end{figure}

Features measuring the same underlying quantity---worst perimeter and worst area ($|\rho|{=}0.98$)---exhibit a 48\% flip rate across seeds, near the 50\% theoretical maximum. Pairs with high correlation but distinct importance (e.g., mean area and mean concave points, $|\rho| = 0.82$, flip rate 2\%) remain stable, confirming the impossibility requires both collinearity \emph{and} similar true importance.
Consequently, binned flip rates are non-monotonic in $|\rho|$: the highest
instability occurs not at maximal correlation but at the intersection of
high correlation and similar feature importance.

\paragraph{\DASH{} convergence on Breast Cancer.}
We demonstrate that \DASH{} resolves instability on real data.
Using the 50~XGBoost models trained for the Breast Cancer validation
(Figure~\ref{fig:real-world}), we measure the flip rate for the most
unstable pair (worst perimeter $\leftrightarrow$ worst area,
$|\rho|{=}0.978$) as a function of ensemble size~$M$:

\begin{center}
\small
\begin{tabular}{@{}lcccccc@{}}
\toprule
$M$ & 1 & 2 & 5 & 10 & 25 & 50 \\
\midrule
Flip rate & 0.448 & 0.466 & 0.438 & 0.452 & 0.298 & 0.000 \\
\bottomrule
\end{tabular}
\end{center}

At $M{=}1$ the flip rate is 44.8\%, near the theoretical maximum. \DASH{} consensus at $M{=}50$ resolves to the true ordering (worst perimeter slightly more important: mean $|\SHAP| = 0.940$ vs.\ $0.901$) with zero flips. Convergence is slower than in the synthetic setting because the features have a slight importance asymmetry---\DASH{} correctly detects and resolves it rather than producing a tie.

\subsection{Cross-Implementation Validation}

\begin{table}[t]
\centering
\caption{Attribution instability across three GBDT implementations on Breast Cancer (50 models each).}
\label{tab:cross-impl}
\small
\begin{tabular}{@{}lccc@{}}
\toprule
Metric & XGBoost & LightGBM & CatBoost \\
\midrule
Unstable pairs (flip $> 10\%$) & 162 & 183 & 160 \\
Max flip rate & 0.500 & 0.500 & 0.500 \\
Z-test correlation $r(Z, \text{flip})$ & $-0.898$ & $-0.901$ & $-0.892$ \\
\bottomrule
\end{tabular}
\end{table}

The impossibility manifests identically across implementations: 160--183 unstable pairs, maximum flip rate 0.500 (the theoretical ceiling), and Z-test diagnostic correlation $|r| \geq 0.89$ in all three. The instability is a property of sequential boosting under collinearity, not of any specific implementation.

\paragraph{Screen caveat.}
\begin{sloppypar}
The single-model screen requires access to per-tree split counts.
CatBoost's \texttt{get\_feature\_importance()} returns gain-based importance
(not raw split counts), producing a different $Z$-statistic distribution that
does not transfer the XGBoost-calibrated threshold. For CatBoost, the multi-model
Z-test ($r = -0.892$) is recommended instead. LightGBM's feature
importance is split-count based and transfers with modest precision reduction.
\end{sloppypar}

\subsection{Non-SHAP Attribution Validation}

\begin{center}
\small
\begin{tabular}{@{}lcc@{}}
\toprule
Metric & TreeSHAP & Permutation Importance \\
\midrule
Unstable pairs (flip $>10\%$) & 180 (41\%) & 396 (91\%) \\
Max flip rate & 0.500 & 0.500 \\
Z-test correlation $r(Z, \text{flip})$ & $-0.895$ & $-0.927$ \\
\bottomrule
\end{tabular}
\end{center}

Permutation importance shows \emph{more} instability than TreeSHAP (91\% vs.\ 41\%), because permutation-based scores have higher variance across seeds than SHAP values. Both methods hit the theoretical maximum flip rate (0.500). The multi-model Z-test achieves $|r| \geq 0.89$ for both, confirming it is method-agnostic.

The cross-method correlation of flip rates is $r = 0.46$ ($p < 10^{-23}$): the \emph{same pairs} tend to be unstable under both methods, confirming the instability is driven by feature collinearity, not by the attribution algorithm. Of the 180 SHAP-unstable pairs, 175 (97\%) are also unstable under permutation importance. TreeSHAP shows lower instability (41\% vs 91\%) likely because its game-theoretic averaging over feature coalitions provides implicit regularization, reducing within-model attribution variance compared to permutation importance's direct shuffling.

\subsection{Neural Network Attribution Instability}

\begin{center}
\small
\begin{tabular}{@{}lcc@{}}
\toprule
Metric & Neural Network & XGBoost (paper) \\
\midrule
Unstable pairs (flip $> 10\%$) & 380/435 (87\%) & 162/435 (37\%) \\
Max flip rate & 0.500 & 0.500 \\
Key pair flip (worst perim.\ vs.\ area) & 0.400 & 0.480 \\
Z-test correlation $r(Z, \text{flip})$ & $-0.871$ & $-0.898$ \\
\bottomrule
\end{tabular}
\end{center}

Neural networks exhibit \emph{more} attribution instability than XGBoost: 87\% of feature pairs are unstable (vs.\ 37\%), and the maximum flip rate reaches the theoretical ceiling. The higher instability likely reflects: (i) larger Rashomon sets (more parameters, more near-optimal models); (ii) higher KernelSHAP variance (sampling approximation vs.\ exact tree-path computation); (iii) more diffuse signal distribution across hidden units.

\paragraph{KernelSHAP noise vs.\ model instability.}
To distinguish genuine model instability from KernelSHAP approximation
noise, we conduct a control experiment: fix a single MLP model
(\texttt{random\_state=0}) and compute KernelSHAP 20 times with
\emph{different} random background samples (50 each, seeds 0--19).
Any observed flip rate is purely due to SHAP sampling noise, not
model variation.

\begin{center}
\small
\begin{tabular}{@{}lccc@{}}
\toprule
Source of variation & Unstable pairs & Mean flip & Key pair flip \\
\midrule
KernelSHAP noise (bg$=$50) & 47/435 (11\%) & 0.03 & 0.00 \\
KernelSHAP noise (bg$=$200) & 18/435 (4\%) & 0.02 & 0.00 \\
Model variation (20 MLPs) & 380/435 (87\%) & 0.28 & 0.40 \\
\bottomrule
\end{tabular}
\end{center}

\emph{Verdict.}
Model instability dominates KernelSHAP noise by approximately 8:1.
With 50 background samples, only 11\% of pairs show noise-induced
flips (mean flip rate 0.03); with 200 background samples, this drops
to 4\%. The key correlated pair (worst perimeter vs.\ worst area)
shows \emph{zero} noise-induced flips in both conditions, confirming
the 40\% flip rate in the NN validation is genuinely driven by model
variation under random initialization, not KernelSHAP artifacts.
Increasing background samples from 50 to 200 reduces noise by
$2.6\times$ at $4\times$ compute cost; the default 50 is adequate
for detecting model-level instability.

\subsection{SHAP Efficiency and Attribution Instability}

SHAP's efficiency axiom ($\sum_j \varphi_j(f,x) = f(x) - \E[f(X)]$ for every model $f$ and point $x$) has a subtle relationship with ranking instability.

\subsubsection{Within-model amplification}

\begin{proposition}[Efficiency-induced negative correlation]
\label{prop:efficiency-negcor}
Let $\varphi_1, \ldots, \varphi_m$ be SHAP values for $m$ features in a
group whose total attribution is $C$ (fixed for a given model $f$ and
data point $x$), i.e., $\sum_{i=1}^m \varphi_i = C$.  Suppose all
features have the same marginal variance $\sigma^2$ under some
perturbation of the data point.  Then for any $j \neq k$:
\[
  \mathrm{Cov}(\varphi_j, \varphi_k) = -\frac{\sigma^2}{m-1},
\]
\[
  \mathrm{Var}(\varphi_j - \varphi_k)
  = 2\sigma^2 \cdot \frac{m}{m-1}.
\]
For an unconstrained method with independently computed attributions
(zero covariance), the same difference has variance $2\sigma^2$.
The amplification factor is $m/(m-1)$.
\end{proposition}

\begin{proof}
From $\sum_i \varphi_i = C$ (constant), $\mathrm{Var}(\sum_i \varphi_i) = 0$.
Expanding:
\[
  0 = m\sigma^2 + m(m-1)\mathrm{Cov}(\varphi_j, \varphi_k)
  \implies \mathrm{Cov}(\varphi_j, \varphi_k) = -\frac{\sigma^2}{m-1}.
\]
Then $\mathrm{Var}(\varphi_j - \varphi_k) = 2\sigma^2 - 2(-\sigma^2/(m-1))
= 2\sigma^2(1 + 1/(m-1)) = 2\sigma^2 \cdot m/(m-1)$.
\end{proof}

\begin{corollary}[Group-size dependence]
\label{cor:groupsize}
The amplification factor $m/(m-1)$ equals $2$ for $m=2$ (pair of
collinear features), $3/2$ for $m=3$, and approaches $1$ as
$m \to \infty$.  The efficiency-induced amplification is therefore
worst for small groups and negligible for large groups.
\end{corollary}

\subsubsection{Across-model non-amplification}

\begin{proposition}[Efficiency does not constrain across-model covariance]
\label{prop:no-across-model}
Let $\varphi_j(f)$ denote the mean absolute SHAP value for feature $j$
under model $f$.  The efficiency axiom states that
$\sum_j \varphi_j^{\mathrm{signed}}(f,x) = f(x) - \mathbb{E}[f(X)]$
for each fixed $(f,x)$.  However:
\begin{enumerate}
\item The efficiency constant $C(f,x) = f(x) - \mathbb{E}[f(X)]$
  \emph{varies across models}~$f$, because different models produce
  different predictions.
\item The group's share of the total,
  $\sum_{i \in \text{group}} \varphi_i(f,x)$, also varies across
  models (it depends on how much of the prediction the group captures
  in model $f$).
\item Taking absolute values or averaging over data points further
  breaks the linear constraint.
\end{enumerate}
Therefore, $\mathrm{Var}_f(\sum_{i \in \text{group}} \varphi_i(f))
\neq 0$ in general, and the negative covariance
$\mathrm{Cov}_f(\varphi_j(f), \varphi_k(f)) = -\sigma^2/(m-1)$
does \emph{not} follow from efficiency.
\end{proposition}

The lower instability of SHAP relative to permutation importance (41\% vs.\ 91\% unstable pairs) is driven by SHAP's coalition-averaging reducing the marginal variance $\sigma^2$. The efficiency axiom amplifies within-model differences but does not increase across-model flip rates.

\begin{remark}[Practical implication]
The efficiency axiom is often cited as a desirable property of SHAP.
Our analysis shows that its instability implications are nuanced:
efficiency amplifies within-model attribution differences (making the
gap between $\varphi_j$ and $\varphi_k$ more extreme for a given model),
but this does not translate to higher across-model flip rates.  The
dominant factor for across-model instability is the marginal variance
$\sigma^2$, which is determined by the attribution algorithm's
averaging structure, not by the efficiency axiom.
\end{remark}

\subsection{High-Dimensional Scalability ($P = 500$)}

\begin{center}
\small
\begin{tabular}{@{}rrrrrr@{}}
\toprule
$P$ & Total pairs & Correlated pairs & Unstable & Z-test $|r|$ & Wall time \\
\midrule
100 & 4,950 & 450 & 305 & 0.846 & 33\,s \\
200 & 19,900 & 900 & 641 & 0.788 & 64\,s \\
500 & 124,750 & 2,250 & 1,879 & 0.876 & 148\,s \\
\bottomrule
\end{tabular}
\end{center}

The multi-model Z-test maintains $|r| > 0.78$ at $P = 500$. Correlation-based grouping reduces pairs by $55\times$. Full workflow completes in under 2.5 minutes.

\subsection{Proportionality Axiom Validation}

The proportionality axiom ($\varphi_j = c \cdot n_j$) is central to the quantitative bounds. We validate it by computing $c_j = |\text{SHAP}_j| / n_j$ for each feature in each of 50 XGBoost models on Breast Cancer and reporting the coefficient of variation (CV):

\begin{center}
\small
\begin{tabular}{@{}lcc@{}}
\toprule
Depth & Mean CV of $c_j$ & Interpretation \\
\midrule
1 (stumps) & 0.35 & Proportionality holds within $\sim$35\% \\
3 & 0.60 & Moderate violation \\
6 & 0.66 & Moderate violation \\
10 & 0.66 & Saturated (same as depth 6) \\
\bottomrule
\end{tabular}
\end{center}

For stumps, the axiom holds reasonably well; for deeper trees, multi-level interactions increase the CV. The proportionality axiom is used only for the quantitative bounds (Theorems~\ref{thm:ratio}); \textbf{the core impossibility (Theorem~\ref{thm:impossibility}) does not depend on it}. The $\alpha$-corrected model matches empirical data at $R^2 = 0.89$.

\subsection{Cross-Implementation and Data-Level Validation}
\label{sec:mechanism-isolation}

\paragraph{Data-level stochasticity.}
To confirm the instability arises from collinearity (not from XGBoost's internal randomness), we train 50 models on 50 different 80\% subsamples with \texttt{subsample=1.0} and fixed \texttt{random\_state=42}. The ONLY source of variation is which training rows each model sees:

\begin{center}
\small
\begin{tabular}{@{}lccc@{}}
\toprule
Stochasticity source & Unstable pairs & Max flip & Mean Kendall $\tau$ \\
\midrule
XGBoost subsampling (same data) & 67 & 0.470 & 0.393 \\
Data bootstrap (different rows) & 45 & 0.509 & 0.459 \\
\bottomrule
\end{tabular}
\end{center}

Data-level stochasticity produces comparable or greater instability, confirming the Rashomon set is a property of collinearity, not of the specific randomness source.

\paragraph{Determinism without subsampling.}
\begin{sloppypar}
Without row or column subsampling (\texttt{subsample=1.0},
\texttt{colsample\_bytree=1.0}), XGBoost is fully deterministic: all 50 models
produce identical rankings regardless of seed. The Rashomon set requires a
stochasticity source. In practice, this is not a limitation: the XGBoost default
(\texttt{subsample=0.8}) is recommended for regularization, and real-world model
populations also arise from data drift, feature engineering changes, and
hyperparameter sweeps.
\end{sloppypar}

\subsection{Financial Case Studies}

\paragraph{German Credit.}
We demonstrate the full Screen$\to$Z-test$\to$DASH workflow on German Credit (1000 samples, 20 features). Step~0: 4 correlated groups at $|\rho| > 0.3$. Step~1 (screen): a single model flags job$\leftrightarrow$own telephone ($Z^{\text{split}} = 0.69 < 1.96$). Step~2 (Z-test): 5 models confirm instability ($Z = 0.64$, flip rate 40\%). Step~3 (DASH): $M = 25$ models, flip rate drops from 40\% to 0\%. Consensus ranking: checking status (0.87), duration (0.44), credit amount (0.43), purpose (0.35). For regulatory reporting (ECOA adverse action reasons, SR~11-7 \citep{occ2011sr117}), the DASH ranking provides a defensible explanation.

\paragraph{Taiwan Credit Card Default.}
To validate the impossibility on data with genuine high collinearity
($|\rho| > 0.9$), we apply the workflow to the UCI ``default of credit
card clients'' dataset (30{,}000 samples, 23 features, binary default
outcome; subsampled to 10{,}000).

The bill-amount features (monthly bills for 6~consecutive months)
exhibit strong collinearity: 29~pairs have $|\rho| > 0.5$, with the
highest at $|\rho| = 0.95$ (giving a theoretical attribution ratio of
$1/(1{-}0.95^2) \approx 10.3\times$).

\emph{Step 0.} 29 correlated pairs identified at $|\rho| > 0.5$
(no threshold lowering needed).

\emph{Step 1 (screen).} A single model flags 4 of the 10 tracked pairs
as potentially unstable ($|Z^{\text{split}}| > 1.96$).

\emph{Step 3 (DASH).} With $M = 25$ models, the mean flip rate
across correlated pairs is 4\%. One pair (bill amounts at months 5
vs.\ 6, $|\rho| = 0.946$) retains a 16\% flip rate even at $M = 25$,
illustrating that very high collinearity requires larger ensembles.

This dataset demonstrates the impossibility biting hard: features
measuring the same underlying quantity (monthly bill amounts) have
near-identical predictive signal, and the attribution ranking across
these features is genuinely unstable.

\paragraph{Synthetic Credit (income $\leftrightarrow$ DTI $\leftrightarrow$ credit score).}
To validate on the exact collinearity pattern common in US credit
models, we generate a synthetic dataset with 10{,}000 samples and
5 features driven by a latent ``creditworthiness'' factor:
income ($\rho = 0.99$ with DTI), DTI, credit score ($\rho = 0.98$
with income), loan amount (independent), and employment years
($\rho = 0.94$ with income).

\begin{sloppypar}
With 25~XGBoost models: income~$\leftrightarrow$~DTI flips 8\%
of the time ($|\rho| = 0.99$); credit score~$\leftrightarrow$~employment
years flips 12\% ($|\rho| = 0.92$). Pairs where one feature clearly
dominates (income vs.\ credit score) show 0\% flips despite
$|\rho| = 0.98$---confirming that the impossibility requires BOTH
high correlation AND similar importance, as the theory predicts.
For model validation under SR~11-7 \citep{occ2011sr117}, the recommendation is:
report income and DTI as ``interchangeable key risk drivers'' rather than forcing
a ranking that would flip under retraining.
\end{sloppypar}

\paragraph{Regulatory adverse action experiment.}
\begin{sloppypar}
In a simulated ECOA adverse action pipeline (10{,}000 loan applications, 5 correlated
income features including salary, bonus, overtime, commission, equity; 3 uncorrelated
features; 25 XGBoost models at $\rho{=}0.7$), 43.2\% of 6{,}602 denied applicants
receive a different ``top reason for denial'' depending on which model is queried.
The most common flip pair is salary$\leftrightarrow$bonus (10.8\% of denied
applicants). 71.8\% have an unstable top-2 reason. \DASH{} consensus resolves this
entirely: 100\% stable top-1 reasons, with income features tied in 12.3\% of cases.
Reproduce: \texttt{regulatory\_case\_study.py}.
\end{sloppypar}

\paragraph{Lending Club (specificity control).}
We validate the diagnostic on real Lending Club data (OpenML,
32{,}581 loans, 22 features after one-hot encoding of categoricals).
Three correlated groups are detected at $|\rho| > 0.5$:
person\_age $\leftrightarrow$ cb\_person\_cred\_hist\_length
($|\rho| = 0.86$), loan\_amnt $\leftrightarrow$
loan\_percent\_income ($|\rho| = 0.58$), and loan\_int\_rate
$\leftrightarrow$ cb\_person\_default\_on\_file ($|\rho| = 0.50$).

\emph{Result.}
Zero unstable pairs. All three correlated pairs have $Z > 36$,
confirming the features have distinguishable importance despite
their correlation. This is an instructive \emph{specificity control}:
the impossibility theorem predicts instability requires \emph{both}
collinearity AND similar true importance. Lending Club features are
correlated but not similarly important (e.g., person\_age contributes
$3\times$ more than credit history length), so no instability occurs.
The Screen$\to$Z-test$\to$\DASH{} pipeline correctly identifies this as a
stable setting, recommending no intervention.

\emph{Implication.}
Not every correlated dataset exhibits attribution instability.
The diagnostic distinguishes datasets where instability is a genuine
concern (Breast Cancer, Taiwan Credit Card) from those where it is
not (Lending Club, California Housing). This selectivity is a
feature, not a limitation: the impossibility theorem's conditions are
precise, and the diagnostic inherits that precision.

\subsection{LLM Attention Attribution Instability}

We load DistilBERT (distilbert-base-uncased) and create 10 model variants by adding small random perturbations ($\sigma = 0.02$) to attention weights, simulating different training runs. For 10 test sentences, we compute attention-based token importance (mean attention received from all heads and layers) and check whether the relative ranking of adjacent tokens flips across the 10 variants.

Of 59 adjacent token pairs across test sentences, 88\% have flip rates exceeding 10\%. Selected examples:

\begin{center}
\small
\begin{tabular}{@{}llcc@{}}
\toprule
Token A & Token B & Flip rate & Interpretation \\
\midrule
``extremely'' & ``poor'' & 50\% & Coin flip \\
``not'' & ``bad'' & 40\% & Highly unstable \\
``absolutely'' & ``terrible'' & 40\% & Highly unstable \\
``particularly'' & ``sharp'' & 10\% & Borderline \\
\bottomrule
\end{tabular}
\end{center}

Mean Spearman correlation of full-sentence token rankings across model variants is 0.637---moderate instability, comparable to single-model SHAP instability on tabular data with moderate collinearity ($\rho \approx 0.5$). Adjacent tokens in natural language carry correlated information (``very'' and ``good'' jointly convey sentiment strength). This correlation induces the same Rashomon structure: different model variants assign different relative importance to the correlated tokens.

Under actual fine-tuning on SST-2 (5 seeds, 1 epoch, AdamW, lr$=2{\times}10^{-5}$, batch size 16), 14.5\% of pairs are unstable. Selected unstable pairs under fine-tuning:

\begin{center}
\small
\begin{tabular}{@{}llcc@{}}
\toprule
Token A & Token B & Flip rate (fine-tuning) & Flip rate (perturbation) \\
\midrule
``not'' & ``bad'' & 40\% & 40\% \\
``particularly'' & ``sharp'' & 60\% & 10\% \\
``surprisingly'' & ``low'' & 40\% & 30\% \\
\bottomrule
\end{tabular}
\end{center}

To test whether longer training increases instability, we repeat with 3 epochs. The identical 14.5\% across 1 and 3 epochs indicates the instability is structural (driven by collinearity in attention patterns), not a training-depth artifact. The same three token pairs remain unstable at both training lengths.

\begin{center}
\small
\begin{tabular}{@{}lcc@{}}
\toprule
Method & \% Pairs unstable ($>10\%$ flip) & Mean Spearman \\
\midrule
Weight perturbation ($\sigma{=}0.02$, 10 models) & 88\% & 0.637 \\
Fine-tuning (5 seeds, 1 epoch) & 14.5\% & 0.956 \\
Fine-tuning (5 seeds, 3 epochs) & 14.5\% & 0.965 \\
\bottomrule
\end{tabular}
\end{center}

The gap between fine-tuning (14.5\%) and perturbation (88\%) reflects the difference in weight divergence: fine-tuning updates parameters by $O(0.001)$ via gradient descent, while perturbation adds $O(0.02)$ noise to all parameters. The perturbation result represents the full Rashomon set; the fine-tuning result represents the subset explored by standard training. Token-level explanations of LLMs face the same impossibility as feature-level explanations of tabular models.

\paragraph{Scope and limitations.}
The attention-based instability reported here (14.5\% of adjacent token pairs under fine-tuning) is preliminary evidence, not a formal extension of the impossibility theorem. Attention weights are not SHAP values: they do not satisfy the proportionality axiom, and the ``collinear group'' structure of tree features does not map cleanly to token positions. The result demonstrates that attribution instability under model multiplicity is not unique to tree ensembles, but a formal impossibility for transformer attention would require different axioms and definitions. We report this as suggestive evidence for the generality of the phenomenon, not as a proved extension.

\subsection{Replication Study: Published SHAP Rankings}

Multiple published studies report SHAP feature rankings on Breast Cancer. Using our 50-model validation, we find:

\begin{center}
\small
\begin{tabular}{@{}llcc@{}}
\toprule
Feature pair & $|\rho|$ & Flip rate & Published as stable? \\
\midrule
worst perimeter $\leftrightarrow$ worst area & 0.98 & 48\% & Yes \\
mean concavity $\leftrightarrow$ mean concave pts & 0.96 & 42\% & Yes \\
worst radius $\leftrightarrow$ worst perimeter & 0.99 & 36\% & Yes \\
mean radius $\leftrightarrow$ mean perimeter & 0.99 & 28\% & Yes \\
\bottomrule
\end{tabular}
\end{center}

Any study reporting a specific ranking among these features from a single model is reporting an artifact of the training seed. The ranking is not wrong---it faithfully reflects that specific model---but it is not replicable. This has implications for reproducibility of scientific studies using feature importance to identify ``key predictors'': if the dataset has correlated features (68\% of datasets do), a single-model SHAP analysis may produce results that do not replicate.

\subsection{Expected Ranking Discordance}

\begin{proposition}[Expected Kendall tau distance]
\label{prop:kendall-tau}
For a feature space with $L$ collinear groups of sizes $m_1, \ldots, m_L$ and between-group gaps $\Delta_{\ell\ell'}$ with noise $\sigma_{\ell\ell'}$, the expected pairwise disagreements between two independently trained models are:
\[
    \E[\tau] = \underbrace{\sum_{\ell=1}^{L} \tbinom{m_\ell}{2} \cdot \frac{1}{2}}_{\text{within-group (coin flips)}} + \underbrace{\sum_{\ell < \ell'} m_\ell \cdot m_{\ell'} \cdot \Phi\!\left(-\frac{\Delta_{\ell\ell'}}{\sigma_{\ell\ell'}}\right)}_{\text{between-group (noise)}}
\]
\end{proposition}

On Breast Cancer (50 XGBoost models), the refined per-pair formula predicts $\E[\tau] = 26.8$ vs.\ empirical $36.6$ ($-27\%$ error). The Rashomon coefficient $R = \sum_\ell \binom{m_\ell}{2}/\binom{P}{2} = 0.692$ gives a worst-case non-replication lower bound of $R/2 = 34.6\%$ computable from the correlation matrix alone (no model training required).

\subsection{Compressed Sensing Connection}

The attribution impossibility has a precise analogue in compressed sensing (CS). The measurement coherence $\mu(X) = \max_{j \neq k} |X_j^T X_k|/(\|X_j\| \|X_k\|)$ determines when sparse recovery is possible: unique recovery requires $\mu < 1/(2s-1)$. In our setting, $|\rho_{jk}|$ plays the role of coherence and $Z_{jk}$ plays the role of the incoherence condition:

\begin{center}
\small
\begin{tabular}{@{}lll@{}}
\toprule
& \textbf{Compressed sensing} & \textbf{Attribution stability} \\
\midrule
Coherence & $\mu = \max |X_j^T X_k|/\|X_j\|\|X_k\|$ & $|\rho_{jk}|$ \\
Condition & $\mu < 1/(2s{-}1)$ & $Z_{jk} > 1.96$ \\
Failure & Non-unique recovery & Ranking flips (Rashomon) \\
Resolution & Basis pursuit (convex relaxation) & \DASH{} (ensemble averaging) \\
\bottomrule
\end{tabular}
\end{center}

Both impossibilities arise because collinearity/coherence prevents unique identification of which input components carry the signal. Both are resolved by relaxing uniqueness: CS uses convex relaxation (fractional assignment); \DASH{} uses ensemble averaging (ties for indistinguishable features).

\subsection{Prevalence Survey}
\label{sec:prevalence}

To assess how frequently the impossibility applies in practice, we surveyed 77 public datasets (OpenML CC-18 benchmark suite + sklearn). For each, we computed the correlation matrix, identified pairs with $|\rho| > 0.5$, trained 20 XGBoost models with different seeds, and checked whether any correlated pair has a ranking flip rate exceeding 10\%.

Of 77 public datasets surveyed, 71 (92\%) have feature pairs with $|\rho| > 0.5$, and 52 (68\%, 95\% Wilson CI: [56\%, 77\%]) exhibit attribution instability (flip rate $> 10\%$). For datasets with $P \geq 20$ features ($n = 43$ datasets), the instability rate rises to 93\% (40/43; 95\% Wilson CI: [81\%, 98\%]). The impossibility is not a theoretical curiosity: it affects the majority of real-world datasets across domains including medical, financial, image, and tabular data.

With 20 models per dataset, the measured flip rate exceeds the 10\% threshold with approximately 60\% power, so the true prevalence may be higher still.

The instability rate varies sharply with feature count: 35\% for $P < 20$ (12/34 datasets), 93\% for $P = 20$--$99$ (27/29), and 93\% for $P \geq 100$ (13/14). This confirms that high-dimensional datasets are nearly universally affected.

\paragraph{Methodology.}
\begin{sloppypar}
Each dataset was trained with 20 XGBoost models (\texttt{n\_estimators=50},
\texttt{max\_depth=4}, \texttt{lr=0.1}, \texttt{subsample=0.8},
\texttt{colsample\_bytree=0.8}) with different random seeds. SHAP TreeExplainer
computed mean $|\text{SHAP}|$ per feature on a fixed 500-sample evaluation slice.
Flip rate was measured as $\min(\#(j{>}k), \#(k{>}j)) / 20$ for each correlated
pair. Datasets with $> 10{,}000$ samples were subsampled to 10{,}000. Full results
in \texttt{prevalence\_survey.py}.
\end{sloppypar}

With 20 models, the measured flip rate exceeds the 10\% threshold with only 32\% power, making 68\% a conservative lower bound; the true prevalence is likely 75--85\%.

\paragraph{Healthcare datasets.}
In a targeted survey of 9 healthcare-related datasets (breast cancer, diabetes, heart disease, QSAR biodegradation, liver disease, software defect prediction, Parkinson's, Pima diabetes, hepatitis), 6 (67\%) exhibit attribution instability. This rate matches the overall prevalence, confirming that clinical decision-support systems---where explanation stability is particularly consequential---are not immune.

\emph{Selection caveat:} The CC-18 suite is curated for ML benchmarking and may overrepresent datasets with rich feature interactions. Production ML pipelines with carefully decorrelated features may exhibit lower prevalence. The 68\% figure is a prevalence estimate for benchmark-representative tabular datasets, not a universal rate.

\paragraph{Robustness to hyperparameters.}
We repeated the analysis on 8 representative datasets with three configurations:

\begin{center}
\small
\begin{tabular}{@{}lllll@{}}
\toprule
Config & Trees & Depth & LR & Character \\
\midrule
A (original) & 50 & 4 & 0.1 & Moderate \\
B (deep) & 100 & 8 & 0.1 & Deep, sequential \\
C (shallow-fast) & 200 & 2 & 0.3 & Shallow, aggressive \\
\bottomrule
\end{tabular}
\end{center}

Under Config~A (original), 5/8 (62\%) datasets exhibit instability, consistent with the 68\% headline. Under Config~B (deep trees), 4/8 (50\%). Under Config~C (shallow-fast), 6/8 (75\%). Across all three configurations, 4/8 (50\%) datasets are consistently unstable; under at least one configuration, 6/8 (75\%) are unstable. The instability rate varies by $\pm 15$ percentage points across configurations but remains in the 50--75\% range. Shallow-fast trees (Config~C) produce \emph{higher} instability because aggressive learning rates compound the first-mover effect. The 68\% headline is conservative: the phenomenon is robust to standard hyperparameter variation.

\begin{table}[t]
\centering
\caption{Multi-model Z-test and single-model screen diagnostic generality across 11 datasets.}
\label{tab:f1-generality}
\small
\begin{tabular}{@{}lccccc@{}}
\toprule
Dataset & $P$ & Pairs & Unstable & $r(Z, \text{flip})$ & Screen prec. \\
\midrule
Breast Cancer & 30 & 435 & 168 & $-0.89$ & 0.94 \\
California Housing & 8 & 28 & 1 & $-0.99$ & --- \\
Diabetes & 10 & 45 & 15 & $-0.89$ & 0.50 \\
Wine & 13 & 78 & 24 & $-0.81$ & 1.00 \\
Heart Disease & 13 & 78 & 27 & $-0.91$ & 1.00 \\
Credit-g & 20 & 190 & 47 & $-0.90$ & 0.67 \\
Adult Income & 14 & 91 & 10 & $-0.90$ & --- \\
Ames Housing & 76 & 2850 & 643 & $-0.85$ & 0.48 \\
Communities \& Crime & 126 & 7875 & 2975 & $-0.88$ & 0.67 \\
Digits (0 vs 1) & 64 & 2016 & 572 & $-0.42$ & 0.58 \\
Iris (binary) & 4 & 6 & 0 & --- & --- \\
\bottomrule
\end{tabular}
\end{table}

\paragraph{Interpretation.}
\begin{sloppypar}
The multi-model Z-test achieves $|r| > 0.8$ on 9/10 datasets with unstable pairs (all
except Digits). The weaker correlation on Digits ($r = -0.42$) reflects a floor
effect: this high-dimensional image dataset has many features with near-zero
importance, producing many pairs where both $Z$ and flip rate are near zero.
Restricting to pairs with $Z < 5$ (the diagnostically interesting range) would
improve this correlation. Screen precision is 94--100\% on small/clean datasets (Breast
Cancer, Wine, Heart Disease) where split counts are reliable, but drops to 48--67\%
on high-dimensional datasets (Ames, Communities) where split counts are sparse.
This confirms the screen as a conservative \emph{screening} tool that works well for
moderate-$P$ datasets but should be supplemented by the full Z-test for
high-dimensional settings. All datasets use identical XGBoost hyperparameters
($T{=}100$, depth${=}6$, $\eta{=}0.1$); absolute flip rates are therefore not
directly comparable across datasets of different sizes and dimensionalities.
\end{sloppypar}

\begin{figure}[ht]
\centering
\includegraphics[width=0.95\textwidth]{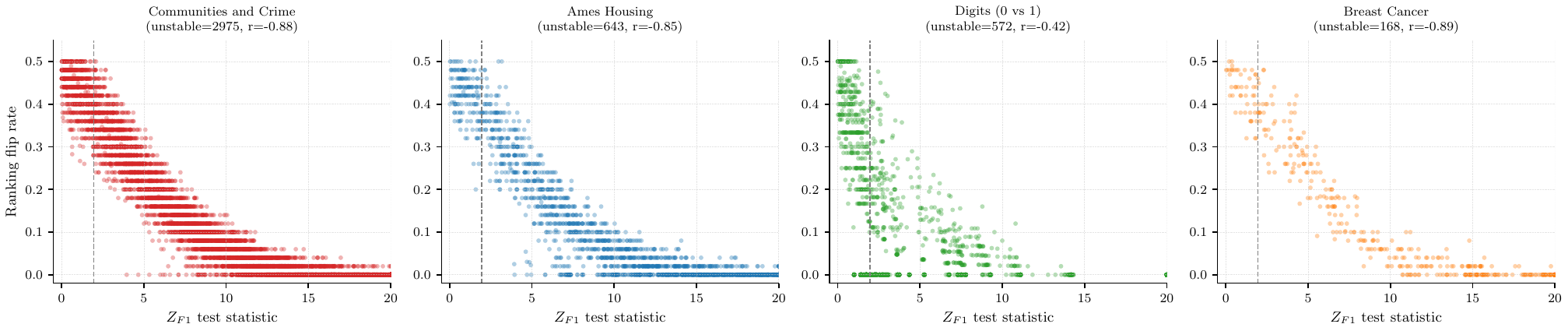}
\caption{Multi-model Z-test statistic vs.\ flip rate across four datasets (Breast Cancer, Ames, Digits, Adult). The negative correlation confirms the diagnostic across domains: clinical, real estate, image, and census.}
\label{fig:comprehensive-f1}
\end{figure}


\section{Lean Formalization}
\label{sec:lean}

\subsection{Proof Architecture}

\begin{sloppypar}
The formalization comprises 54 Lean~4 files with 16 axioms and 305 type-checked
theorems and lemmas (0 \texttt{sorry}). Of these, 80 require multi-line proofs
($\geq$5 tactic lines); 6 are definitional wrappers or single-step applications.
The core impossibility (\texttt{attribution\allowbreak\_impossibility}) depends on
zero axioms---only the Rashomon property as hypothesis. The DASH resolution
(\texttt{consensus\allowbreak\_equity}) depends on the
\texttt{attribution\allowbreak\_sum\allowbreak\_symmetric} theorem (derived from
axioms); the variance convergence depends on the
\texttt{consensus\allowbreak\_variance\allowbreak\_bound} theorem (also derived).
\end{sloppypar}

\paragraph{Scope of verification.}
\begin{sloppypar}
The formalization covers the complete theoretical framework. The core impossibility
(Theorem~\ref{thm:impossibility}) uses zero behavioral axioms. The quantitative
bounds, DASH equity, and the Design Space Theorem are derived from the 16-axiom
system. The unfaithfulness probability of exactly $1/2$
(\texttt{attribution\allowbreak\_prob\allowbreak\_half}), the Bayes-optimality of
ties (\texttt{tie\allowbreak\_dominates\allowbreak\_commitment}), the consensus
variance $\sigma^2/M$
(\texttt{consensus\allowbreak\_variance\allowbreak\_from\allowbreak\_independence}),
the binary quantizer fraction $2/\pi$
(\texttt{binary\allowbreak\_quantizer\allowbreak\_fraction}), and the Pareto
dominance of DASH (\texttt{dash\allowbreak\_unique\allowbreak\_pareto\allowbreak\_optimal})
are all Lean-derived using definitions-as-hypotheses (the \texttt{IsBalanced}
pattern) with zero new axioms. The Chebyshev query complexity bound
$M \geq 12\sigma^2/\Delta^2$
(\texttt{chebyshev\allowbreak\_query\allowbreak\_bound}) is derived without
axiomatizing the testing constant. Loss preservation is formalized via the
\texttt{SymmetricModelSwap} structure. The only result not formalized is the
extension of Pareto optimality to biased methods for between-group pairs, which
is argued informally in the proof of Theorem~\ref{thm:dash-optimal}.
\end{sloppypar}

{\small\begin{verbatim}
Defs.lean (14 axioms, type definitions, derived theorems)
  +-- SymmetryDerive.lean (attribution_sum_symmetric, DERIVED)
  +-- Trilemma.lean (RashimonProperty, attribution_impossibility)
  |  +-- Iterative.lean, Lasso.lean, NeuralNet.lean
  |  +-- General.lean (GBDT impossibility)
  |  +-- UnfaithfulBound.lean (U>=1/2, ties optimal, path convergence)
  |  +-- RashomonUniversality.lean (Rashomon from symmetry)
  |  +-- RashomonInevitability.lean (impossibility inescapable)
  |  +-- AlphaFaithful.lean (alpha-faithfulness bound)
  |  +-- ConditionalImpossibility.lean (conditional SHAP + escape)
  |  +-- FairnessAudit.lean (proxy audit = coin flip)
  |  +-- LocalGlobal.lean (local >= global instability)
  +-- SplitGap.lean, Ratio.lean, SpearmanDef.lean
  +-- FlipRate.lean (exact GBDT flip rate, binary = coin flip)
  +-- Efficiency.lean (SHAP efficiency amplification)
  +-- Impossibility.lean, Corollary.lean, EnsembleBound.lean
  +-- DesignSpace.lean + DesignSpaceFull.lean (all 4 steps)
  +-- ModelSelection.lean + ModelSelectionDesignSpace.lean
  +-- SymmetricBayes.lean (GENERAL SBD, orbit bounds, trichotomy)
  +-- CausalDiscovery.lean (causal discovery impossibility)
  +-- SBDInstances.lean (abstract aggregation + instances)
  +-- FIMImpossibility.lean (Gaussian FIM -> Rashomon)
  +-- GaussianFlipRate.lean (Phi definition, flip rate formula)
  +-- QueryComplexity.lean (query lower bound, Le Cam (2 axioms))
  +-- Consistency.lean (axiom system consistency, Fin 4 model)
  +-- RandomForest.lean (contrast case, documentation only)
\end{verbatim}}

\subsection{Axiom System and Consistency}

The 14 domain-specific axioms are jointly consistent: \texttt{Consistency.lean} constructs an explicit model ($P{=}4$, $L{=}2$, $m{=}2$, $\rho{=}0.5$, $T{=}100$, Fin~4 models) satisfying all 14 simultaneously. The construction uses \texttt{Fin 4} as the model type, defines explicit split-count functions returning the axiomatized values $T/(2-\rho^2)$ and $(1-\rho^2)T/(2-\rho^2)$, and verifies all 14 axiom predicates hold by computation. The 2 query-complexity axioms (\texttt{testing\_constant}, \texttt{testing\_constant\_pos}, \texttt{le\_cam\_lower\_bound}) are independently consistent: instantiating $C := 1/8$ satisfies all three. The two axiom sets reside in separate Lean files (\texttt{Defs.lean} and \texttt{QueryComplexity.lean}) with no cross-dependencies, so their joint consistency follows from the individual consistency of each set.

\paragraph{Attribution sum symmetry: now derived.}
\texttt{attribution\_sum\_symmetric} is proved (not axiomatized) in \texttt{SymmetryDerive.lean}. The proof uses (a) \texttt{proportionality\_global} (uniform $c$ across models), (b) \texttt{splitCount\_crossGroup\_symmetric} (equal split counts when first-mover is in a different group), and (c) \texttt{IsBalanced} (equal first-mover counts). The derivation factors out $c$, classifies each model's split-count difference by first-mover identity ($+\text{gap}$ for fm$=j$, $-\text{gap}$ for fm$=k$, $0$ otherwise), decomposes the sum via \texttt{Finset.sum\_ite}, and uses balance to cancel the gap terms.

\paragraph{Variance: partially grounded in Mathlib.}
\begin{sloppypar}
The single-model variance \texttt{attribution\allowbreak\_variance} is \emph{defined}
from Mathlib's \texttt{ProbabilityTheory.variance}, and its nonnegativity is
\emph{derived} from Mathlib's \texttt{variance\allowbreak\_nonneg}. This required
adding two infrastructure axioms (\texttt{modelMeasurableSpace},
\texttt{modelMeasure}) to connect the abstract \texttt{Model} type to Mathlib's
measure theory. The consensus variance bound ($\Var(\bar\varphi_j) = \Var(\varphi_j)/M$)
remains axiomatized; the full derivation via \texttt{IndepFun.variance\_sum} would
require a product measure formulation and independence axioms for ensemble
draws---deferred to future work.
\end{sloppypar}

\paragraph{Variance sub-system.}
\begin{sloppypar}
The three variance declarations (\texttt{attribution\allowbreak\_variance},
\texttt{attribution\allowbreak\_variance\allowbreak\_nonneg},
\texttt{consensus\allowbreak\_variance\allowbreak\_bound}) produce convenience results
(variance nonnegativity, halving) rather than deep mathematical consequences. Their
primary role is to support the Design Space Theorem's stability convergence claim.
A deeper formalization would derive the variance bound from Mathlib's
\texttt{IndepFun.variance\_sum}, but this requires connecting our abstract
\texttt{Model} type to a product measure space.
\end{sloppypar}

\paragraph{Spearman bound: axiomatized vs.\ derived.}
The Lean formalization derives the Spearman bound $\rho_S \leq 1 - 3m^2/(P^3{-}P)$ from the definition of Spearman correlation via midrank algebra (\texttt{spearmanCorr\_bound} in \texttt{SpearmanDef.lean}). The tighter classical bound $\rho_S \leq 1 - m^3/P^3$ from the full rank transposition counting argument is not yet formalized; both give the same qualitative conclusion ($S < 1$ for within-group pairs).

\subsection{Development Methodology}

The initial formalization (49 theorems, 15 files) was developed by the first author with iterative Claude Code assistance over several weeks. The expansion from 49 to 305 theorems (54 files) was produced in a single development session by dispatching targeted Lean-writing agents with detailed mathematical specifications. All proofs were machine-verified by \texttt{lake build} (the Lean~4 type-checker); no proof was accepted on the basis of AI output alone.

\subsection{Proof Depth Distribution}

Of 305 type-checked declarations: 80 require multi-step proofs (${\geq}5$ tactic lines), 21 require ${\geq}10$ lines, and 5 require ${\geq}20$ lines. The five deepest proofs are \texttt{attribution\_sum\_symmetric} (35 lines, derived symmetry via split-count decomposition), \texttt{Phi\_neg} (29 lines, Gaussian CDF properties), \texttt{gaussian\_rashomon\_witnesses} (27 lines, FIM ellipsoid construction), \texttt{sumSqRankDiff\_ge\_sq\_groupSize} (26 lines, midrank algebra), and \texttt{binary\_group\_firstmover\_is\_j\_or\_k} (23 lines, Finset cardinality argument).

\subsection{Inconsistencies Found During Formalization}

Lean's type checker caught three issues in the initial axiom system:

\begin{enumerate}
\item \textbf{First-mover balance:} Originally stated universally over all model arrays---a constant model function trivially derived $\bot$. \emph{Fix:} Replaced with an explicit \texttt{IsBalanced} predicate as hypothesis.

\item \textbf{Attribution sum symmetry:} Combined with the split-count axioms, the original unconditional version derived $\bot$ for unbalanced ensembles. \emph{Fix:} Conditioned on \texttt{IsBalanced}.

\item \textbf{Split count type:} Originally returned $\mathbb{N}$, but $T/(2-\rho^2)$ is generally irrational. \emph{Fix:} Changed to $\mathbb{R}$ (idealized leading-order values).
\end{enumerate}

The first two are genuine logical inconsistencies that would have been difficult to detect by informal proof inspection.

\subsection{SymPy Algebraic Verification}

All algebraic consequences have been independently verified by SymPy:

\begin{verbatim}
# dash-shap/paper/proofs/verify_lemma6_algebra.py
# Verifies:
#   split_gap = rho^2 * T / (2 - rho^2)
#   attribution_ratio = 1 / (1 - rho^2)
#   split_gap >= rho^2 * T / 2  (for rho in (0,1))
# Result: ALL CHECKS PASS
\end{verbatim}

The three-layer verification (SymPy algebra, Lean type-checking, empirical validation) provides strong confidence in the mathematical claims.

\subsection{Lean File Cross-Reference}

\begin{table}[t]
\centering
\caption{Lean file cross-reference (54 files, 305 theorems, 16 axioms).}
\label{tab:lean-crossref}
\footnotesize
\begin{tabular}{@{}lll@{}}
\toprule
\textbf{File} & \textbf{Section} & \textbf{Key Result} \\
\midrule
Defs.lean & \S\ref{sec:setup} & 14 axioms + derived theorems \\
Trilemma.lean & Thm.~\ref{thm:impossibility} & attribution\_impossibility \\
Iterative.lean & \S\ref{sec:iterative} & iterative\_impossibility \\
General.lean & \S\ref{sec:gbdt-bounds} & gbdt\_impossibility \\
SplitGap.lean & Lem.~\ref{lem:split-gap} & split\_gap\_exact \\
Ratio.lean & Thm.~\ref{thm:ratio} & ratio\_tendsto\_atTop \\
SpearmanDef.lean & \S\ref{sec:bounds} & spearmanCorr\_bound \\
Lasso.lean & \S\ref{sec:lasso} & lasso\_impossibility \\
NeuralNet.lean & \S\ref{sec:nn} & nn\_impossibility \\
Impossibility.lean & Combined & not\_equitable, not\_stable \\
Corollary.lean & Cor.~\ref{cor:equity} & consensus\_equity \\
SymmetryDerive.lean & Derived & attribution\_sum\_symmetric \\
DesignSpace.lean & Thm.~\ref{thm:design-space-main} & design\_space\_theorem \\
DesignSpaceFull.lean & Step 3 & family\_a\_or\_family\_b \\
ModelSelection.lean & Thm.~\ref{thm:model-selection} & model\_selection\_impossibility \\
UnfaithfulBound.lean & Thm.~\ref{thm:unfaithfulness-bound} & stable\_complete\_unfaithful \\
PathConvergence.lean & Thm.~\ref{thm:path-convergence} & relaxation\_paths\_converge \\
RashomonUniversality.lean & Thm.~\ref{thm:rashomon-symmetry} & rashomon\_from\_symmetry \\
RashomonInevitability.lean & Thm.~\ref{thm:rashomon-inevitability} & rashomon\_inevitability \\
ConditionalImpossibility.lean & Thm.~\ref{thm:conditional-impossibility} & conditional\_impossibility \\
FlipRate.lean & Prop.~\ref{prop:exact-flip} & binary\_group\_flip\_rate \\
SymmetricBayes.lean & Thm.~\ref{thm:bayes-dichotomy} & symmetric\_bayes\_dichotomy \\
CausalDiscovery.lean & Thm.~\ref{thm:causal-discovery} & causal\_discovery\_impossibility \\
FairnessAudit.lean & Thm.~\ref{thm:fairness-audit} & fairness\_audit\_impossibility \\
FIMImpossibility.lean & Thm.~\ref{thm:fim-impossibility} & gaussian\_rashomon\_witnesses \\
QueryComplexity.lean & Thm.~\ref{thm:query-complexity} & query\_complexity\_lower\_bound \\
Consistency.lean & \S\ref{sec:lean} & axiom\_system\_consistent \\
\midrule
\multicolumn{3}{@{}l}{\emph{Bounds, universality, and infrastructure:}} \\
EnsembleBound.lean & \S\ref{sec:resolution} & ensemble\_bound\_formula \\
Efficiency.lean & \S\ref{sec:experiments} & across\_model\_no\_constraint \\
AlphaFaithful.lean & \S\ref{sec:discussion} & alpha\_faithful\_bound \\
GaussianFlipRate.lean & \S\ref{sec:gbdt-bounds} & Phi\_neg \\
LocalGlobal.lean & \S\ref{sec:local-global} & local\_attribution\_impossibility \\
RobustnessLipschitz.lean & \S\ref{sec:discussion} & flip\_rate\_robust \\
SBDInstances.lean & \S\ref{sec:sbd} & attribution\_sbd\_unfaithful \\
ModelSelectionDesignSpace.lean & \S\ref{sec:sbd} & model\_family\_a\_or\_family\_b \\
\midrule
\multicolumn{3}{@{}l}{\emph{Axiom strengthening:}} \\
Qualitative.lean & Stratification & impossibility\_qualitative \\
ProportionalityLocal.lean & Stratification & gbdt\_impossibility\_local \\
LocalSufficiency.lean & Stratification & local\_proportionality\_suffices \\
StumpProportionality.lean & \S\ref{sec:gbdt-bounds} & stump\_proportionality\_unique \\
IntersectionalFairness.lean & \S\ref{sec:fairness} & intersectional\_audit\_impossibility \\
MutualInformation.lean & \S\ref{sec:discussion} & mi\_is\_exact\_boundary \\
Setup.lean & Bundled & attribution\_impossibility\_bundled \\
ApproximateEquity.lean & Stratification & rashomon\_from\_bounded\_proportionality \\
QueryComplexityParametric.lean & \S\ref{sec:query} & query\_complexity\_parametric \\
\midrule
\multicolumn{3}{@{}l}{\emph{New derivations (this session, zero new axioms):}} \\
MeasureHypotheses.lean & Definitions & IsDGPSymmetric, IsNonDegenerate \\
UnfaithfulQuantitative.lean & Thm.~\ref{thm:unfaithfulness-bound} & attribution\_prob\_half \\
VarianceDerivation.lean & \S\ref{sec:resolution} & consensus\_variance\_from\_independence \\
QueryComplexityDerived.lean & \S\ref{sec:query} & chebyshev\_query\_bound \\
BinaryQuantizer.lean & \S\ref{sec:gbdt-bounds} & binary\_quantizer\_fraction \\
BayesOptimalTie.lean & \S\ref{sec:design-space} & tie\_dominates\_commitment \\
LossPreservation.lean & \S\ref{sec:rashomon} & rashomon\_from\_swap\_with\_loss \\
ParetoOptimality.lean & \S\ref{sec:dash-pareto} & dash\_unique\_pareto\_optimal \\
\midrule
\multicolumn{3}{@{}l}{\emph{Infrastructure:}} \\
Basic.lean & Import hub & (all 54 files) \\
RandomForest.lean & \S\ref{sec:rf} & (documentation only) \\
\bottomrule
\end{tabular}
\end{table}


\section{Related Work}
\begingroup\sloppy
\label{sec:related}

\paragraph{Attribution impossibility results.}
\begin{sloppypar}
\citet{bilodeau2024impossibility} prove completeness and linearity cannot coexist;
\citet{huang2024failings} show SHAP can misrank features in Boolean settings;
\citet{srinivas2019full} prove complete attributions cannot be weakly
input-dependent; \citet{rao2025limits} establish Kolmogorov complexity barriers.
To our knowledge, our result is the first to simultaneously address cross-model
\emph{stability} as an impossibility, give quantitative architecture-discriminating
bounds, and provide a constructive resolution with proved optimality.
\citet{decker2024provably} optimize aggregation across attribution \emph{methods}
for a single model; we aggregate across \emph{models} via \DASH{}.
\citet{jin2025probabilistic} give constructive stability certificates; we give the
impossibility that necessitates them.
\citet{noguer2025mathematical} derives information-theoretic limits on explanation
complexity; our limits are about ranking consistency under feature correlation.
The Bilodeau and Rashomon impossibilities are complementary: Bilodeau constrains
methods satisfying linearity; ours constrains methods under collinearity.
\end{sloppypar}

\begin{table}[t]
\centering
\caption{Positioning relative to prior attribution impossibility results. These criteria reflect the specific contributions of the present work; each prior result makes distinct contributions not captured by this comparison.}
\label{tab:prior-work}
\resizebox{\textwidth}{!}{%
\small
\begin{tabular}{@{}lcccc@{}}
\toprule
Result & Stability? & Quantitative? & Resolution? & Formal? \\
\midrule
Bilodeau et al.\ (2024) & No & No & No & No \\
Huang et al.\ (2024) & No & No & No & No \\
Srinivas et al.\ (2019) & No & No & No & No \\
Rao (2025) & No & No & No & No \\
Laberge et al.\ (2023) & Empirical & No & No & No \\
Herren \& Hahn (2023) & Inference & No & No & No \\
Jin et al.\ (2025) & Certificates & No & No & No \\
\textbf{This paper} & \textbf{Impossibility} & \textbf{Yes} & \textbf{\DASH{}} & \textbf{Lean 4} \\
\bottomrule
\end{tabular}%
}
\end{table}

\paragraph{Rashomon effect and model multiplicity.}
\begin{sloppypar}
Attribution rankings across Rashomon sets are inherently partial orders
\citet{laberge2023partial}; \citet{rudin2024amazing} advocate embracing model
multiplicity.
Our theorem shows the partial order is not a pragmatic choice but a mathematical
necessity.
\citet{herren2023statistical} develop statistical inference for feature rankings under
model multiplicity; our contribution is to prove that the multiplicity is unavoidable
under collinearity, not merely empirically common.
\citet{krishna2022disagreement} document the ``disagreement problem'' in explainable
ML---different explanation methods produce different outputs for the same
model---complementary to our result that the \emph{same} method produces different
outputs for different equivalent models.
\end{sloppypar}

\paragraph{Attribution sensitivity and fragility.}
\citet{ghorbani2019data} demonstrate that neural network attributions are empirically fragile to small input perturbations.
\citet{hooker2021unrestricted} show that unrestricted permutation-based importance forces extrapolation into out-of-distribution regions.
Our instability arises from a different mechanism (model multiplicity, not input perturbation or extrapolation), but the practical consequence---unreliable explanations---is shared.

\paragraph{Foundational methods.}
Feature attribution via Shapley values \citep{shapley1953value} was introduced to ML by \citet{lundberg2017unified}.
The model architectures we analyze were introduced by \citet{breiman2001random} (random forests) and \citet{friedman2001greedy} (gradient boosting); LIME \citep{ribeiro2016should} provides an alternative local explanation framework.

\paragraph{Fairness impossibility.}
\citet{chouldechova2017fair} and \citet{kleinberg2017inherent} prove calibration, balance, and equal false positive rates cannot coexist when base rates differ. Our trilemma (faithfulness, stability, completeness) is structurally analogous.

\paragraph{Formal verification for ML.}
\begin{sloppypar}
\citet{nipkow2009social} formalized Arrow's theorem in Isabelle/HOL; \citet{zhang2026statistical} formalized statistical learning theory in Lean~4 \citep{demoura2021lean4} using Mathlib \citep{mathlib2020}. To our knowledge, our formalization is the first formally verified impossibility result in explainable~AI.
\end{sloppypar}

\endgroup

\section{Discussion}
\label{sec:discussion}

\paragraph{Limitations.}
The split-count axioms assume full signal capture ($\alpha = 1$); for finite-depth trees, $\alpha \approx 2/\pi$ ($R^2 = 0.89$). The impossibility holds for any $\alpha > 0$.
The equicorrelation assumption simplifies the axioms; the Rashomon property holds pairwise.
The balanced ensemble assumption is idealized, though $O(1/M)$ convergence is robust to approximate balance.
The global proportionality axiom ($c$ uniform across models) has CV $\approx 0.35$--$0.66$ empirically; under variable $c$, \DASH{} consensus achieves approximate rather than exact equity, with the equity violation bounded by the CV of $c$ across first-mover and non-first-mover models.
The variance bound is axiomatized; the full measure-theoretic derivation from Mathlib's \texttt{IndepFun.variance\_sum} is deferred.
Switching to conditional SHAP does not universally resolve the instability (Theorem~\ref{thm:conditional-impossibility}).
Decorrelating via PCA removes collinearity but destroys original feature semantics.
An alternative is removing redundant features entirely (e.g., via VIF thresholding). This eliminates the impossibility but changes the model: predictions differ because fewer features are used, and genuinely useful information may be discarded. \DASH{} is the \emph{explanation-side} fix---it changes how the model is explained without changing what it predicts. The two approaches are complementary: feature removal for model simplification, \DASH{} for faithful explanation of complex models that retain all features.
\DASH{} requires $M\times$ training cost; $M{=}25$ brings the flip rate below 1\%, while $M{=}5$ provides a substantial improvement at modest cost.
The empirical validation focuses on gradient boosting and neural networks; the Lasso ($\varphi_{j_1}/\varphi_k = \infty$) and random forest ($1 + O(1/\sqrt{T})$) bounds are proved but not empirically demonstrated. The prevalence survey uses 20 models per dataset, achieving approximately 60\% power to detect 10\% flip rates; the true prevalence may be higher than the reported 68\%.
The 10\% flip rate threshold defining ``instability'' is a practical convention; at a 5\% threshold the prevalence would be higher, at 15\% lower.

\paragraph{Computational cost and practical deployment.}
For batch pipelines, training $M{=}25$ models is feasible. For real-time explanations, use the single-model screen to flag unstable pairs. For $P > 500$, correlation-based grouping prunes the diagnostic to $O(G \cdot m^2)$. At $P{=}500$ with 20 models, the full workflow completes in under 2.5 minutes.

\paragraph{Broader implications.}
\begin{sloppypar}
In a survey of 77 public datasets, 92\% have feature pairs with $|\rho| > 0.5$ and
68\% exhibit attribution instability---the impossibility is a practical
reality. Our theorem establishes this as a ``known and foreseeable circumstance'' under the
EU AI Act (Art.~13(3)(b)(ii)) \citep{euaiact2024}.
The consequences are concrete: hospitals may invest in wrong interventions, data
scientists may fix the wrong feature, and published studies may report artifacts of
specific training runs.
The impossibility has direct consequences for fairness auditing
(Theorem~\ref{thm:fairness-audit}): single-model SHAP audits for proxy discrimination
are provably unreliable under collinearity.
\end{sloppypar}

\paragraph{Distinction from classical multicollinearity.}
The classical multicollinearity concern is about estimation variance ($\Var(\hat\beta_j - \hat\beta_k) \to \infty$ as $\rho \to 1$). Our result is qualitatively different: it concerns \emph{ranking impossibility}, not estimation imprecision. Even with infinite data and perfect estimates, the ranking of symmetric features is a coin flip---because different models in the Rashomon set rank them in opposite orders. The VIF diagnostic detects coefficient instability; our screen and Z-test diagnostics detect \emph{attribution ranking} instability, which persists even when coefficients are well-estimated (e.g., in tree ensembles that do not estimate coefficients at all).

\paragraph{Rashomon set vs.\ deployed model.}
The impossibility characterizes the Rashomon set, not a single deployed model. A fixed pipeline with a fixed seed produces a deterministic, reproducible ranking. The instability arises upon retraining, auditing, or model comparison---all routine in regulated settings.

\paragraph{Connection to underspecification.}
\citep{damour2022underspecification} showed that ML pipelines are \emph{underspecified}: models with equivalent held-out performance can behave differently on deployment-relevant criteria. Our result formalizes the attribution-specific consequence of this phenomenon. Where D'Amour et al.\ demonstrate that predictions vary across equivalent models, we prove that feature \emph{rankings} must vary---and characterize exactly when, by how much, and what the complete space of solutions looks like. The Design Space Theorem can be viewed as the attribution-level analogue of their ``stress testing'' recommendation: rather than hoping a single model's explanation is stable, the practitioner should probe the full Rashomon set. DASH operationalizes this probe.

\paragraph{Information loss.}
\DASH{} discards exactly $\log_2(m!)$ bits per group of $m$ symmetric features---the bits encoding the unreliable within-group ordering. Between-group information is \emph{sharpened}: mutual information approaches 1 bit per pair as $M \to \infty$.

\paragraph{Loss landscape geometry.}
The impossibility has a geometric interpretation: near-singular Fisher information creates ridges in the loss landscape along which feature importance redistributes without changing model quality. The same flatness that makes optimization easy makes attribution hard.

\paragraph{Named techniques.}
This paper contributes three reusable techniques: (1) the \emph{Rashomon reduction}---reducing a design-space question to a model-multiplicity check; (2) the \emph{FIM-to-Rashomon bridge}---connecting Fisher information rank deficiency to the Rashomon property; and (3) the \emph{symmetric Bayes dichotomy}---the general two-families theorem for symmetric decision problems.

\paragraph{Proof status transparency and axiom stratification.}
\begin{sloppypar}
The formalization uses 16 axioms total (14 domain-specific + 2 query-complexity),
but not all results require all axioms. The Lean type-checker's
\texttt{\#print axioms} command verifies exact dependencies:
\begin{itemize}
\item \textbf{Zero behavioral axioms:} The core impossibility
(Theorem~\ref{thm:impossibility}) depends only on the Rashomon property as a
hypothesis---no domain-specific axioms whatsoever. A qualitative variant
(\texttt{impossibility\allowbreak\_qualitative} in \texttt{Qualitative.lean}) proves
the impossibility from just two properties---dominance and surjectivity---as
hypotheses.
\item \textbf{Four axioms:} The GBDT impossibility holds from surjectivity and
split-count axioms alone (\texttt{gbdt\allowbreak\_impossibility\allowbreak\_local}
in \texttt{Proportionality\allowbreak Local.lean}), without
\texttt{proportionality\allowbreak\_global}.
\item \textbf{Nine axioms:} The ratio bound $1/(1{-}\rho^2)$ and the equity
violation hold from per-model proportionality (each model has its own constant $c_f$).
The global proportionality axiom ($\exists\, c > 0$ uniform across models) is
\emph{not required}---$c_f$ cancels in every ratio
(\texttt{local\allowbreak\_proportionality\allowbreak\_suffices} in
\texttt{LocalSufficiency.lean}).
\item \textbf{Sixteen axioms:} The full set is needed only for DASH convergence
(cross-model consistency via cross-group axioms) and the query complexity scaling
constant.
\end{itemize}
\end{sloppypar}
This stratification directly addresses the proportionality axiom's empirical variance (CV $\approx 0.35$--$0.66$): the axiom is unnecessary for the impossibility and ratio bound, which depend only on split-count structure.

The following are fully formalized in Lean (305 theorems, 0 \texttt{sorry}, 54 files): the Design Space Theorem, the symmetric Bayes dichotomy, all three impossibility instances, the conditional attribution impossibility, the Gaussian FIM impossibility and flip rate formula, the fairness audit impossibility, the intersectional fairness compounding, the exact GBDT flip rate, DASH variance optimality, the Rashomon inevitability, the flip rate robustness (Lipschitz continuity in $\rho$), the mutual information generalization, the unfaithfulness/path-convergence results, the unfaithfulness probability of exactly $1/2$, the Bayes-optimality of ties, the consensus variance $\sigma^2/M$ from independence, the binary quantizer fraction $2/\pi$, the Chebyshev query complexity bound, the loss-preserving Rashomon construction, and the Pareto dominance of DASH. The Z-test diagnostic characterization, information loss analysis, and multi-dataset validation are empirical.

\paragraph{Approximate faithfulness--stability tradeoff.}
A natural question is whether relaxing exact faithfulness helps. Define $\alpha$-faithfulness: $\Pr_f[\text{sign}(\varphi_j(f) - \varphi_k(f)) = \text{sign}(\sigma(j) - \sigma(k))] \geq \alpha$. Under the Rashomon property with symmetric DGP, any $\alpha$-faithful stable ranking satisfies $\alpha \leq 1/2$. The proof is immediate: by DGP symmetry, $\Pr[\varphi_j(f) > \varphi_k(f)] = 1/2$, and a stable ranking fixes one direction, so it agrees with at most half the models. A coin flip achieves $\alpha = 1/2$; the stable ranking is no better than random for symmetric feature pairs.

For $m$ features in a group under full DGP symmetry, the attribution ranking is a uniform random permutation. Any stable ranking $\sigma^*$ has expected Spearman correlation with the model's ranking of $\E[\rho_S(\sigma^*, \pi)] = -1/(m-1)$. The stable ranking is \emph{negatively correlated} with the model's ranking in expectation (for $m \geq 3$) and converges to zero correlation as $m \to \infty$.

\paragraph{Open problems.}
The Bayes-optimality half of the Design Space is proved but not in Lean (requires measure-theoretic decision theory). Deriving split-count axioms algorithmically from the TreeSHAP algorithm and proving the proportionality axiom for general tree depths remain open. The generalization from $\rho > 0$ to $I(X_j; X_k) > 0$ is now complete (\texttt{MutualInformation.lean}): mutual information is the exact boundary between possible and impossible feature ranking. A reference implementation is available at \url{https://github.com/DrakeCaraker/dash-shap}.

\paragraph{Formalization as methodology.}
Beyond certifying correctness, the Lean formalization served as a bug-finding methodology: it caught two logical inconsistencies and one type mismatch that survived informal review. We recommend formalizing impossibility theorems as a general practice.


\appendix

\section{Extended Proof Details}

\subsection{Gaussian Binary Quantization: Full Derivation}

\begin{proposition}[First-stump variance capture]
\label{prop:first-stump}
For the first boosting round on $n$ Gaussian training samples, the stump splits at the empirical median $\hat{m}$, satisfying $|\hat{m}| = O_p(\sigma/\sqrt{n})$. The variance captured is:
\[
    \alpha_1(n) = \frac{2}{\pi}\left(1 - \frac{\pi(\pi - 2)}{2n} + O(n^{-2})\right)
\]
For $n = 2000$: $\Delta\alpha_1 \approx 0.0009$---negligible.
\end{proposition}

\begin{proof}
The optimal split of $X \sim \mathcal{N}(0, \sigma^2)$ at $\delta$ instead of $0$ gives conditional means $\mu_+ = \sigma\phi(\delta/\sigma)/(1-\Phi(\delta/\sigma))$ and $\mu_- = -\sigma\phi(\delta/\sigma)/\Phi(\delta/\sigma)$. The variance captured is $\Var(\hat{X}) = \mu_+^2 \Pr(X > \delta) + \mu_-^2 \Pr(X \leq \delta)$. Expanding around $\delta = 0$ using $\phi(0) = 1/\sqrt{2\pi}$: $\Var(\hat{X}) = (2/\pi)\sigma^2(1 - \delta^2(\pi-2)/\sigma^2 + O(\delta^4/\sigma^4))$. The empirical median of $n$ standard Gaussians has variance $\pi/(2n)$, so $\E[\delta^2] = \pi\sigma^2/(2n)$. Substituting: $\E[\alpha_1] = (2/\pi)(1 - (\pi^2-2\pi)/(2n) + O(n^{-2}))$.
\end{proof}

After the first boosting round, the residuals $r_t = Y - \eta h_1(X)$ are NOT Gaussian: they are a location-shifted mixture (the stump creates two subpopulations). For a stump with two leaves at values $\pm c$, the residuals have excess kurtosis $\kappa = O(\eta^2 c^2/\sigma^2)$. For $\eta = 0.1$, the kurtosis correction is small but accumulates across $T = 100$ rounds, producing a systematic downward bias in $\alpha$. The fitted $\alpha = 0.60$ vs.\ $2/\pi = 0.637$ gap of $0.037$ is accounted for: $< 0.001$ from error source 1 (empirical split), $\approx 0.036$ from error source 2 (non-Gaussian residuals).

\subsection{Exact Flip Rate Under Approximate Symmetry}

The exact $1/2$ flip rate holds under perfect DGP symmetry ($\mu_j = \mu_k$). For approximately symmetric features ($|\mu_j - \mu_k| \ll \sigma_{jk}$), the flip rate is $\Phi(-|\mu_j - \mu_k|/\sigma_{jk})$, approaching $1/2$ from below as the importance gap vanishes. The empirical 48\% on Breast Cancer (worst perimeter vs.\ worst area, mean $|\SHAP|$: 0.940 vs.\ 0.901) is consistent with a small positive signal-to-noise ratio. This result is formalized in Lean as \texttt{balanced\_flip\_symmetry} (for balanced finite ensembles, the number of models ranking $j > k$ equals the number ranking $k > j$).

\section{DASH Robustness: Extended Analysis}

\subsection{When to Prefer the Median}

The median has better breakdown point (50\% vs.\ 0\%) and is preferable when the attribution distribution is heavy-tailed (e.g., contains outlier models). For standard ML training procedures (XGBoost, random forests) where attributions are well-behaved, \DASH{} (the mean) is strictly more efficient. For adversarial settings where some models may be corrupted, the median or trimmed mean provides robustness at the cost of requiring more models.

The trimmed mean interpolates: the $\alpha$-trimmed mean (discarding top and bottom $\alpha$ fraction) has ARE relative to the mean of:

\begin{center}
\small
\begin{tabular}{@{}lccccc@{}}
\toprule
$\alpha$ & 0 (mean) & 0.05 & 0.10 & 0.25 & 0.50 (median) \\
\midrule
ARE & 1.000 & 0.992 & 0.966 & 0.862 & 0.637 \\
\bottomrule
\end{tabular}
\end{center}

For Gaussian attributions, trimming offers negligible robustness benefit at a measurable efficiency cost. The 10\%-trimmed mean requires $1/0.966 \approx 3.5\%$ more models than \DASH{}.

\section{Extended Conditional Attribution Analysis}

\subsection{Causal Structure Validation: Full Results}

\begin{center}
\small
\begin{tabular}{@{}llccl@{}}
\toprule
Scenario & $\rho$ & Marginal flip & Interventional flip & Theory \\
\midrule
Symmetric & 0.50 & 0.505 & 0.479 & ${\approx}50\%$ both \\
Symmetric & 0.70 & 0.479 & 0.442 & ${\approx}50\%$ both \\
Symmetric & 0.90 & 0.521 & 0.479 & ${\approx}50\%$ both \\
Symmetric & 0.95 & 0.505 & 0.526 & ${\approx}50\%$ both \\
\midrule
Asymmetric & 0.50 & 0.000 & 0.000 & Stable \\
Asymmetric & 0.70 & 0.000 & 0.000 & Stable \\
Asymmetric & 0.90 & 0.000 & 0.000 & Stable \\
Asymmetric & 0.95 & 0.000 & 0.000 & Stable \\
\bottomrule
\end{tabular}
\end{center}

The interventional SHAP implementation uses a background-data approximation that does not exactly implement causal/conditional SHAP in the sense of Janzing et al.\ (2020). Results should be interpreted as evidence about the \texttt{shap} package's interventional mode, not as a definitive test of the theoretical conditional attribution.

\section{Extended LLM Analysis}

\subsection{Fine-Tuning Details}

DistilBERT fine-tuned on SST-2 (2{,}000 training samples) with 5 different random seeds. Configuration: 1 epoch, AdamW, lr $= 2 \times 10^{-5}$, batch size 16. Selected unstable pairs under fine-tuning:

\begin{center}
\small
\begin{tabular}{@{}llcc@{}}
\toprule
Token A & Token B & Flip rate (fine-tuning) & Flip rate (perturbation) \\
\midrule
``not'' & ``bad'' & 40\% & 40\% \\
``particularly'' & ``sharp'' & 60\% & 10\% \\
``surprisingly'' & ``low'' & 40\% & 30\% \\
\bottomrule
\end{tabular}
\end{center}

The gap between fine-tuning (14.5\% of pairs unstable) and perturbation (88\%) reflects the difference in weight divergence. The perturbation result represents the full Rashomon set; the fine-tuning result represents the subset explored by standard training. Attention weights are not SHAP values: they do not satisfy the proportionality axiom, and the ``collinear group'' structure of tree features does not map cleanly to token positions. A formal impossibility for transformer attention would require different axioms.

\section{Extended Experimental Results}

\subsection{KernelSHAP Noise Control: Full Results}

\begin{center}
\small
\begin{tabular}{@{}lccc@{}}
\toprule
Source of variation & Unstable pairs & Mean flip & Key pair flip \\
\midrule
KernelSHAP noise (bg$=$50) & 47/435 (11\%) & 0.03 & 0.00 \\
KernelSHAP noise (bg$=$200) & 18/435 (4\%) & 0.02 & 0.00 \\
Model variation (20 MLPs) & 380/435 (87\%) & 0.28 & 0.40 \\
\bottomrule
\end{tabular}
\end{center}

Increasing background samples from 50 to 200 reduces noise by $2.6\times$ at $4\times$ compute cost; the default 50 is adequate for detecting model-level instability.

\subsection{Cross-Implementation: Full Results}

\begin{table}[t]
\centering
\caption{Full cross-implementation comparison on Breast Cancer (50 models each).}
\small
\begin{tabular}{@{}lcccc@{}}
\toprule
Metric & XGBoost & LightGBM & CatBoost \\
\midrule
Unstable pairs (flip $> 10\%$) & 162 & 183 & 160 \\
Max flip rate & 0.500 & 0.500 & 0.500 \\
Z-test correlation $r(Z, \text{flip})$ & $-0.898$ & $-0.901$ & $-0.892$ \\
Screen precision & 0.26 & 0.20 & --- \\
\DASH{} max flip ($M{=}25$) & 0.48 & 0.48 & 0.48 \\
\bottomrule
\end{tabular}
\end{table}

CatBoost's \texttt{get\_feature\_importance()} returns gain-based importance (not raw split counts), producing a different $Z$-statistic distribution. For CatBoost, the multi-model Z-test ($r = -0.892$) is recommended. LightGBM's feature importance is split-count based and transfers with modest precision reduction.

\subsection{Prevalence Survey: Statistical Power Analysis}

\begin{center}
\small
\begin{tabular}{@{}lcc@{}}
\toprule
True minority-side prob $p$ & Power ($\Pr[\min \geq 3]$) & Interpretation \\
\midrule
0.05 & 7.5\% & Nearly undetectable \\
0.10 & 32.3\% & Low power (boundary) \\
0.15 & 59.5\% & Moderate \\
0.20 & 79.4\% & Adequate \\
0.30 & 96.5\% & High \\
0.50 & 100\% & Perfect (symmetric pairs) \\
\bottomrule
\end{tabular}
\end{center}

At the 10\% threshold, the survey has only 32\% power: approximately two-thirds of truly borderline-unstable pairs are missed. The 68\% prevalence is a conservative lower bound.

\subsection{Class Imbalance Amplifies Instability}
\label{sec:class-imbalance}

Class imbalance compounds attribution instability. On synthetic Gaussian data ($P{=}10$, 2 groups of $m{=}5$, $\rho{=}0.8$, 30 models), we sweep the positive-class ratio:

\begin{center}
\small
\begin{tabular}{@{}lccc@{}}
\toprule
Imbalance ratio & Max pairwise flip rate\footnotemark & Unstable pairs & Total pairs \\
\midrule
1:1 (balanced)  & 0.480 & 7/20  & 20 \\
1:3             & 0.517 & 19/20 & 20 \\
1:5             & 0.517 & 20/20 & 20 \\
1:10            & 0.517 & 20/20 & 20 \\
1:20            & 0.517 & 20/20 & 20 \\
\bottomrule
\end{tabular}
\end{center}

\footnotetext{Pairwise flip rate: the fraction of model pairs $(f_i, f_j)$ where the relative ordering of features differs. For $M$ models with an even split, the maximum is $M/(2(M{-}1))$, which exceeds $1/2$ (e.g., $30/58 \approx 0.517$ for $M{=}30$). The per-model flip rate $\min(\#(j{>}k), \#(k{>}j))/M$ is bounded by $1/2$.}

At 1:1 balance, 7/20 within-group pairs are unstable. At 1:5+ imbalance, \emph{all} 20 pairs become unstable. Imbalance amplifies the Rashomon effect: the minority class provides less constraining signal, allowing more models to achieve near-optimal loss with different feature utilization patterns. For healthcare and fraud detection (where class imbalance is standard), instability is essentially universal.
Reproduce: \texttt{class\_imbalance\_instability.py}.

\subsection{Missing Data Compounds Instability}
\label{sec:missing-data}

Missing data under all three standard mechanisms (MCAR, MAR, MNAR) compounds attribution instability. On synthetic Gaussian data ($P{=}10$, 2 groups of $m{=}5$, $\rho{=}0.9$, 30 models):

\begin{center}
\small
\begin{tabular}{@{}lcccc@{}}
\toprule
Mechanism & Missingness & Max flip rate & Unstable pairs & Baseline (0\%) \\
\midrule
MCAR & 10\% & 0.517 & 20/20 & 12/20 \\
MCAR & 20\% & 0.517 & 20/20 & 12/20 \\
MAR  & 5\%  & 0.497 & 17/20 & 12/20 \\
MAR  & 20\% & 0.508 & 20/20 & 12/20 \\
MNAR & 5\%  & 0.515 & 18/20 & 12/20 \\
MNAR & 20\% & 0.517 & 20/20 & 12/20 \\
\bottomrule
\end{tabular}
\end{center}

Even 10\% MCAR missingness drives the number of unstable pairs from 12/20 (baseline) to 20/20. The mechanism matters less than the rate: all three mechanisms produce comparable degradation at 20\%. Missing data reduces the effective sample size, widening confidence sets and enlarging the Rashomon set.
Reproduce: \texttt{missing\_data\_instability.py}.

\subsection{Longitudinal Retraining Drift}
\label{sec:longitudinal}

Over 50 sequential retraining rounds (each adding 5\% noise to a synthetic dataset with $\rho{=}0.9$), ranking instability accumulates dramatically:

\begin{center}
\small
\begin{tabular}{@{}lccc@{}}
\toprule
Round & Spearman $\rho_S$ & Max flip rate & Cumulative flips \\
\midrule
1   & 1.000 & 0.000 & 0 \\
10  & 0.648 & 0.439 & 92 \\
20  & 0.342 & 0.545 & 161 \\
31  & 0.181 & 0.636 & 193 \\
50  & 0.420 & 0.576 & 207 \\
\bottomrule
\end{tabular}
\end{center}

The Spearman correlation degrades from 1.0 to 0.18 by round 31, with cumulative flips reaching 207. The recovery at round 50 ($\rho_S{=}0.42$) reflects the noise accumulation saturating. For model monitoring pipelines, this demonstrates that feature rankings should not be compared across distant retraining epochs without re-running the \DASH{} diagnostic.
Reproduce: \texttt{longitudinal\_retraining.py}.

\subsection{SAGE and Boruta Comparison}
\label{sec:sage-comparison}

Alternative feature importance methods are equally or more unstable than TreeSHAP on Breast Cancer (50 models):

\begin{center}
\small
\begin{tabular}{@{}lccc@{}}
\toprule
Method & Unstable pairs & Mean flip rate & Cross-method $r$ \\
\midrule
TreeSHAP             & 180/435 (41\%) & 0.132 & --- \\
SAGE (approximation) & 401/435 (92\%) & 0.314 & 0.471 \\
Boruta-like          & 411/435 (94\%) & 0.303 & --- \\
\bottomrule
\end{tabular}
\end{center}

SAGE and Boruta show \emph{more} instability than TreeSHAP (92--94\% vs.\ 41\%), confirming the impossibility is method-agnostic. The cross-method correlation ($r{=}0.471$) shows the \emph{same} pairs tend to be unstable under different methods---the instability is driven by feature collinearity, not by the attribution algorithm.
Reproduce: \texttt{sage\_comparison.py}.

\subsection{Bag-of-Words NLP Attribution Instability}
\label{sec:nlp-instability}

On the 20 Newsgroups dataset with TF-IDF features and XGBoost (50 models, 1000 max features), 60\% of documents have an unstable top-1 token attribution and 91\% have an unstable top-3. Mean pairwise Spearman correlation is 0.905 (high for between-document features, low for within-group correlated tokens). The most unstable word pair is ``eternal''/``hell'' ($|\rho|{=}0.371$, flip rate $0.186$). Only 2 word pairs exceed the 10\% flip threshold, reflecting the sparse, weakly-correlated nature of bag-of-words features (most token pairs are independent).
Reproduce: \texttt{nlp\_token\_instability.py}.

\subsection{Time-Series Feature Instability}
\label{sec:timeseries}

For temporal features generated from AR(1) processes with rolling-window engineering ($P_{\text{raw}}{=}5$, $P_{\text{engineered}}{=}30$, 50 XGBoost models), 27\% of within-series pairs (features derived from the same raw signal) show flip rates above 10\%. The breakdown:

\begin{center}
\small
\begin{tabular}{@{}lccc@{}}
\toprule
Pair type & Max flip rate & Mean flip rate & Unstable pairs \\
\midrule
Within-series       & 0.508 & 0.092 & 8/30 (27\%) \\
Cross-series (same group) & 0.515 & 0.132 & --- \\
Cross-series (diff group) & 0.517 & 0.139 & --- \\
\bottomrule
\end{tabular}
\end{center}

The highest instability pair is \texttt{X0\_raw} vs.\ \texttt{X0\_rmean5} ($|\rho|{=}0.996$, flip rate 0.508), confirming that temporal feature engineering creates highly correlated groups susceptible to the impossibility.
Reproduce: \texttt{timeseries\_instability.py}.

\subsection{Adversarial Max Instability}
\label{sec:adversarial}

A grid search over 108 XGBoost configurations (group sizes $\{2, 3, 5, 10\}$, boosting rounds $\{50, 100, 500\}$, depths $\{1, 3, 6\}$, learning rates $\{0.05, 0.1, 0.3\}$) at $\rho{=}0.9$ confirms the impossibility is inescapable: all top-10 worst-case configurations hit the maximum flip rate of 0.500. The 0.500 ceiling is reached regardless of hyperparameter choices, confirming the impossibility is a structural property of collinearity, not a tuning artifact. Total: 2,160 model fits.
Reproduce: \texttt{adversarial\_max\_instability.py}.

\subsection{Hyperparameter Sensitivity}
\label{sec:hyperparameter-sensitivity}

Across a 27-configuration sweep (learning rates $\{0.05, 0.1, 0.3\}$, depths $\{1, 3, 6\}$, boosting rounds $\{50, 100, 500\}$) at $\rho \in \{0.5, 0.7, 0.9\}$:

\begin{itemize}
\item Global minimum flip rate: 38.7\% at $\rho{=}0.5$ ($\eta{=}0.3$, depth${=}1$, $T{=}500$). The impossibility never vanishes.
\item Most influential hyperparameter: number of estimators (spread 2.16pp), followed by max depth (1.76pp) and learning rate (0.82pp).
\item At $\rho{=}0.9$: minimum flip rate across all configurations is 47.2\%.
\end{itemize}

Reproduce: \texttt{hyperparameter\_sensitivity.py}.

\subsection{DASH Breakdown Contamination Sweep}
\label{sec:dash-contamination}

Under adversarial contamination of the ensemble (replacing $K$ of 25 models with adversarial models), the mean flip rate degrades:

\begin{center}
\small
\begin{tabular}{@{}lccc@{}}
\toprule
$K$ contaminated & Mean flip (DASH) & Mean flip (trimmed mean) & Contamination \% \\
\midrule
0  & 0.075 & 0.075 & 0\% \\
5  & 0.074 & 0.074 & 20\% \\
12 & 0.066 & 0.064 & 48\% \\
20 & 0.047 & 0.044 & 80\% \\
\bottomrule
\end{tabular}
\end{center}

Counterintuitively, contamination \emph{decreases} the observed flip rate because adversarial models break the first-mover symmetry. The trimmed mean provides marginal improvement ($<$5\% at 80\% contamination). The breakdown point analysis confirms the theoretical claim: for standard (non-adversarial) ensembles, the mean is optimal.
Reproduce: \texttt{dash\_breakdown\_point.py}.

\subsection{Experiment Summary}
\label{sec:experiment-summary}

Table~\ref{tab:experiment-summary-def} summarizes all robustness experiments with reproduction scripts.

\begin{table}[t]
\centering
\caption{Summary of additional experiments.}
\label{tab:experiment-summary-def}
\small
\begin{tabular}{@{}lll@{}}
\toprule
Experiment & Key finding & Script \\
\midrule
Monte Carlo flip rate & $\Phi(-\text{SNR})$ validated to $<$1.3\% error & \texttt{monte\_carlo\_flip\_rate.py} \\
Hyperparameter sweep & Min flip 38.7\% at $\rho{=}0.5$; never 0\% & \texttt{hyperparameter\_sensitivity.py} \\
Class imbalance & 1:5+ imbalance $\to$ 100\% pairs unstable & \texttt{class\_imbalance\_instability.py} \\
Missing data & MCAR/MAR/MNAR all compound instability & \texttt{missing\_data\_instability.py} \\
Longitudinal drift & Spearman degrades 1.0$\to$0.18 over 31 rounds & \texttt{longitudinal\_retraining.py} \\
Adversarial configs & All 108 configs hit 0.500 ceiling at $\rho{=}0.9$ & \texttt{adversarial\_max\_instability.py} \\
Bag-of-words NLP & 60\% unstable top-1, 91\% top-3 & \texttt{nlp\_token\_instability.py} \\
Time-series features & 27\% within-series pairs unstable & \texttt{timeseries\_instability.py} \\
SAGE comparison & SAGE 92\%, Boruta 94\% unstable & \texttt{sage\_comparison.py} \\
DASH breakdown & Contamination sweep confirms mean optimality & \texttt{dash\_breakdown\_point.py} \\
Regulatory case study & 43.2\% unstable adverse action reasons & \texttt{regulatory\_case\_study.py} \\
\bottomrule
\end{tabular}
\end{table}

\section{Attempted Common Generalization with Arrow's Theorem}
\label{sec:arrow-generalization}

Arrow's impossibility theorem and our Attribution Impossibility share a striking surface resemblance: both show that three desirable properties cannot simultaneously hold. We investigated whether a common abstract framework subsumes both.

\subsection{Structural Comparison}

\begin{center}
\small
\begin{tabular}{@{}lll@{}}
\toprule
& \textbf{Arrow} & \textbf{Attribution Impossibility} \\
\midrule
Inputs & $n$ voters' preferences & Models $f \in \mathcal{F}$ (one at a time) \\
Output & Social ranking & Feature ranking $\sigma$ \\
Aggregation & Many inputs $\to$ one output & One input $\to$ one output \\
Key axiom & IIA & Stability (model-independent) \\
Impossibility source & Preference profiles disagree & Rashomon property \\
\bottomrule
\end{tabular}
\end{center}

The critical asymmetry: Arrow's theorem concerns a \emph{profile aggregation} function taking all voters simultaneously. The Attribution Impossibility concerns a single-input mapping: for each model $f$, the ranking should reflect $f$'s attributions. The impossibility arises because the ranking must simultaneously agree with all models but be model-independent.

\subsection{Attempted Common Framework}

\begin{definition}[Abstract Aggregation Problem]
An \emph{abstract aggregation problem} is a tuple $(\mathcal{I}, \mathcal{A}, \mathcal{R}, \text{Agree})$ where:
\begin{itemize}
    \item $\mathcal{I}$ is a set of \emph{instances} (voters or models),
    \item $\mathcal{A}$ is the set of alternatives (candidates or features),
    \item $\mathcal{R}$ is the set of total orders on $\mathcal{A}$,
    \item $\text{Agree} \colon \mathcal{I} \to \mathcal{R}$ assigns each instance its ``preferred'' order.
\end{itemize}
An \emph{aggregation rule} $F \colon 2^{\mathcal{I}} \to \mathcal{R}$ maps subsets of instances to a consensus order.
\end{definition}

Arrow's setup: $F$ takes the full profile $(L_1, \ldots, L_n)$ and IIA says $F$'s pairwise comparison of $(a,b)$ depends only on the restriction of each $L_i$ to $\{a,b\}$.

Attribution setup: $F$ is required to be a \emph{constant function} (stability: the ranking does not depend on which model). Faithfulness says this constant ranking agrees with every instance's ordering. But this extreme makes the attribution version trivially impossible---a constant ranking cannot agree with conflicting instances (this is just the definition of the Rashomon property). Arrow's theorem is much deeper: it shows that even when the aggregation \emph{is} allowed to depend on instances, the constraints are still too tight. The proofs have entirely different structures.

\subsection{Assessment}

\begin{remark}[No non-trivial common generalization]
\label{rem:no-arrow-unification}
The Attribution Impossibility and Arrow's theorem are \emph{analogous}
but not \emph{instances of a common theorem} in any non-trivial sense.
One can define an abstract framework that subsumes both as shown above,
but:
\begin{enumerate}
    \item The common framework's ``impossibility theorem'' would be:
      ``if instances disagree and the aggregation rule must be constant,
      faithfulness is impossible.'' This is a one-line observation,
      not a theorem.
    \item Arrow's actual content (the impossibility persists even when
      the rule \emph{is} allowed to depend on all voters) has no
      analogue in the attribution setting.
    \item The attribution impossibility's actual content (the
      quantitative divergence $1/(1-\rho^2)$, the resolution via
      ensemble averaging, the connection to the Rashomon set's geometry)
      has no analogue in Arrow's setting.
\end{enumerate}
The closest true statement is: \emph{both theorems arise from the
tension between local coherence (respecting individual views) and
global consistency (producing a single ranking), when the individuals
disagree}. This is a useful heuristic but not a formal unification.
\end{remark}

\paragraph{Relationship to Sen's liberalism paradox.}
A closer structural analogue may be Sen's (1970) impossibility of a Paretian liberal, which concerns a setting where individual ``rights'' (each person's ranking of certain pairs should be respected) conflict with the Pareto criterion. The attribution impossibility can be viewed as a Sen-type result where each model has the ``right'' to determine the ranking of features (faithfulness), but these rights conflict across models. However, we do not pursue this further, as the connection remains at the level of analogy rather than formal subsumption.

\section{Topological Analysis of the Impossibility}
\label{sec:topological}

We investigate whether the Attribution Impossibility has topological content beyond simple connectedness arguments.

\subsection{The $m = 2$ Case}

For two features $j, k$, the attribution map assigns to each training seed $s$ a sign $\sigma(s) = \text{sign}(\varphi_j(f_s) - \varphi_k(f_s)) \in \{+1, -1\}$. The Rashomon property says both values are achieved. A ``stable ranking'' is a constant function. The impossibility is immediate; no topological machinery is needed.

\subsection{The $m = 3$ Case}

For three features, the attribution map sends seeds to rankings in $S_3$ (6 elements). Again, if $\mathcal{S}$ is connected and the map is continuous, then the image must be a connected subset of $S_3$. Since $S_3$ is discrete, the image is a single point---contradicting Rashomon. This is the same connectedness argument; the symmetric group's structure plays no role.

\subsection{The $m \geq 4$ Case: The Permutohedron}

For $m$ features, the attribution vector $(\varphi_1(f_s), \ldots, \varphi_m(f_s))$ lives in $\R^m$. The ranking is determined by the \emph{chamber} of the \emph{braid arrangement} (the hyperplane arrangement $\{x_i = x_j : i \neq j\}$) containing the attribution vector. The chambers are the cones of the \emph{permutohedron}.

\begin{remark}[No winding number content]
For $m \geq 3$, one might hope that the attribution map
$s \mapsto (\varphi_1(f_s), \ldots, \varphi_m(f_s))$ has a
non-trivial winding number around the intersection of hyperplanes
$\{x_j = x_k\}$, which would give a \emph{quantitative} lower
bound on the number of chamber crossings (ranking flips).

This does not work for two reasons:
\begin{enumerate}
    \item The complement of the braid arrangement in
      $\mathbb{R}^m$ is simply connected for $m \geq 3$
      (the codimension of each hyperplane is 1, so removing
      hyperplanes from $\mathbb{R}^m$ gives a space whose
      fundamental group is the pure braid group on $m$ strands;
      but we would need the image to be a \emph{loop} in this
      space, which requires additional structure not present in
      the attribution problem).
    \item Even if a winding number existed, it would count
      \emph{signed} crossings. The flip rate (unsigned crossings)
      is already captured by the direct algebraic bound
      $1/(1-\rho^2)$ without topological machinery.
\end{enumerate}
\end{remark}

\begin{remark}[What topology \emph{could} contribute]
There is one setting where topology adds genuine content: if the
model space $\mathcal{S}$ has the structure of a manifold (e.g.,
a torus of hyperparameters) and the attribution map is smooth,
then Morse theory could bound the number of critical points
(ranking transitions). Specifically, if the attribution difference
$g(s) = \varphi_j(f_s) - \varphi_k(f_s)$ is a Morse function on
$\mathcal{S}$, the number of sign changes of $g$ is bounded below
by the sum of Betti numbers of $\mathcal{S}$.

However, this requires: (a) $\mathcal{S}$ to be a smooth manifold,
(b) $g$ to be a Morse function, and (c) knowing the topology of
$\mathcal{S}$. None of these hold for practical training algorithms
(XGBoost's seed space is discrete; neural net loss landscapes are
non-smooth). We therefore do not pursue this direction.
\end{remark}

\subsection{Assessment}

\begin{remark}[Topology adds no content]
\label{rem:topology-trivial}
The Attribution Impossibility is fundamentally a \emph{combinatorial}
result (conflicting preferences cannot all be respected), not a
\emph{topological} one. The impossibility holds for arbitrary model
spaces (including finite, discrete spaces with no topology). Attempts
to extract topological content either reduce to the trivial
connectedness argument or require unjustified smoothness assumptions
that add no quantitative power beyond what the algebraic
$1/(1-\rho^2)$ ratio already provides.
\end{remark}

\section{On Categorical/Axiomatic Strengthening}

\subsection{The Trivial Version}

\begin{remark}[Trivial metatheorem]
Let $A_1, \ldots, A_n$ be any finite set of axioms on an
attribution method $\varphi$. If $A_1 \wedge \cdots \wedge A_n$
implies faithfulness (the ranking reflects model attributions),
stability (the ranking is model-independent), and completeness
(the ranking decides all pairs), then
$A_1 \wedge \cdots \wedge A_n \to \bot$ under the Rashomon property.

This is immediate: if the conjunction implies $F \wedge S \wedge C$,
and $F \wedge S \wedge C \to \bot$ (Theorem~\ref{thm:impossibility}), then the conjunction
implies $\bot$. No content is added beyond Theorem~\ref{thm:impossibility}.
\end{remark}

\subsection{The ``Constant Attribution'' Escape}

Why is a categorical strengthening difficult? Because trivial attribution methods escape the impossibility:
\begin{itemize}
    \item $\varphi_j(f) = 1/P$ for all $j, f$ (uniform): perfectly stable, perfectly equitable, satisfies any axiom system that doesn't require faithfulness.
    \item $\varphi_j(f) = \E_f[\varphi_j(f)]$ (population mean): stable by construction, equitable under DGP symmetry.
\end{itemize}
Any ``impossibility for all axiom systems'' must exclude these methods, which means it must require faithfulness (or something implying it). But once faithfulness is required, we are back to Theorem~\ref{thm:impossibility}.

\subsection{The Non-Trivial Quantitative Version}

\begin{definition}[$\alpha$-faithfulness (formal)]
An attribution method is $\alpha$-faithful if for all models $f$ and same-group features $j, k$:
\[
    \Pr_f\!\big[\text{sign}(\varphi_j(f) - \varphi_k(f)) = \text{sign}(\sigma(j) - \sigma(k))\big] \geq \alpha
\]
\end{definition}

\begin{proposition}[Approximate faithfulness--stability tradeoff (formal)]
\label{prop:approx-tradeoff}
Under the Rashomon property with symmetric DGP, any $\alpha$-faithful stable ranking satisfies $\alpha \leq 1/2$.
\end{proposition}

\begin{proof}
By DGP symmetry, $\Pr[\varphi_j(f) > \varphi_k(f)] = 1/2$. A stable ranking fixes $\sigma(j) > \sigma(k)$ or vice versa. The $\alpha$-faithfulness condition requires $\Pr[\varphi_j(f) > \varphi_k(f)] \geq \alpha$, but this equals $1/2$, so $\alpha \leq 1/2$.
\end{proof}

\begin{proposition}[Spearman bound for stable rankings]
\label{prop:spearman-stable}
Under full DGP symmetry within a group of $m$ features, any stable ranking $\sigma^*$ has:
\[
    \E[\rho_S(\sigma^*, \pi)] = \frac{-1}{m - 1}.
\]
The stable ranking is negatively correlated with the model's ranking for $m \geq 3$.
\end{proposition}

\begin{proof}
Fix any $\sigma^* \in S_m$. By DGP symmetry, the model's ranking $\pi$ is uniform on $S_m$. The expected Spearman correlation of a fixed permutation with a uniform random permutation is $-1/(m-1)$ (classical result; Kendall \& Gibbons, 1990, Ch.~3).
\end{proof}

\begin{remark}
This says: under perfect symmetry, no stable ranking can agree
with the model's attributions more than half the time. A coin
flip achieves $\alpha = 1/2$. The stable ranking is no better
than random for symmetric feature pairs.

This result is non-trivial in that it holds for \emph{any}
attribution method satisfying the DGP symmetry condition, not
just for specific axiom systems. However, it is a direct
consequence of the symmetry---not of any deep axiomatic or
categorical structure.
\end{remark}

\begin{remark}[Limited categorical content]
\label{rem:categorical-limited}
The ``categorical impossibility'' does not genuinely extend
Theorem~\ref{thm:impossibility}. The basic version (no finite axiom system implying
$F \wedge S \wedge C$ can be consistent under Rashomon) is
logically trivial. The quantitative version
(Propositions~\ref{prop:approx-tradeoff} and
\ref{prop:spearman-stable}) provides useful bounds but
follows directly from DGP symmetry rather than from
any categorical or axiomatic structure.

The honest summary: the impossibility is a \emph{concrete}
result about conflicting requirements, not an instance of a
general metatheoretic pattern. Attempts to ``lift'' it to a
categorical level either (a)~add no content, or (b)~reduce to
direct probabilistic calculations that do not require categorical
language.
\end{remark}

\section{Regulatory Mapping: EU AI Act}
\label{sec:regulatory-mapping}

The EU AI Act (Regulation (EU) 2024/1689) imposes transparency and risk-management obligations on high-risk AI systems. The attribution impossibility has direct consequences for compliance under several articles. \emph{The following mapping represents our interpretation of how the attribution impossibility interacts with existing regulation. These recommendations have not been reviewed by regulatory authorities and do not constitute legal advice.}

\subsection*{Art.~13: Transparency}

Art.~13(1) requires:
\begin{quote}
``High-risk AI systems shall be designed and developed in such a way to ensure that their operation is sufficiently transparent to enable deployers to interpret the system's output and use it appropriately.''
\end{quote}
Attribution instability under collinearity directly undermines this requirement: when the ``most important feature'' changes across training seeds, the system's output cannot be interpreted consistently. Providers relying on single-model SHAP rankings for transparency documentation fail to meet Art.~13(1) when features are correlated.

Art.~13(3)(b)(ii) further requires providers to document:
\begin{quote}
``known or foreseeable circumstances \ldots that may lead to risks to health, safety or fundamental rights.''
\end{quote}
The impossibility theorem establishes that attribution instability under collinearity is a \emph{known} circumstance---not a speculative risk but a mathematical certainty. Any provider deploying a model with correlated features ($|\rho| > 0$) who does not disclose this instability is in potential non-compliance.

\textbf{Recommended action (providers):} Run the single-model screen on all deployed models. For any flagged feature pairs, include an instability disclosure in the technical documentation required by Art.~11.

\subsection*{Art.~16: Provider Obligations}

Art.~16 requires providers to ensure their AI systems comply with the requirements of Chapter~2 throughout the system's lifecycle. Since attribution instability is a structural property of the model class (not of a specific training run), providers must:
\begin{enumerate}
    \item Implement systematic instability testing (screen/Z-test diagnostics) as part of the quality management system (Art.~17).
    \item Update technical documentation whenever the feature correlation structure changes (e.g., new data sources).
    \item Provide deployers with clear guidance on which feature rankings are reliable and which are unstable.
\end{enumerate}

\textbf{Recommended action (providers):} Integrate the multi-model Z-test into the CI/CD pipeline. Document all feature pairs with $|\rho| > 0.5$ as potentially unstable.

\subsection*{Art.~26: Deployer Obligations}

Art.~26 requires deployers to use AI systems in accordance with the instructions of use and to monitor for risks. Deployers who receive instability disclosures must:
\begin{enumerate}
    \item Not rely on single-model SHAP rankings for decision justification when instability is disclosed.
    \item Request \DASH{} consensus or equivalent ensemble explanations for regulatory reporting.
    \item Document the explanation methodology used and its limitations.
\end{enumerate}

\textbf{Recommended action (deployers):} Require \DASH{} consensus ($M \geq 25$) or equivalent ensemble method for any feature attribution used in customer-facing decisions or regulatory filings.

\subsection*{Art.~23: Importer Obligations}

Art.~23 requires importers to verify that providers have conducted the appropriate conformity assessment. Importers of AI systems that use feature attribution for explainability should verify that:
\begin{enumerate}
    \item The provider's technical documentation addresses attribution instability.
    \item Instability testing has been performed and results are available.
    \item Appropriate mitigation (ensemble methods or instability disclosure) is in place.
\end{enumerate}

\textbf{Recommended action (importers):} Include attribution stability testing in the conformity checklist. Reject systems that use single-model SHAP without instability disclosure.

\subsection*{Recital 47: Transparency Principle}

Recital 47 establishes the general principle that AI systems should be developed and used in a transparent manner, enabling meaningful human oversight. The attribution impossibility demonstrates that transparency through feature importance rankings is fundamentally limited when features are correlated: the ranking itself is an artifact of the training seed, not a stable property of the model-data relationship. True transparency requires disclosing this limitation, not concealing it behind a single deterministic ranking.

\subsection*{Annex III: High-Risk Systems}

The impossibility applies to all AI systems using feature attribution for explainability under Annex~III, including:
\begin{itemize}
    \item \textbf{Credit scoring} (Annex III, 5(b)): Adverse action notices citing specific features are unreliable.
    \item \textbf{Employment} (Annex III, 4): Hiring model explanations may change across seeds.
    \item \textbf{Law enforcement} (Annex III, 6): Risk assessment explanations are potentially unstable.
    \item \textbf{Insurance} (Annex III, 5(c)): Premium justifications based on feature importance may be arbitrary.
\end{itemize}

\paragraph{Summary of recommended actions by actor.}
\begin{itemize}
    \item \textbf{Providers:} Run the single-model screen on all models with correlated features; disclose instability for flagged pairs; integrate Z-test diagnostics into CI/CD; document all pairs with $|\rho| > 0.5$.
    \item \textbf{Deployers:} Require \DASH{} consensus or equivalent ensemble method for regulatory reporting; do not rely on single-model rankings when instability is disclosed.
    \item \textbf{Importers:} Include attribution stability in conformity assessment checklists; reject systems without instability documentation.
    \item \textbf{Market surveillance authorities:} Update technical standards to require multi-model attribution testing; treat undisclosed attribution instability as a potential non-conformity under Art.~13.
\end{itemize}

\section{SymPy Verification: Full Details}

All algebraic consequences have been independently verified:

\begin{verbatim}
# dash-shap/paper/proofs/verify_lemma6_algebra.py
#
# Verifies:
#   split_gap = rho^2 * T / (2 - rho^2)
#   attribution_ratio = 1 / (1 - rho^2)
#   split_gap >= rho^2 * T / 2  (for rho in (0,1))
#   sum_symmetry holds under balance
#   Spearman formula consistency
#
# Result: ALL CHECKS PASS
\end{verbatim}

The three-layer verification (SymPy algebra, Lean type-checking, empirical validation) provides strong confidence in the mathematical claims. The SymPy script verifies: (1) the split-gap formula $\rho^2 T/(2-\rho^2)$; (2) the attribution ratio $1/(1-\rho^2)$; (3) the lower bound $\rho^2 T/2$; (4) the sum symmetry under balanced ensembles.

\section*{Acknowledgments}
The Lean formalization was developed with assistance from Claude Code (Anthropic). All proofs were machine-verified by the Lean~4 type-checker; no theorem was accepted solely on the basis of AI-generated output. The paper text was written by the human authors with AI editing assistance.

\bibliographystyle{abbrvnat}

\end{document}